\documentclass[10pt]{article}

\usepackage{hyperref}
\usepackage{url}
\usepackage{amsmath,amssymb,enumerate,subfigure,multirow,hhline,color}
\usepackage{xcolor}
\usepackage{graphicx}
\usepackage{graphics}
\usepackage[sort]{cite}
\usepackage{amsthm}
\usepackage{algorithm}
\usepackage{algpseudocode}

\usepackage{tcolorbox}
\usepackage{epsfig,psfrag}
\usepackage{tabularx}
\usepackage{tabulary}
\usepackage[authoryear]{natbib}
\usepackage{hyperref}
\hypersetup{breaklinks=true,colorlinks=false,linkcolor=blue,citecolor=blue}
\usepackage[noabbrev,capitalise,nameinlink]{cleveref}

\usepackage{booktabs}

\newtheorem{theorem}{Theorem}

\newtheorem{corollary}{Corollary}

\newtheorem{lemma}{Lemma}

\newtheorem{proposition}{Proposition}

\allowdisplaybreaks

\DeclareMathOperator{\argmin}{arg\,min}
\DeclareMathOperator{\argmax}{arg\,max}

\title{Analysis of Approximate Linear Programming Solution to Markov Decision Problem with Log Barrier Function}

\author{Donghwan Lee, Hyukjun Yang, \& Bum Geun Park \\
Department of Electrical Engineering\\
Korea Advanced Institute of Science and Technology\\
Daejeon, 34141, South Korea \\
\texttt{\{donghwan,jundol32,j4t123\}@kaist.ac.kr} \\
}

\begin{document}

\maketitle

\begin{abstract}
There are two primary approaches to solving Markov decision problems (MDPs): dynamic programming based on the Bellman equation and linear programming (LP). Dynamic programming methods are the most widely used and form the foundation of both classical and modern reinforcement learning (RL). By contrast, LP-based methods have been less commonly employed, although they have recently gained attention in contexts such as offline RL. The relative underuse of the LP-based methods stems from the fact that it leads to an inequality-constrained optimization problem, which is generally more challenging to solve effectively compared with Bellman-equation-based methods. The purpose of this paper is to establish a theoretical foundation for solving LP-based MDPs in a more effective and practical manner. Our key idea is to leverage the log-barrier function, widely used in inequality-constrained optimization, to transform the LP formulation of the MDP into an unconstrained optimization problem. This reformulation enables approximate solutions to be obtained easily via gradient descent. While the method may appear simple, to the best of our knowledge, a thorough theoretical interpretation of this approach has not yet been developed. This paper aims to bridge this gap.
\end{abstract}

\section{Introduction}

There are two primary approaches to solving Markov decision problems (MDPs): dynamic programming methods~\citep{bertsekas1996neuro,puterman2014markov} based on the Bellman equation and linear programming (LP) methods~\citep{puterman2014markov,de2003linear,ghate2013linear,ying2020note}. Dynamic programming is by far the most widely used approach and constitutes the foundation of both classical and modern reinforcement learning (RL)~\citep{sutton1998reinforcement}. In contrast, LP-based methods have been employed less frequently~\citep{wang2016online,chen2016stochastic,lee2019stochastic,lee2018stochastic,chen2018scalable,serrano2020faster,neu2023efficient,nachum2020reinforcement,bas2021logistic,lu2021convex,lu2022convex} in RL, though they have recently gained traction in contexts such as offline reinforcement learning~\citep{ozdaglar2023revisiting,zhan2022offline,gabbianelli2024offline,kamoutsi2021efficient,sikchi2024dual}. The relative underuse of LP formulations can be attributed to the fact that they result in inequality-constrained optimization problems, which are generally more difficult to solve effectively compared to Bellman-equation-based methods.

Recently, the LP formulation of MDPs has recently received renewed interest, especially in the offline-RL literature~\citep{ozdaglar2023revisiting,zhan2022offline,gabbianelli2024offline,kamoutsi2021efficient,sikchi2024dual}, because it offers several advantages over Bellman-equation–based approaches. Unfortunately, LP-based RL methods typically rely on primal–dual schemes~\citep{wang2016online,chen2016stochastic,lee2019stochastic,lee2018stochastic,chen2018scalable,serrano2020faster,neu2023efficient,nachum2020reinforcement,bas2021logistic,lu2021convex,lu2022convex}, which are known to often exhibit relatively slow convergence and higher computational costs in practice compared to standard RL approaches. For these reasons, we believe there is a clear need to develop effective alternatives for solving the LP form.

The purpose of this paper is to establish a theoretical foundation for addressing LP-based MDPs in a more effective and practical manner. The key idea is to employ the log-barrier function~\citep{Boyd2004}, widely used in inequality-constrained optimization, to reformulate the LP representation of the MDP into an unconstrained optimization problem. This reformulation allows approximate solutions to be obtained efficiently using gradient descent~\citep{nesterov2018lectures,bertsekas1999nonlinear}. While this approach may appear simple, to the best of our knowledge, a comprehensive theoretical interpretation has not yet been developed. This paper aims to bridge this gap.

More specifically, we investigate the single-objective function $f_\eta$ induced by the log-barrier formulation with the barrier parameter $\eta>0$, whose minimizer yields an approximate solution, $\tilde Q_\eta$, to the original MDP. This approximate solution, $\tilde Q_\eta$, corresponds to an approximate optimal Q-function. We first conduct an error analysis, deriving not only an upper bound but also a lower bound on the error norm, $\|\tilde Q_\eta- Q^* \|_\infty$, between the approximate solution, $\tilde Q_\eta$, and the true optimal Q-function $Q^*$. These bounds depend linearly on the log-barrier parameter $\eta$, which implies that both the upper and lower bounds decrease linearly as $\eta$ becomes smaller. Beyond the error norm, we also establish error bounds for the MDP objective function $J^\pi$ itself. In particular, our framework yields both a primal approximate solution $\tilde Q_\eta$ and a dual approximate solution $\tilde \lambda_\eta$, from which the corresponding primal policy and dual policy are derived. For each case, we derive bounds on the deviation of their objective values from the optimal objective value, and these bounds likewise diminish linearly with $\eta$.

In addition, we establish several properties of the objective function $f_\eta$, including its convexity, properties of its domain, and properties of its sublevel set. As mentioned earlier, the approximate LP solution can be obtained via gradient descent, and we provide an analysis of the convergence behavior of this gradient descent method. Lastly, we explore the applicability and extension of the proposed theoretical foundation to deep RL. Specifically, we introduce a novel loss function, derived from the log-barrier formulation, which serves as an alternative to the conventional mean-squared-error Bellman loss in deep Q-network (DQN)~\citep{mnih2015human} and deep deterministic policy gradient (DDPG)~\citep{lillicrap2015continuous}. This yields a new deep RL algorithm within the DQN and DDPG frameworks. The effectiveness of the proposed method is demonstrated through comparative evaluations with standard DQN and DDPG across multiple OpenAI Gym benchmark tasks. The experimental results demonstrate that the proposed method performs on par with conventional DQN across the evaluated environments, and achieves markedly superior performance in specific tasks. In addition, experimental results show that incorporating the proposed method into DDPG yields markedly improved learning performance compared to the conventional DDPG algorithm in a wide range of tasks.
Finally, the main contributions are briefly summarized as follows: (1) Log-barrier LP for MDPs \& error bounds. Introduce a novel log-barrier formulation of the MDP LP and derive rigorous error bounds for the approximate solution, explicitly quantifying how the approximation error scales with the barrier weight $\eta$. (2) Analytic properties \& convergence. Prove structural properties of the objective (convexity, local strong convexity, local Lipschitzness, convex feasible domain) and show exponential convergence of deterministic gradient descent in the tabular setting. (3) Preliminary deep-RL evaluation. Propose a deep-RL variant (log-barrier loss) and provide empirical results showing stable training and, in several environments, superior performance to standard deep-RL baselines.

\section{Related works}
Research on MDPs has traditionally been dominated by dynamic programming (DP) methods based on the Bellman equation~\citep{bertsekas1996neuro,puterman2014markov,sutton1998reinforcement}, which underpin classical RL algorithms and modern deep RL methods such as DQN~\citep{mnih2015human}, but these approaches can become less flexible in large-scale or constrained settings. As an alternative, linear programming (LP) formulations of MDPs~\citep{puterman2014markov} have been studied extensively:~\citet{de2003linear} introduced approximate linear programming (ALP), which reduces the number of decision variables by using linear function approximations;~\citet{malek2014linear} further exploited the dual LP with stochastic convex optimization methods in the average-cost setting, showing performance guarantees relative to a restricted policy class;~\citet{lakshminarayanan2017linearly} proposed the linearly relaxed ALP (LRALP), which alleviates computational load by projecting constraints into a lower-dimensional subspace while controlling approximation error. More recent developments have advanced convex formulations such as convex Q-learning~\citep{lu2021convex,lu2022convex}, logistic Q-learning~\citep{bas2021logistic}, and primal-dual algorithms~\citep{wang2016online,lee2018stochastic,serrano2020faster,neu2023efficient} with growing attention in offline RL where environment interaction is limited~\citep{nachum2020reinforcement,zhan2022offline,ozdaglar2023revisiting,gabbianelli2024offline}. In parallel, optimization and RL communities have explored barrier-based techniques: the log-barrier method is a classical tool in convex optimization~\citep{Boyd2004,nesterov2018lectures,bertsekas1999nonlinear}, and has recently been adapted for safe RL, e.g.,~\citet{zhang2024constrained,zhang2024constrained2} introduced a constrained soft actor-critic variant using a smoothed log-barrier for stable constraint handling in continuous control tasks. Despite these advances, existing LP-based approaches have not leveraged barrier functions to resolve inequality-constrained formulations, and existing barrier-based RL methods have not been applied to the LP representation of MDPs. Our work closes this gap by introducing a log-barrier reformulation of the LP approach to MDPs, yielding an unconstrained objective amenable to gradient-based optimization while retaining the structural advantages of LP formulations.

\section{Preliminaries}

\subsection{Markov decision problem}
We consider the infinite-horizon discounted Markov decision problem~\citep{puterman2014markov} and Markov decision process, where the agent sequentially takes actions to maximize cumulative discounted rewards. In a Markov decision process with the state-space ${\cal S}:=\{ 1,2,\ldots ,|{\cal S}|\}$ and action-space ${\cal A}:= \{1,2,\ldots,|{\cal A}|\}$, where $|{\cal S}|$ and $|{\cal A}|$ denote cardinalities of each set, the decision maker selects an action $a \in {\cal A}$ at the current state $s\in {\cal S}$, then the state
transits to the next state $s'\in {\cal S}$ with probability $P(s'|s,a)$, and the transition incurs a
reward $r(s,a,s') \in {\mathbb R}$, where $P(s'|s,a)$ is the state transition probability from the current state
$s\in {\cal S}$ to the next state $s' \in {\cal S}$ under action $a \in {\cal A}$, and $r(s,a,s')$ is the reward function. For convenience, we consider a deterministic reward function and simply write $r(s_k,a_k ,s_{k + 1}) =:r_{k+1},k \in \{ 0,1,\ldots \}$. A deterministic policy, $\pi :{\cal S} \to {\cal A}$, maps a state $s \in {\cal S}$ to an action $\pi(s)\in {\cal A}$. The objective of the Markov decision problem is to find an optimal policy, $\pi^*$, such that the cumulative discounted rewards over infinite time horizons is maximized, i.e., $\pi^*:= \argmax_{\pi\in \Theta} {\mathbb E}\left[\left.\sum_{k=0}^\infty {\gamma^k r_{k+1}}\right|\pi\right]$, where $\gamma \in [0,1)$ is the discount factor, $\Theta$ is the set of all deterministic policies, $(s_0,a_0,s_1,a_1,\ldots)$ is a state-action trajectory generated by the Markov chain under policy $\pi$, and ${\mathbb E}[\cdot|\pi]$ is an expectation conditioned on the policy $\pi$. Moreover, Q-function under policy $\pi$ is defined as $Q^{\pi}(s,a)={\mathbb E}\left[ \left. \sum_{k=0}^\infty {\gamma^k r_{k+1}} \right|s_0=s,a_0=a,\pi \right], (s,a)\in {\cal S} \times {\cal A}$, and the optimal Q-function is defined as $Q^*(s,a)=Q^{\pi^*}(s,a)$ for all $(s,a)\in {\cal S} \times {\cal A}$. Once $Q^*$ is known, then an optimal policy can be retrieved by the greedy policy $\pi^*(s)=\argmax_{a\in {\cal A}}Q^*(s,a)$. Throughout, we assume that the Markov decision process is ergodic so that the stationary state distribution exists. In this paper, we define an upper bound of the reward function as $|r(s,a,s')| \le {r_{\max }}, (s,a,s') \in {\cal S} \times {\cal A} \times {\cal S}$.

\subsection{LP formulation of MDP based on Q-function}
In this paper, for the sake of clarity and brevity, the majority of technical proofs are presented in Appendix. It is well known that a Markov decision problem (MDP) can be formulated as a linear program (LP)~\citep{luenberger1984linear,puterman2014markov}. While the LP formulation is typically expressed in terms of the value function~\citep{puterman2014markov}, one can also consider an LP formulation based on the Q-function. To this end, let us consider the following LP:
\begin{align}
&\mathop{\min }\limits_{Q \in {\mathbb R}^{|{\cal S}||{\cal A}|}} \mathop{\sum}\limits_{(s,a) \in {\cal S} \times {\cal A}} {\rho (s,a)Q(s,a)}\label{eq:primal-LP1}\\
&{\rm subject\,\, to}\nonumber\\
&R(s,a) + \gamma \mathop{\sum}\limits_{s' \in {\cal S}} {P(s'|s,a)Q(s',a')}  \le Q(s,a),\quad (s,a,a') \in {\cal S} \times {\cal A} \times {\cal A},\nonumber
\end{align}
where $R(s,a)$ is the expected reward conditionend on $(s,a)\in {\cal S}\times {\cal A}$, $\rho$ denotes any probability distribution over ${\cal S} \times {\cal A}$ with strictly positive support. For convenience, in this paper, we define the following Bellman operators
\begin{align*}
(TQ)(s,a):=& R(s,a) + \gamma \mathop{\sum}\limits_{s' \in {\cal S}} {P(s'|s,a){{\max }_{a' \in {\cal A}}}Q(s',a')}\\
(FQ)(s,a,a'):=& R(s,a) + \gamma \mathop{\sum}\limits_{s' \in {\cal S}} {P(s'|s,a)Q(s',a')}.
\end{align*}
Note that we can always find a strictly feasible solution of the above LP.
For instance, if we choose $Q(s,a) = \frac{{{r_{\max }} + \varepsilon }}{{1 - \gamma }} > 0$, $(s,a) \in {\cal S} \times {\cal A}$ with any $\varepsilon>0$, then $R(s,a) + \gamma \sum_{s' \in {\cal S}} {P(s'|s,a)Q(s',a')}  - Q(s,a)= R(s,a) - {r_{\max }} - \varepsilon  < 0$. The LP formulation~\cref{eq:primal-LP1} constructed on the basis of the Q-function~\citep{lee2019stochastic,lee2018stochastic} has not been extensively studied compared to the LP formulation based on the value function.
Although the LP formulation involving the Q-function was considered in~\citet{lee2019stochastic,lee2018stochastic}, it differs significantly from the LP discussed above. In particular, the LP in~\citet{lee2019stochastic,lee2018stochastic} involves not only the Q-function but also an additional value function, thereby employing a somewhat more indirect approach compared to the above formulation.
Accordingly, we begin by briefly introducing several theoretical properties and interpretations of this formulation, prior to presenting our main results. As a preliminary but fundamental result, it is straightforward to show that the solution of the above LP is unique and corresponds to the optimal Q-function, $Q^*$ (\cref{sec:app:LP-opt1}). In addition, the dual~\citep[Chapter~5]{Boyd2004} of the above LP can be derived in the following form (\cref{sec:app:dual-LP1}).
\begin{lemma}\label{lemma:dual-LP}
The dual problem of the LP~\cref{eq:primal-LP1} is given by
\begin{align}
&{\max _{\lambda  \ge 0}}\mathop{\sum}\limits_{(s,a,a') \in {\cal S} \times {\cal A} \times {\cal A}} {\lambda (s,a,a')R(s,a)}\label{eq:dual-LP1}\\
&{\rm subject\,\, to}\nonumber\\
&\mathop{\sum}\limits_{i \in {\cal A}} {\lambda (s,a,i)}  - \gamma \mathop{\sum}\limits_{(i,j) \in {\cal S} \times {\cal A}} {P(s|i,j)\lambda (i,j,a)}  = \rho (s,a),\quad \forall (s,a) \in {\cal S} \times {\cal A}.\label{eq:dual-constraint1}
\end{align}
\end{lemma}
The original LP in~\cref{eq:primal-LP1} is referred to as the primal LP (or primal problem), while the above LP~\cref{eq:dual-LP1} is called the dual LP (or dual problem). The variable $Q$ in the primal LP is referred to as the primal variable, while the variable $\lambda$ in the dual LP is called the dual variable.
We can examine several important properties and interpretations of the dual LP. For instance, the optimal dual variable $\lambda^*$ corresponds to a probability distribution, which represents the stationary state–action–next-action distribution under the optimal policy $\pi^*$ constructed from the dual variable as follows:
\begin{align*}
\pi^* ( \cdot |s): = \left[ {\begin{array}{*{20}{c}}
{\frac{\lambda^* (s,1)}{{\sum_{a \in {\cal A}} {\lambda^* (s,a)} }}}&{\frac{\lambda^* (s,2)}{{\sum_{a \in {\cal A}} {\lambda^* (s,a)} }}}& \cdots &{\frac{\lambda^* (s,|{\cal A}|)}{{\sum_{a \in {\cal A}} {\lambda^* (s,a)} }}}
\end{array}} \right],
\end{align*}
where ${\lambda ^*}(s,a): = \sum_{a' \in {\cal A}} {{\lambda ^*}(s,a,a')}$. We note that when the optimal policy is deterministic, then the above policy becomes a one-hot vector indicating the optimal action~\citep{chen2016stochastic}. Similarly, if $Q^*$ is the primal optimal solution (the solution of~\cref{eq:primal-LP1}), then $\beta^*(s) := \argmax _{a \in {\cal A}}Q^*(s,a)$ likewise induces an optimal policy. Additional details can be found in Appendix (\cref{sec:app:dual-feasi1,sec:app:dual-feasi2}).

\section{Log-barrier function approach}

In the previous section, we provided a brief discussion of the LP formulation in~\cref{eq:primal-LP1} based on the Q-function. We now turn to the main results of this paper. First of all, note that~\cref{eq:primal-LP1} involves inequality constraints. A common approach to handling such constraints is to introduce the Lagrangian together with Lagrange multipliers, or dual variables~\citep{bertsekas1999nonlinear,Boyd2004}. One then seeks the primal and dual variables through first-order primal–dual iterations~\citep{kojima1989primal}. However, this method typically suffers from slow convergence and does not, in general, guarantee convergence in many practical settings~\citep{applegate2021practical}.
Another practical and widely used approach to handling inequality constraints is the use of barrier functions, most notably the log-barrier function~\citep{Boyd2004}. This method imposes a large (or even infinite) penalty on variables that violate the inequality constraints, thereby forcing the iterates to remain within the feasible region while enabling the optimization problem to be solved.
In this paper, we apply the log-barrier function to the LP formulation in~\cref{eq:primal-LP1}, and we undertake an in-depth study and interpretation of this barrier-based approach to solving MDPs. This line of research represents an approach that has not yet been addressed in the existing literature.

The log-barrier function is a classical tool in constrained optimization used to handle inequality constraints. For a constraint of the form $g(x) \le 0$, the log-barrier introduces a penalty term $\eta \varphi(x) := -\eta \ln(-g(x))$, where $\eta >0$ is a barrier parameter. This function approaches infinity as $g(x)$ gets close to zero, thereby preventing the iterates from leaving the feasible region. As $\eta$ decreases, the solution of the barrier-augmented problem converges to the solution of the original constrained optimization problem. Moreover, one can prove that the log-barrier function $\varphi(x) := -\ln(-x)$ is strictly convex in its domain $\{x\in {\mathbb R}: x<0\}$. Using the log-barrier function, the inequality constraints can be integrated into a single objective function as follows:
\begin{align}
{f_\eta }(Q): = \mathop{\sum}\limits_{(s,a) \in {\cal S} \times {\cal A}} {Q(s,a)\rho (s,a)}  + \eta \mathop{\sum}\limits_{(s,a,a') \in {\cal S} \times {\cal A} \times {\cal A}} {w(s,a,a')\varphi \left( {(FQ)(s,a,a') - Q(s,a)} \right)},\label{eq:f}
\end{align}
where $\eta >0$ is the barrier parameter (weight) and $w(s,a,a')>0,(s,a,a')\in {\cal S} \times {\cal A}\times {\cal A}$ are weight parameters of the inequality constraints, which are introduced in order to consider random sampling approaches in stochastic implementations later. For instance, in the deep-RL variant later, $w$ should be interpreted as the empirical distribution of state–action pairs induced by the replay buffer or by mini-batch sampling.
For instance, we can set $w(s,a,a')=1$ for all $(s,a,a')\in {\cal S}\times {\cal A}\times {\cal A}$. The objective function $f_\eta$ has the following domain:
\begin{align*}
{\cal D}: = \left\{ Q \in {\mathbb R}^{|{\cal S}||{\cal A}|}:(FQ)(s,a,a') - Q(s,a) < 0,(s,a,s') \in {\cal S} \times {\cal A} \times {\cal A} \right\},
\end{align*}
which can be also seen as the strictly feasible set. Moreover, it can be shown that $\cal D$ is convex, open, bounded below, and unbounded above. Moreover, the objective function $f_\eta$ is strictly convex in $\cal D$. To proceed, let us define the level set of the objective function for $c>0$
\begin{align*}
{\cal L}_c: = \{ Q \in {\cal D}:{f_\eta }(Q) \le c\}.
\end{align*}
We can also establish that ${\cal L}_c$ is convex, closed, bounded, and $f_\eta$ is strongly convex and has Lipschitz continuous gradient in ${\cal L}_c$ with any $c>0$. The detailed theoretical analysis is given in~\cref{sec:app:property-domain1,sec:app:property-level1,sec:app:property-level2}.

The introduction of the log-barrier function enables us to reformulate the MDP as an unconstrained optimization problem. Consequently, a natural approach to address this problem is to employ gradient descent. For this purpose, we first establish the closed-form expression of the gradient presented in the following lemma, which can be proved via direct calculations.
\begin{lemma}
\label{lem:gradient}
The gradient of $f_\eta(Q)$ for $Q\in {\cal D}$ is given by
\begin{align*}
({\nabla _Q}{f_\eta }(Q))(s,a) = \rho (s,a) + \gamma \mathop{\sum}\limits_{(s',a') \in {\cal S} \times {\cal A}} {P(s|s',a'){\lambda _\eta }(s',a',a)}  - \mathop{\sum}\limits_{a' \in {\cal A}} {{\lambda _\eta }(s,a,a')}
\end{align*}
for $(s,a) \in {\cal S} \times {\cal A}$, where ${\lambda _\eta }(s,a,a'): =  \frac{\eta w(s,a,a') }{{ Q(s,a) - (FQ)(s,a,a')}}$.
\end{lemma}

Using the closed-form gradient obtained above, we can gain insight into the solution of the optimization problem.
To this end, let us assume that $\tilde Q_\eta$ is a minimizer of $f_\eta(Q)$
\begin{align*}
\tilde Q_\eta: = \argmin_{Q \in {\cal D}}f_\eta(Q).
\end{align*}
The corresponding first-order optimality condition, ${\left. {{\nabla _Q}f(Q)} \right|_{Q = \tilde Q_\eta}} = 0$, is given by
\begin{align*}
\left( {{{\left. {{\nabla _Q}f(Q)} \right|}_{Q = {{\tilde Q}_\eta }}}} \right)(s,a) =& \rho (s,a) + \gamma \mathop{\sum}\limits_{(s',a') \in {\cal S} \times {\cal A}} {{{\tilde \lambda }_\eta }(s',a',a)P(s|s',a')}  - \mathop{\sum}\limits_{a' \in S} {{{\tilde \lambda }_\eta }(s,a,a')}\\
  =& 0,\quad (s,a) \in {\cal S} \times {\cal A}.
\end{align*}
where ${\tilde \lambda _\eta }(s,a,a'): = \frac{\eta w(s,a,a') }{{\tilde Q(s,a) - (F{{\tilde Q}_\eta })(s,a,a')}}$.

As mentioned earlier, since $f_\eta$ is strictly convex over the domain, this first-order condition constitutes the necessary and sufficient condition for the optimal solution.
We can observe that the above equation is exactly identical to the equality constraints in the dual problem with ${\tilde \lambda _\eta }$ as the dual variables. In other words, ${\tilde \lambda _\eta }$ approximates the true optimal dual variable $\lambda^*$. Therefore, from the above solution $\tilde Q _\eta$, we can consider two types of policies. Interpreting the solution as an approximate solution to the LP formulation of the MDP, we may derive a greedy policy based on $\tilde Q _\eta$, as well as a policy constructed from the approximate dual variable $\tilde \lambda _\eta$ which also depends on $\tilde Q _\eta$.
Hereafter, we refer to the policy induced by the primal optimal solution $\tilde Q _\eta$ as the primal $\eta$-policy, and the policy induced by the dual optimal solution $\tilde \lambda_\eta$ as the dual $\eta$-policy.
In particular, the primal $\eta$-policy, which is deterministic, can be written as
\begin{align*}
{\tilde \beta _\eta }(s) = \argmax _{a \in {\cal A}}{\tilde Q_\eta }(s,a),
\end{align*}
and the dual $\eta$-policy, which is stochastic, can be written as
\begin{align*}
{{\tilde \pi }_\eta }( \cdot |s): = \left[ {\begin{array}{*{20}{c}}
{\frac{{{{\tilde \lambda }_\eta }(s,1)}}{{\mathop{\sum}\limits_{a \in {\cal A}} {{{\tilde \lambda }_\eta }(s,a)} }}}&{\frac{{{{\tilde \lambda }_\eta }(s,2)}}{{\mathop{\sum}\limits_{a \in {\cal A}} {{{\tilde \lambda }_\eta }(s,a)} }}}& \cdots &{\frac{{{{\tilde \lambda }_\eta }(s,|{\cal A}|)}}{{\sum\limits_{a \in {\cal A}} {{{\tilde \lambda }_\eta }(s,a)} }}}
\end{array}} \right],
\end{align*}
where ${{\tilde \lambda }_\eta }(s,a): = \sum_{a' \in {\cal A}} {{{\tilde \lambda }_\eta }(s,a,a')}$. By minimizing the above objective function, we can obtain an approximate solution of the primal LP. Since the objective function is strictly convex, the approximate solution can be efficiently found with a gradient descent algorithm.
We can prove that, under certain mild conditions, this gradient descent with a constant step-size converges exponentially to~${\tilde Q}_\eta$ (\cref{sec:app:gradient1,sec:app:gradient2}).

Next, note that the minimizer of the log-barrier-based objective function provides only an approximate solution to the LP formulation of the MDP, rather than an exact one. Nevertheless, by decreasing the barrier parameter $\eta$, the solution is permitted to approach the boundary of the inequality constraints, i.e., the equality constraints, and thus progressively converges to the exact LP solution. In the limit as $\eta \to 0$, the solution converges to the true solution of the MDP. Building on this insight, we can express the error between $\tilde Q_\eta$ and $Q^*$ as a function of $\eta$.
The following theorem establishes such an error bound between $\tilde Q_\eta$ and $Q^*$, and, in addition, presents a bound on the Bellman error corresponding to
$\tilde Q_\eta$ (\cref{sec:app:proof:bound1}).
\begin{theorem}\label{thm:bounds2}
We have
\begin{enumerate}
\item $\eta \mathop{\min }\limits_{(s,a,a') \in {\cal S} \times {\cal A} \times {\cal A}} w(s,a,a') < {\left\| {{{\tilde Q}_\eta } - {Q^*}} \right\|_\infty } \le \frac{{\eta \sum\limits_{(s,a,a') \in {\cal S} \times {\cal A} \times {\cal A}} {w(s,a,a')} }}{{{{\min }_{(s,a) \in {\cal S} \times {\cal A}}}\rho (s,a)}}$

\item $\eta (1 - \gamma )\mathop{\min }\limits_{(s,a,a') \in {\cal S} \times {\cal A} \times {\cal A}} w(s,a,a') < {\left\| {{{\tilde Q}_\eta } - T{{\tilde Q}_\eta }} \right\|_\infty } \le \frac{{(1 + \gamma )\eta \sum\limits_{(s,a,a') \in {\cal S} \times {\cal A} \times {\cal A}} {w(s,a,a')} }}{{\min_{(s,a) \in {\cal S} \times {\cal A}}}\rho (s,a)}$
\end{enumerate}
\end{theorem}

From the above result, we can establish both upper and lower bounds on the $l_\infty$-norm of the optimality error and the Bellman error. Both bounds depend on $\eta$ and are shown to be linear functions of $\eta$. Hence, as $\eta \to 0$, the upper and lower bounds of the error norms decrease linearly, i.e., ${\tilde Q}_\eta \to Q^*$. Moreover, these bounds also reveal the interplay between the bounds and other hyperparameters. In the above result, we established bounds on the norm of the error. In addition, we can also derive upper and lower bounds on the MDP objective function itself. The following theorem presents such bounds for the objective functions corresponding to the primal $\eta$-policy and the dual $\eta$-policy, expressed relative to the optimal objective value $J^{\pi^*}$ (\cref{sec:app:proof:bound2}).
\begin{theorem}\label{thm:bounds3}
We have
\begin{enumerate}
\item $J^{\pi ^*} - \eta \sum\limits_{(s,a,a') \in {\cal S} \times {\cal A} \times {\cal A}} {w(s,a,a')}  \le {J^{{\tilde \pi }_\eta }} \le J^{\pi ^*}$

\item ${J^{{\pi ^*}}} - \frac{{\eta (1 + \gamma )\sum\limits_{(s,a,a') \in {\cal S} \times {\cal A} \times {\cal A}} {w(s,a,a')} }}{{(1 - \gamma ){\min_{(s,a) \in {\cal S} \times {\cal A}}}\rho (s,a)}} \le {J^{{{\tilde \beta }_\eta }}} \le J^{\pi ^*}$

\item $-\frac{{\eta (1 + \gamma )\sum\limits_{(s,a,a') \in {\cal S} \times {\cal A} \times {\cal A}} {w(s,a,a')} }}{{(1 - \gamma ){{\min }_{(s,a) \in {\cal S} \times {\cal A}}}\rho (s,a)}} \le {J^{{\tilde \beta }_\eta }} - {J^{{\tilde \pi }_\eta}} \le \eta \sum\limits_{(s,a,a') \in {\cal S} \times {\cal A} \times {\cal A}} {w(s,a,a')}$
\end{enumerate}
\end{theorem}

The above theorem shows that the upper and lower bounds of the objective functions corresponding to the primal $\eta$-policy and the dual $\eta$-policy also depend linearly on $\eta$. Hence, as $\eta\to 0$, the objective functions associated with both the primal and dual $\eta$-policies converge to the optimal objective value $J^*$.

\section{Policy evaluation}
Although the primary focus of this paper is on computing the optimal Q-function, the proposed framework is equally applicable to policy evaluation. In this section, we therefore provide a brief account of its application to the policy evaluation setting.
To this end, let us assume a given policy $\pi$, and consider the problem of finding its corresponding Q-function $Q^\pi$. This problem can be solved through the following LP:
\begin{align*}
&\mathop{\min }\limits_{Q \in {\mathbb R}^{|{\cal S}||{\cal A}|}} \sum\limits_{(s,a) \in {\cal S} \times {\cal A}} {\rho (s,a)Q(s,a)}\\
&{\rm subject\,\, to}\quad ({T^\pi }Q)(s,a)  \le Q(s,a),\quad (s,a) \in {\cal S} \times {\cal A},
\end{align*}
where $R(s,a)$ is the expected reward conditioned on $(s,a)\in {\cal S}\times {\cal A}$, $\rho$ denotes any probability distribution over ${\cal S} \times {\cal A}$ with strictly positive support, and
\begin{align*}
({T^\pi }Q)(s,a): = R(s,a) + \gamma \sum\limits_{(s',a') \in {\cal S} \times {\cal A}} {P(s'|s,a)\pi (a'|s')Q(s',a')} ,\quad (s,a) \in {\cal S} \times {\cal A}.
\end{align*}
Then, analogous to the case of finding the optimal Q-function, we can derive the following results through a similar theoretical analysis. In particular, we can prove that the unique optimal solution of the above LP is $Q^\pi$. Moreover, let us consider the objective function with log-barrier function
\[f_\eta^\pi (Q): = \sum\limits_{(s,a) \in {\cal S} \times {\cal A}} {Q(s,a)\rho (s,a)}  + \eta \sum\limits_{(s,a) \in {\cal S} \times {\cal A}} {w(s,a)\varphi \left( {({T^\pi }Q)(s,a) - Q(s,a)} \right)} \]
where $\eta >0$ is a weight parameter and $w(s,a)>0,(s,a)\in {\cal S} \times {\cal A}$ are weight parameters of the inequality constraints.
Then, we can prove that the corresponding optimal solution $\tilde Q_\eta^\pi: = \argmin_{Q \in {\cal D}}f_\eta^\pi(Q)$ approximates $Q^\pi$. Many of the analytical results established in the previous section can be applied in a similar manner. For brevity, all related details are provided in~\cref{sec:app:policy-evaluation}.

\section{Deep RL variants}
\label{sec:DeepRL}
Thus far, we have provided a theoretical analysis of the approximate solution to the LP formulation of MDPs using the log-barrier function. Although the primary focus of this paper is on theoretical analysis with tabular setting, in this section we further explore the potential extension of the proposed framework to deep RL.
In particular, we introduce a novel DQN algorithm inspired by the idea of standard DQN~\citep{mnih2015human}. Note that when a deep neural network is used, the model becomes a nonlinear function of the parameters $\theta$, and the precise theoretical results derived for the tabular setting no longer apply directly. Nevertheless, the tabular analysis in the previous sections provides useful intuition and insights on deep RL extensions. Similar to the conventional DQN framework, we employ an experience replay buffer $D$ and mini-batch sampling $B$. Furthermore, in the definition of $f_\eta$, the probability density is replaced with samples from the mini-batch, leading to the following loss function:
\begin{align*}
L(\theta ): = \frac{1}{{|B|}}\sum\limits_{(s,a,r,s') \in B,a' \in {\cal A}} {\left[ {{Q_\theta }(s,a) + \eta \varphi (r + \gamma {Q_\theta }(s',a') - {Q_\theta }(s,a))} \right]} ,
\end{align*}
where $B$ is a mini-batch uniformly sampled from the experience replay buffer $D$, $|B|$ is the size of the mini-batch, $Q_\theta$ is a deep neural network approximation of Q-function, $\theta\in {\mathbb R}^m$ is the parameter to be determined, and $(s,a,r,s')$ is the transition sample of the state-action-reward-next state. The loss function can be seen as a stochastic approximation of $f_\eta$, where $\rho$ and $w$ can be set to probability distributions corresponding to the replay buffer. However, this stochastic approximation is generally biased because it approximates a function in which the state transition probabilities appear outside the log-barrier function. In fact, by applying Jensen’s inequality, we can show that the above loss function is essentially an unbiased stochastic approximation of an upper bounding surrogate funciton of $f_\eta$ (\cref{sec:app:upper-surrogate}). However, note that in the deterministic case, the upper surrogate function coincides with the true objective function with zero Jenson gap. In this setting, $L(\theta )$ becomes an unbiased stochastic approximation of $f_\eta$. For example, a dynamical system expressed as $s_{k+1} = f(s_k,a_k)$ is deterministic, and hence the upper bound coincides with $f_\eta$. Therefore, the loss function $L$ can be regarded as an unbiased stochastic approximation of $f_\eta$ as follows: ${f_\eta }(Q_\theta ) = \sum_{(s,a) \in {\cal S} \times {\cal A}} {{Q_\theta }(s,a)\rho (s,a)} + \eta \sum_{(s,a,a') \in {\cal S} \times {\cal A} \times {\cal A}} {w(s,a,a')\varphi \left( {r(s,a,f(s,a)) + \gamma {Q_\theta }(f(s,a),a') - {Q_\theta }(s,a)} \right)} \cong L(\theta )$. Another difference from the conventional DQN is that the algorithm proposed in this paper does not employ target variables, and hence the update of target variables is also omitted. Apart from this distinction, the remaining components are similar to those of standard DQN. The overall pseudocode of the algorithm and implementation details for stabilization of the algorithm are summarized in~\cref{appendix:DQN algo}.

Next, we briefly discuss the extension of policy evaluation, previously introduced, to the deep RL setting. In particular, we consider its potential application to DDPG~\citep{lillicrap2015continuous}. Here, the policy is deterministic, and the following loss function can be formulated:
\begin{align*}
L_{\text{critic}}(\theta ;\pi_\chi): = \frac{1}{{|B|}}\sum\limits_{(s,a,r,s') \in B} {\left[ {{Q_\theta }(s,a) + \eta \varphi (r + \gamma Q_\theta(s',\pi_\chi(s')) - {Q_\theta }(s,a))} \right]} ,
\end{align*}
where $\pi_\chi$ denotes the deterministic policy being learned, and $\chi$ is its parameter vector. A discussion similar to that of the preceding loss function applies here, and the corresponding modified DDPG algorithm along with its implementation details is provided in~\cref{appendix:DDPG algo}. Finally, we note that we can heuristically apply several alternative choices (e.g., SoftPlus) instead of the log-barrier function. While these alternatives often yield reasonable performance, the proposed method typically performed better in our experiments.

\section{Experiments}
\label{sec:experiments}

To validate the performance of the proposed method, we conduct experiments in both discrete and continuous control environments. We implement our deep RL variant as described in Section~\ref{sec:DeepRL}, naming our algorithms \textbf{Log-barrier DQN} and \textbf{Log-barrier DDPG}, respectively. For discrete control tasks, we compare our algorithm with the standard DQN~\citep{mnih2015human}, and for continuous control tasks, we benchmark it against the standard DDPG~\citep{lillicrap2015continuous}. Additional algorithmic details are provided in~\cref{appendix:DQN algo} and~\cref{appendix:DDPG algo}, with detailed hyperparameters listed in~\cref{appendix:hyperparameters}.

\paragraph{Log-Barrier DQN} For our DQN experiments, we chose five environments form the Gymnasium library (Acrobot-v1, CartPole-v1, LunarLander-v3, MountainCar-v0, and Pendulum-v1). For MountainCar-v0, we replace the original sparse reward with a dense shaping reward, and for Pendulum-v1 we discretize the continuous action space to enable DQN training. As shown in figure~\cref{fig:dqn-exp}, the results reveal a remarkable point that the Log-barrier DQN demonstrates rapid adaptation and significant stability in the CartPole environment, where the agent's survival is threatened by a critical angle criterion. We hypothesize that this is due to the sharp decision boundary between states where the agent survives and states where the episode is terminated. Standard DQN, which relies on Bellman updates with an mean square error (MSE) loss, can suffer from error propagation across this boundary. In contrast, the proposed approach, which uses the LP form, can globally mitigate this hazard. Instead of estimating value from neighboring states, our method directly minimizes its objective while satisfying the LP constraints, which are transformed into an unconstrained problem using a barrier function.

\begin{figure}[h!]
\centering
\includegraphics[width=14cm,height=3.5cm]{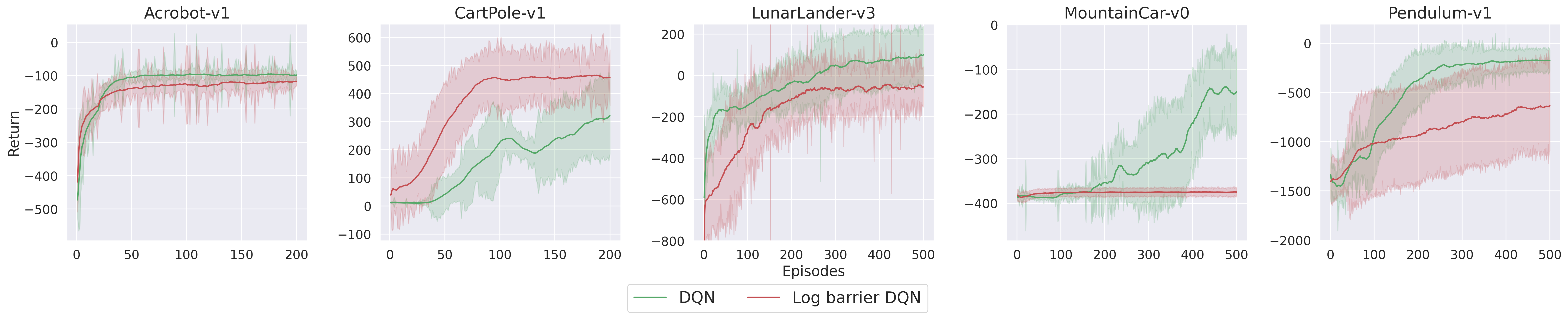}
\caption{Learning curves comparing the Log-barrier DQN and standard DQN on the Gymnasium control environments. Each curve represents the average return over 10 random seeds, with the shaded area indicating one standard deviation from the mean.}
\label{fig:dqn-exp}
\end{figure}

\paragraph{Log-Barrier DDPG} In the continuous control experiments with DDPG, the proposed method demonstrated superior performance on four complex MuJoCo environments (Ant, Walker2d, HalfCheetah, Humanoid), while showing no significant advantage on the simpler Hopper task, as shown in~\cref{fig:ddpg-exp}. We attribute this success to a fundamental property of critic update mechanism of the proposed algorithm. We conjecture that the core of the advantage arises from the fact that the LP form is inherently a minimization objective, which naturally counteracts the Q-value overestimation bias prevalent in actor-critic methods.
Standard DDPG critics learn by minimizing the MSE to a target value, a process that simply follows the target without regard for its potential bias. If the target is inflated, the critic learns an inflated value, leading to a feedback loop of compounding overestimation. In contrast, the proposed approach minimizes the value function itself, subject to the constraints of Bellman consistency. This systematic search for the tightest and lowest possible Q-values that still satisfy the dynamics acts as a powerful, implicit regularizer against optimistic value estimates. This results in a more conservative and stable critic, which provides a more reliable gradient to the actor, leading to superior final performance. Consequently, this allows the proposed method to overcome the limitations of standard DDPG, successfully solving environments such as Ant and Humanoid that were previously thought to be beyond its capabilities.
\begin{figure}[h!]
\centering
\includegraphics[width=14cm,height=3.5cm]{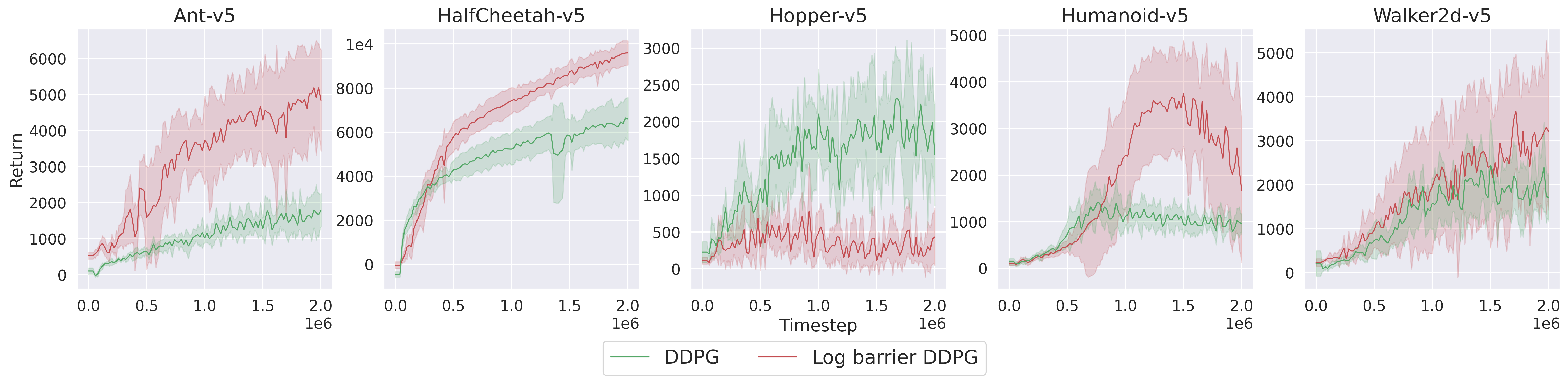}
\caption{Learning curves comparing the Log-barrier DDPG and standard DDPG on the Mujoco continuous control environments. Each curve represents the average return over 8 random seeds, with the shaded area indicating one standard deviation from the mean.}
\label{fig:ddpg-exp}
\end{figure}

\section{Conclusion}
We developed a theoretical framework for solving MDPs via their LP formulation using a log-barrier term. Reformulating the LP into a single objective $f_\eta$, we show that gradient descent efficiently produces approximate solutions $\tilde Q_\eta$ and prove error bounds (including $\|\tilde Q_\eta - Q^* \|_\infty$) that scale linearly with the barrier parameter $\eta$. We also characterize primal and dual approximate solutions, their induced policies, and prove structural properties of $f_\eta$ (e.g., convexity, local strong convexity/Lipschitzness) together with exponential convergence of gradient descent in the tabular setting. Practically, we derive a novel log-barrier loss for deep RL and evaluate it in DQN and DDPG: the method matches standard DQN in most cases and outperforms conventional DDPG in several tasks. While experiments are limited in scope and the approach requires careful hyperparameter tuning, the empirical results support the promise of the log-barrier formulation; fuller large-scale validation and robustness improvements remain important directions for future work.

\section*{Acknowledgments}
The work was supported in by the Institute of Information Communications Technology Planning Evaluation (IITP) funded by the Korea government under Grant 2022-0-00469.

\bibliographystyle{apalike}
\bibliography{reference}

@article{applegate2021practical,
  title={Practical large-scale linear programming using primal-dual hybrid gradient},
  author={Applegate, David and D{\'\i}az, Mateo and Hinder, Oliver and Lu, Haihao and Lubin, Miles and O'Donoghue, Brendan and Schudy, Warren},
  journal={Advances in Neural Information Processing Systems},
  volume={34},
  pages={20243--20257},
  year={2021}
}

@article{de2003linear,
  title={The linear programming approach to approximate dynamic programming},
  author={De Farias, Daniela Pucci and Van Roy, Benjamin},
  journal={Operations research},
  volume={51},
  number={6},
  pages={850--865},
  year={2003}
}

@inproceedings{malek2014linear,
  title={Linear programming for large-scale Markov decision problems},
  author={Malek, Alan and Abbasi-Yadkori, Yasin and Bartlett, Peter},
  booktitle={International conference on machine learning},
  pages={496--504},
  year={2014}
}

@article{lakshminarayanan2017linearly,
  title={A linearly relaxed approximate linear program for Markov decision processes},
  author={Lakshminarayanan, Chandrashekar and Bhatnagar, Shalabh and Szepesv{\'a}ri, Csaba},
  journal={IEEE Transactions on Automatic control},
  volume={63},
  number={4},
  pages={1185--1191},
  year={2017}
}

@article{ghate2013linear,
  title={A linear programming approach to nonstationary infinite-horizon Markov decision processes},
  author={Ghate, Archis and Smith, Robert L},
  journal={Operations Research},
  volume={61},
  number={2},
  pages={413--425},
  year={2013}
}

@article{ying2020note,
  title={A note on optimization formulations of Markov decision processes},
  author={Ying, Lexing and Zhu, Yuhua},
  journal={arXiv preprint arXiv:2012.09417},
  year={2020}
}

@inproceedings{sikchi2024dual,
  title={Dual RL: Unification and New Methods for Reinforcement and Imitation Learning},
  author={Sikchi, H and Zhang, A and Niekum, S},
  booktitle={International Conference on Learning Representations},
  year={2024},
  organization={International Conference on Learning Representations}
}

@inproceedings{wang2016online,
  title={An online primal-dual method for discounted Markov decision processes},
  author={Wang, Mengdi and Chen, Yichen},
  booktitle={2016 IEEE 55th Conference on Decision and Control (CDC)},
  pages={4516--4521},
  year={2016}
}

@article{chen2016stochastic,
  title={Stochastic primal-dual methods and sample complexity of reinforcement learning},
  author={Chen, Yichen and Wang, Mengdi},
  journal={arXiv preprint arXiv:1612.02516},
  year={2016}
}

@inproceedings{chen2018scalable,
  title={Scalable bilinear pi learning using state and action features},
  author={Chen, Yichen and Li, Lihong and Wang, Mengdi},
  booktitle={International Conference on Machine Learning},
  pages={834--843},
  year={2018}
}

@inproceedings{lee2019stochastic,
  title={Stochastic primal-dual Q-learning algorithm for discounted MDPs},
  author={Lee, Donghwan and He, Niao},
  booktitle={2019 american control conference (acc)},
  pages={4897--4902},
  year={2019}
}

@article{lee2018stochastic,
  title={Stochastic primal-dual Q-learning},
  author={Lee, Donghwan and He, Niao},
  journal={arXiv preprint arXiv:1810.08298},
  year={2018}
}

@inproceedings{serrano2020faster,
  title={Faster saddle-point optimization for solving large-scale Markov decision processes},
  author={Serrano, Joan Bas and Neu, Gergely},
  booktitle={Learning for Dynamics and Control},
  pages={413--423},
  year={2020}
}

@inproceedings{neu2023efficient,
  title={Efficient global planning in large MDPs via stochastic primal-dual optimization},
  author={Neu, Gergely and Okolo, Nneka},
  booktitle={International Conference on Algorithmic Learning Theory},
  pages={1101--1123},
  year={2023}
}

@article{nachum2020reinforcement,
  title={Reinforcement learning via fenchel-rockafellar duality},
  author={Nachum, Ofir and Dai, Bo},
  journal={arXiv preprint arXiv:2001.01866},
  year={2020}
}

@inproceedings{bas2021logistic,
  title={Logistic Q-learning},
  author={Bas-Serrano, Joan and Curi, Sebastian and Krause, Andreas and Neu, Gergely},
  booktitle={International conference on artificial intelligence and statistics},
  pages={3610--3618},
  year={2021}
}

@inproceedings{lu2021convex,
  title={Convex Q-learning},
  author={Lu, Fan and Mehta, Prashant G and Meyn, Sean P and Neu, Gergely},
  booktitle={2021 American Control Conference (ACC)},
  pages={4749--4756},
  year={2021}
}

@inproceedings{lu2022convex,
  title={Convex analytic theory for convex Q-learning},
  author={Lu, Fan and Mehta, Prashant G and Meyn, Sean P and Neu, Gergely},
  booktitle={2022 IEEE 61st Conference on Decision and Control (CDC)},
  pages={4065--4071},
  year={2022}
}

@inproceedings{ozdaglar2023revisiting,
  title={Revisiting the linear-programming framework for offline rl with general function approximation},
  author={Ozdaglar, Asuman E and Pattathil, Sarath and Zhang, Jiawei and Zhang, Kaiqing},
  booktitle={International Conference on Machine Learning},
  pages={26769--26791},
  year={2023}
}

@inproceedings{zhan2022offline,
  title={Offline reinforcement learning with realizability and single-policy concentrability},
  author={Zhan, Wenhao and Huang, Baihe and Huang, Audrey and Jiang, Nan and Lee, Jason},
  booktitle={Conference on Learning Theory},
  pages={2730--2775},
  year={2022}
}

@inproceedings{gabbianelli2024offline,
  title={Offline primal-dual reinforcement learning for linear mdps},
  author={Gabbianelli, Germano and Neu, Gergely and Papini, Matteo and Okolo, Nneka M},
  booktitle={International Conference on Artificial Intelligence and Statistics},
  pages={3169--3177},
  year={2024}
}

@inproceedings{kamoutsi2021efficient,
  title={Efficient performance bounds for primal-dual reinforcement learning from demonstrations},
  author={Kamoutsi, Angeliki and Banjac, Goran and Lygeros, John},
  booktitle={International Conference on Machine Learning},
  pages={5257--5268},
  year={2021}
}

@article{zhang2024constrained,
  title={Constrained Reinforcement Learning for Safe Heat Pump Control},
  author={Zhang, Baohe and Frison, Lilli and Brox, Thomas and B{\"o}decker, Joschka},
  journal={arXiv preprint arXiv:2409.19716},
  year={2024}
}

@article{zhang2024constrained2,
  title={Constrained reinforcement learning with smoothed log barrier function},
  author={Zhang, Baohe and Zhang, Yuan and Frison, Lilli and Brox, Thomas and B{\"o}decker, Joschka},
  journal={arXiv preprint arXiv:2403.14508},
  year={2024}
}

@article{mnih2015human,
  title={Human-level control through deep reinforcement learning},
  author={Mnih, Volodymyr and Kavukcuoglu, Koray and Silver, David and Rusu, Andrei A and Veness, Joel and Bellemare, Marc G and Graves, Alex and Riedmiller, Martin and Fidjeland, Andreas K. and Ostrovski, Georg and others},
  journal={Nature},
  volume={518},
  number={7540},
  pages={529},
  year={2015}
}

@article{lillicrap2015continuous,
  title={Continuous control with deep reinforcement learning},
  author={Lillicrap, Timothy P and Hunt, Jonathan J and Pritzel, Alexander and Heess, Nicolas and Erez, Tom and Tassa, Yuval and Silver, David and Wierstra, Daan},
  journal={arXiv preprint arXiv:1509.02971},
  year={2015}
}

@incollection{kojima1989primal,
  title={A primal-dual interior point algorithm for linear programming},
  author={Kojima, Masakazu and Mizuno, Shinji and Yoshise, Akiko},
  booktitle={Progress in Mathematical Programming: Interior-point and related methods},
  pages={29--47},
  year={1989},
  publisher={Springer}
}

@book{sutton1998reinforcement,
  title={Reinforcement learning: {A}n introduction},
  author={Sutton, Richard S and Barto, Andrew G},
  year={1998},
  publisher={MIT Press}
}

@book{Boyd2004,
 author = {S. Boyd and L. Vandenberghe},
 title = {Convex Optimization},
 publisher = {Cambridge University Press},
 year = {2004}
}

@book{luenberger1984linear,
  title={Linear and nonlinear programming},
  author={Luenberger, David G and Ye, Yinyu and others},
  volume={2},
  year={1984},
  publisher={Springer}
}

@book{bertsekas1999nonlinear,
  title={Nonlinear programming},
  author={Bertsekas, Dimitri P},
  year={1999},
  publisher={Athena scientific Belmont}
}

@book{puterman2014markov,
  title={Markov decision processes: discrete stochastic dynamic programming},
  author={Puterman, Martin L},
  year={2014},
  publisher={John Wiley \& Sons}
}

@book{bertsekas1996neuro,
  title={Neuro-dynamic programming},
  author={Bertsekas, Dimitri P and Tsitsiklis, John N},
  year={1996},
  publisher={Athena Scientific Belmont, MA}
}

@book{bertsekas2012dynamic2,
  title={Dynamic programming and optimal control: Volume II},
  author={Bertsekas, Dimitri},
  year={2012},
  publisher={Athena scientific}
}

@book{nesterov2018lectures,
  title={Lectures on convex optimization},
  author={Nesterov, Yurii},
  volume={137},
  year={2018},
  publisher={Springer}
}

@article{jaakkola1993convergence,
  title={Convergence of stochastic iterative dynamic programming algorithms},
  author={Jaakkola, Tommi and Jordan, Michael and Singh, Satinder},
  journal={Advances in neural information processing systems},
  volume={6},
  year={1993}
}

@article{tsitsiklis1994asynchronous,
  title={Asynchronous stochastic approximation and Q-learning},
  author={Tsitsiklis, John N},
  journal={Machine learning},
  volume={16},
  number={3},
  pages={185--202},
  year={1994},
  publisher={Springer}
}

@article{munos2008finite,
  title={Finite-Time Bounds for Fitted Value Iteration.},
  author={Munos, R{\'e}mi and Szepesv{\'a}ri, Csaba},
  journal={Journal of Machine Learning Research},
  volume={9},
  number={5},
  year={2008}
}

\clearpage
\appendix

%\begin{titlepage}
%    \centering
%    \vspace*{5cm}
%    {\Huge \textbf{Supplemental Material:}\par}
%    \vspace{1cm}
%    {\Large ``Analysis of approximate linear programming solution to Markov decision problem with log barrier function''\par}
%\end{titlepage}

% {   \centering
 %   {\Huge \textbf{Supplemental Material:}\par}
 %   \vspace{0.1cm}
 %   {\Large ``Analysis of approximate linear programming solution to Markov decision problem with log barrier function''\par}}

 {   \centering
    {\Huge \textbf{Appendix}\par}}

{\color{red}

}

\section{Appendix:~\cref{lemma:optimality}}\label{sec:app:LP-opt1}
In the next lemma, we establish that the optimal solution to the LP formulation of the MDP~\cref{eq:primal-LP1}, as presented in this paper, is precisely the optimal Q-function, $Q^*$.
\begin{lemma}\label{lemma:optimality}
The optimal solution of the LP~\cref{eq:primal-LP1} is unique and given by $Q^*$
\end{lemma}
\begin{proof}
Let us consider the following LP defined also in~\cref{eq:primal-LP1}:
\begin{align*}
&\mathop{\min }\limits_{Q \in {\mathbb R}^{|{\cal S}||{\cal A}|}} \sum\limits_{(s,a) \in {\cal S} \times {\cal A}} {\rho (s,a)Q(s,a)}\\
&{\rm subject\,\, to}\\
&R(s,a) + \gamma \sum\limits_{s' \in {\cal S}} {P(s'|s,a)Q(s',a')}  \le Q(s,a),\quad (s,a,a') \in {\cal S} \times {\cal A} \times {\cal A},
\end{align*}
where $\rho$ denotes any probability distribution over ${\cal S} \times {\cal A}$ with strictly positive support.
Let us assume that the optimal solution of the above LP is $\hat Q$. We want to prove that $\hat Q = Q^*$. First of all, since $\hat Q$ is feasible, it satisfies
\begin{align*}
(F\hat Q)(s,a,a') \le \hat Q(s,a),\quad (s,a,a') \in {\cal S} \times {\cal A} \times {\cal A}.
\end{align*}
This implies
\begin{align*}
(T\hat Q)(s,a) \le \hat Q(s,a),\quad (s,a) \in {\cal S} \times {\cal A},
\end{align*}
and vice versa. Equivalently, it can be written as $T\hat Q \le \hat Q$.
Because $T$ is a monotone operator~\citep{bertsekas1996neuro}, we have
\begin{align*}
\mathop{\lim }\limits_{k \to \infty } T^k \hat Q \le \hat Q \Rightarrow Q^* \le \hat Q.
\end{align*}
Moreover, since $T{Q^*} = {Q^*}$, ${Q^*}$ is a feasible solution, and hence, $\hat Q \le {Q^*}$.
Therefore, ${Q^*} \le \hat Q \le {Q^*}$, which implies that $\hat Q = Q^*$.

\end{proof}

\section{Appendix: proof of~\cref{lemma:dual-LP}}\label{sec:app:dual-LP1}

Let us consider the Lagrangian function~\citep{Boyd2004}
\begin{align*}
L(Q,\lambda ) =& \sum\limits_{(s,a) \in {\cal S} \times {\cal A}} {\rho (s,a)Q(s,a)} \\
& + \sum\limits_{(s,a,a') \in {\cal S}\times {\cal A} \times {\cal A}} {\lambda (s,a,a')(R(s,a) + \gamma \sum\limits_{s'\in {\cal S}} P(s'|s,a)Q(s',a') - Q(s,a) )},
\end{align*}
where $\lambda$ is the Lagrangian multiplier vector (or dual variable). Next, the function can be written by
\begin{align*}
L(Q,\lambda ) =& \sum\limits_{(s,a,a') \in {\cal S} \times {\cal A} \times {\cal A}} {\lambda (s,a,a')R(s,a)} \\
& + \sum\limits_{(s,a) \in {\cal S} \times {\cal A}} {Q(s,a)\left\{ {\sum\limits_{i \in {\cal A}} {\lambda (s,a,i)}  - \gamma \sum\limits_{(i,j) \in {\cal S} \times {\cal A}} {P(s|i,j)\lambda (i,j,a)}  - \rho (s,a)} \right\}}
\end{align*}

One can observe that the optimal value of the dual problem ${\max _{\lambda  \ge 0}}{\min _{Q \in {\mathbb R}^{|{\cal S}||{\cal A}|}}}L(Q,\lambda )$ is finite only when
\begin{align*}
\sum\limits_{i \in {\cal A}} {\lambda (s,a,i)}  - \gamma \sum\limits_{(i,j) \in {\cal S} \times {\cal A}} {P(s|i,j)\lambda (i,j,a)}  - \rho (s,a) = 0,\quad \forall (s,a) \in {\cal S} \times {\cal A}.
\end{align*}
Therefore, one gets the following dual problem:
\begin{align*}
&{\max _{\lambda  \ge 0}}\sum\limits_{(s,a,a') \in {\cal S} \times {\cal A} \times {\cal A}} {\lambda (s,a,a')R(s,a)}\\
&{\rm subject\,\, to}\\
&\sum\limits_{i \in {\cal A}} {\lambda (s,a,i)}  - \gamma \sum\limits_{(i,j) \in {\cal S} \times {\cal A}} {P(s|i,j)\lambda (i,j,a)}  = \rho (s,a),\quad \forall (s,a) \in {\cal S} \times {\cal A}.
\end{align*}
This completes the proof.

\section{Appendix:~\cref{prop:dual-property1}}\label{sec:app:dual-feasi1}

We examine several important properties and interpretations of the dual LP.
\begin{theorem}\label{prop:dual-property1}
Suppose $\lambda (s,a,a'),\forall (s,a,a') \in {\cal S} \times {\cal A} \times {\cal A}$ is any dual feasible solution (any solutions satisfying the dual constraints in~\cref{eq:dual-constraint1}).
Then, we obtain the following results:
\begin{enumerate}
\item $(1-\gamma)\lambda$ is a probability distribution, i.e.,
\[(1-\gamma)\lambda (s,a,a') \ge 0,\quad \forall (s,a,a') \in {\cal S} \times {\cal A} \times {\cal A},\quad \sum\limits_{(s,a,a') \in {\cal S} \times {\cal A} \times {\cal A}} {(1-\gamma)\lambda (s,a,a')}  = 1\]

\item Moreover, let us define the marginal distributions
\begin{align}
\lambda (s,a): = \sum\limits_{a' \in {\cal A}} {\lambda (s,a,a')} ,\quad \lambda (s): = \sum\limits_{a \in {\cal A}} {\lambda (s,a)} ,\quad \rho (s): = \sum\limits_{a' \in {\cal A}} {\rho (s,a)}.\label{eq:margials1}
\end{align}
Then, we have
\begin{align}
\lambda (s') - \gamma \sum\limits_{s \in {\cal S}} {{P^\pi }(s'|s)\lambda (s)}  = \rho (s'),\quad \forall s' \in {\cal S},\label{eq:lambda-equation1}
\end{align}
where
\begin{align}
\pi ( \cdot |s): = \left[ {\begin{array}{*{20}{c}}
{\frac{{\lambda (s,1)}}{{\sum_{a \in {\cal A}} {\lambda (s,a)} }}}&{\frac{{\lambda (s,2)}}{{\sum_{a \in {\cal A}} {\lambda (s,a)} }}}& \cdots &{\frac{\lambda (s,|{\cal A}|)}{{\sum_{a \in {\cal A}} {\lambda (s,a)} }}}
\end{array}} \right]\label{eq:dual-policy1}
\end{align}
and
\[{P^\pi }(s'|s): = \sum_{a \in {\cal A}} {P(s'|s,a)\pi (a|s)} \]
is the probability of the transition from $s$ to $s'$ under the policy $\pi$.

\item If the marginal $\rho (s): = \sum_{a \in {\cal A}} {\rho (s,a)} ,s \in {\cal S}$ represents the initial state distribution, then $(1-\gamma)\lambda ( \cdot )$ is the discounted state occupation probability distribution defined as
\[(1 - \gamma )\lambda (s) = {\mu ^\pi }(s) = (1 - \gamma )\sum\limits_{k = 0}^\infty  {\gamma ^k {\mathbb P}[s_k = s|\pi ,\rho ]} \]

\item If the marginal $\rho (s): = \sum_{a \in {\cal A}} {\rho (s,a)} ,s \in {\cal S}$ represents the initial state distribution, then the dual objective function satisfies
\[\sum\limits_{(s,a,a') \in {\cal S} \times {\cal A} \times {\cal A}} {\lambda (s,a,a')R(s,a)}  = \sum\limits_{s \in {\cal S}} {{V^\pi }(s)\rho (s)}  = {J^\pi }\]
\end{enumerate}
\end{theorem}
\cref{prop:dual-property1} provides an interpretation of any feasible dual solution. In particular, analogous to the LP formulation~\citep{puterman2014markov} based on the value function, the dual variable in the Q-function-based LP formulation corresponds to a probability distribution. It represents the stationary state–action–next-action distribution under the policy $\pi$ in~\cref{eq:dual-policy1} constructed from the dual variable.
Moreover, we can prove that the marginal distribution $\lambda(s)$ in~\cref{eq:margials1} corresponds to the discounted state-occupation probability distribution. Building on this observation, it can be shown that the dual objective function coincides with the objective value attained by the policy $\pi$ in~\cref{eq:dual-policy1}.
The preceding discussions are valid for any dual feasible variable $\lambda$, and, of course, the same properties extend to the dual optimal solution $\lambda^*$. In particular, if $\lambda^*$ is the dual optimal solution, then the associated policy~\cref{eq:dual-policy1} with $\lambda$ replaced by $\lambda^*$ coincides with an optimal policy. Similarly, if $Q^*$ is the primal optimal solution (the solution of~\cref{eq:primal-LP1}), then $\beta^*(s) := \argmax _{a \in {\cal A}}Q^*(s,a)$ likewise induces an optimal policy.

\begin{proof}
For the first statement, summing the dual constraints in~\cref{eq:dual-constraint1} over $(s,a)\in {\cal S}\times {\cal A}$, we get
\begin{align*}
\sum\limits_{(s,a,i) \in {\cal S} \times {\cal A} \times {\cal A}} {\lambda (s,a,i)}  - \gamma \sum\limits_{(s,a) \in {\cal S} \times {\cal A},(i,j) \in {\cal S} \times {\cal A}} {P(s|i,j)\lambda (i,j,a)}  = \sum\limits_{(s,a) \in {\cal S} \times {\cal A}} {\rho (s,a)},
\end{align*}
which leads to
\begin{align*}
(1 - \gamma )\sum\limits_{a \in {\cal A},(i,j) \in {\cal S} \times {\cal A}} {\lambda (i,j,a)}  = 1.
\end{align*}
It proves the first statement.
For the second statement, summing the dual constraints in~\cref{eq:dual-constraint1} over $a'\in {\cal A}$ leads to
\begin{align*}
\lambda (s,a) - \gamma \sum\limits_{(i,j) \in {\cal S} \times {\cal A}} {P(s|i,j)\lambda (i,j,a)}  = \rho (s,a).
\end{align*}
Moreover, summing the above equation over $a\in {\cal A}$ leads to
\begin{align*}
\sum\limits_{a \in {\cal A}} {\lambda (s',a) - \gamma \sum\limits_{(s,a) \in {\cal S} \times {\cal A}} {P(s'|s,a)\lambda (s,a)} }  = \rho (s'),
\end{align*}
which can be further written as
\begin{align*}
&\sum\limits_{a \in A} {\lambda (s',a)}  - \gamma \sum\limits_{(s,a) \in {\cal S} \times {\cal A}} {P(s'|s,a)\frac{{\lambda (s,a)}}{{\sum\limits_{i \in {\cal A}} {\lambda (s,i)} }}\sum\limits_{j \in {\cal A}} {\lambda (s,j)} }\\
=& \sum\limits_{a \in {\cal A}} {\lambda (s',a)}  - \gamma \sum\limits_{(s,a) \in {\cal S} \times {\cal A}} {P(s'|s,a)\pi (a|s)\sum\limits_{i \in {\cal A}} {\lambda (s,i)} } \\
=& \sum\limits_{a \in {\cal A}} {\lambda (s',a)}  - \gamma \sum\limits_{s \in {\cal S}} {P^\pi(s'|s)\sum\limits_{a \in {\cal A}} {\lambda (s,a)} }\\
=& \lambda (s') - \gamma \sum\limits_{s \in {\cal S}} {{P^\pi }(s'|s)\lambda (s)} \\
=& \rho (s'),
\end{align*}
and this proves the second statement.

To prove the third statement, we can first consider the vectorized form of~\cref{eq:lambda-equation1}
\begin{align*}
\lambda  - \gamma {({P^\pi })^T}\lambda  = \rho,
\end{align*}
where
\begin{align*}
\lambda  = \left[ {\begin{array}{*{20}{c}}
{\lambda (1)}\\
 \vdots \\
{\lambda (|{\cal S}|)}
\end{array}} \right] \in {\mathbb R}^{|{\cal S}|},\rho  = \left[ {\begin{array}{*{20}{c}}
{\rho (1)}\\
 \vdots \\
{\rho (|S|)}
\end{array}} \right] \in {\mathbb R}^{|{\cal S}|},{P^\pi } = \left[ {\begin{array}{*{20}{c}}
{{P^\pi }(1|1)}& \cdots &{{P^\pi }(|{\cal S}||1)}\\
 \vdots & \ddots & \vdots \\
{{P^\pi }(1||S|)}& \cdots &{{P^\pi }(|{\cal S}|||{\cal S}|)}
\end{array}} \right] \in {\mathbb R}^{|{\cal S}| \times |{\cal S}|},
\end{align*}
which leads to
\begin{align*}
{\lambda ^T} = {\rho ^T} + \gamma {\rho ^T}{P_\pi } + {\gamma ^2}{\rho ^T}P_\pi ^2 +  \cdots,
\end{align*}
which can be written by
\[\lambda (s) = \sum\limits_{k = 0}^\infty  { \gamma ^k {\mathbb P}[s_k = s|\pi ,\rho ]}  = \frac{1}{1 - \gamma }{\mu ^\pi }(s),\quad s \in {\cal S}.\]
This completes the proof of the third statement.

The last statement can be proved through the following identities:
\begin{align*}
\sum\limits_{(s,a,a') \in {\cal S} \times {\cal A} \times {\cal A}} {\lambda (s,a,a')R(s,a)}  =& \sum\limits_{(s,a) \in {\cal S} \times {\cal A}} {\lambda (s,a)R(s,a)}\\
=& \sum\limits_{(s,a) \in {\cal S} \times {\cal A}} {\pi (a|s)R(s,a)} \lambda (s)\\
=& \sum\limits_{s \in {\cal S}} {R^\pi (s)} \lambda (s)\\
=& \sum\limits_{s \in {\cal S}} {{V^\pi }(s)\rho (s)}\\
=& {J^\pi },
\end{align*}
where $R^\pi$ denotes the expected reward under the policy $\pi$, and $V^\pi$ denotes the value function under $\pi$.
\end{proof}

\section{Appendix:~\cref{prop:dual-property12}}\label{sec:app:dual-feasi2}

Based on~\cref{prop:dual-property1}, we can readily obtain the following corollary.
\begin{corollary}\label{prop:dual-property12}
Suppose that $Q^*$ is the primal optimal solution, and $\lambda^*$ is the dual optimal solution.

Then, we obtain the following results:
\begin{enumerate}
\item $\beta^*$ defined in the following is an optimal policy:
\begin{align*}
\beta^*(s) = \argmax _{a \in {\cal A}}Q^*(s,a).
\end{align*}

\item Suppose that the marginal $\rho (s) = \sum\limits_{a \in {\cal A}} {\rho (s,a)}$ represents the initial state distribution and $\varphi^* (a|s): = \frac{{\rho (s,a)}}{{\sum\limits_{a \in {\cal A}} {\rho (s,a)} }}$ represents an optimal policy, i.e., $\beta^*$. Then, $\pi^*$ defined in the following is an optimal policy:
\begin{align*}
\pi^* ( \cdot |s): = \left[ {\begin{array}{*{20}{c}}
{\frac{\lambda^* (s,1)}{\sum_{a \in {\cal A}} {\lambda^* (s,a)} }}&{\frac{\lambda^* (s,2)}{\sum_{a \in {\cal A}} {\lambda^* (s,a)} }}& \cdots &{\frac{\lambda^* (s,|{\cal A}|)}{\sum_{a \in {\cal A}} \lambda^* (s,a) }}
\end{array}} \right]
\end{align*}

%\item Both policies are identical: $\beta^* = \pi^*$
\end{enumerate}
\end{corollary}
\begin{proof}
The first statement can be directly obtained from~\cref{lemma:optimality} (the primal optimal solution is the optimal Q-function $Q^*$).
For the second statement, we first note that the optimal primal solution $Q^*$ does not depend on $\rho$, while the optimal dual solution $\lambda^*$ depends on $\rho$. Suppose that the marginal
\begin{align*}
\rho (s) = \sum\limits_{a \in A} {\rho (s,a)}
\end{align*}
represents the initial state distribution and
\begin{align*}
\varphi^* (a|s): = \frac{{\rho (s,a)}}{{\sum\limits_{a \in {\cal A}} {\rho (s,a)} }}
\end{align*}
represents an optimal policy. using~\cref{prop:dual-property1} and strong duality, one can derive
\begin{align*}
J^{\varphi ^*}: =& \sum\limits_{(s,a) \in {\cal S} \times {\cal A}} {\rho (s){\varphi ^*}(a|s){Q^*}(s,a)}\\
=& \sum\limits_{(s,a) \in {\cal S} \times {\cal A}} {\rho (s,a) Q^*(s,a)}\\
=& \sum\limits_{(s,a,a') \in {\cal S} \times {\cal A} \times {\cal A}} {\lambda^* (s,a,a')R(s,a)}\\
=& \sum\limits_{(s,a) \in {\cal S} \times {\cal A}} {\lambda^* (s,a)R(s,a)}\\
=&\sum\limits_{(s,a) \in {\cal S} \times {\cal A}} {\pi^*(a|s)R(s,a)\lambda^*(s)}\\
=& \sum\limits_{s \in {\cal S}} {R^{\pi^*} (s)\lambda^* (s)}\\
=& \sum\limits_{s \in {\cal S}} {V^{\pi^*}(s)\rho (s)} \\
=& J^{\pi^*}
\end{align*}
This implies that $\pi^*$ is an optimal policy as well.

%To prove the last statement, i.e., $\beta^*=\pi^*$, let us consider the corresponding Lagrangian function
%Let us consider the Lagrangian function
%\begin{align*}
%L(Q,\lambda ) =& \sum\limits_{(s,a) \in {\cal S} \times {\cal A}} {\rho (s,a)Q(s,a)} \\
%& + \sum\limits_{(s,a,a') \in {\cal S}\times {\cal A} \times {\cal A}} {\lambda (s,a,a')(R(s,a) + \gamma \sum\limits_{s'\in {\cal S}} P(s'|s,a)Q(s',a') - Q(s,a) )},
%\end{align*}
%where $\lambda$ is the Lagrangian multiplier vector (or dual variable). The KKT condition of the LP form is as follows:
%\paragraph{Stationary condition:}
%\[\frac{{\partial L(Q^*,\lambda^* )}}{{\partial Q(s,a)}} = \sum\limits_{i \in A} {\lambda^* (s,a,i)}  - \gamma \sum\limits_{(i,j) \in S \times A} {P(s|i,j)\lambda^* (i,j,a)}  - \rho (s,a) = 0,\quad \forall (s,a) \in S \times A\]

%\paragraph{Primal feasibility condition:}
%\[R(s,a) + \gamma \sum\limits_{s' \in S} {P(s'|s,a)Q^*(s',a')} -Q^*(s,a) \le 0,\quad (s,a,a') \in S \times A \times A\]

%\paragraph{Dual feasibility condition:}
%\[\lambda^* (s,a,a') \ge 0,\quad (s,a,a') \in S \times A \times A\]

%\paragraph{Complementary slackness condition:}
%\[\lambda^* (s,a,a')\left\{ {R(s,a) + \gamma \sum\limits_{s' \in S} {P(s'|s,a)Q^*(s',a')}  - Q^*(s,a)} \right\} = 0,\quad (s,a,a') \in S \times A \times A\]

\end{proof}

\section{Appendix:~\cref{lemma:bounded-D}}\label{sec:app:property-domain1}

In this paper, using the log-barrier function, the inequality constraints of the LP form of MDPs is integrated into a single objective function as follows:
\begin{align*}
{f_\eta }(Q): = \sum\limits_{(s,a) \in {\cal S} \times {\cal A}} {Q(s,a)\rho (s,a)}  + \eta \sum\limits_{(s,a,a') \in {\cal S} \times {\cal A} \times {\cal A}} {w(s,a,a')\varphi \left( {(FQ)(s,a,a') - Q(s,a)} \right)},
\end{align*}
where $\eta >0$ is the barrier parameter (weight) and $w(s,a,a')>0,(s,a,a')\in {\cal S} \times {\cal A}\times {\cal A}$ are weight parameters of the inequality constraints. Several useful properties of the domain $\cal D$ of $f_\eta$
\begin{align*}
{\cal D}: = \left\{ Q \in {\mathbb R}^{|{\cal S}||{\cal A}|}:(FQ)(s,a,a') - Q(s,a) < 0,(s,a,s') \in {\cal S} \times {\cal A} \times {\cal A} \right\},
\end{align*}
are summarized in the following lemma.
\begin{lemma}\label{lemma:bounded-D}
The following statements hold true:
\begin{enumerate}
\item $\cal D$ is convex and open

\item $\cal D$ is bounded below and unbounded above.
\end{enumerate}
\end{lemma}
\begin{proof}
For the first statement, let us note that $\cal D$ is the set of points which satisfy the linear inequalities
\begin{align*}
(FQ)(s,a,a') - Q(s,a) < 0,\quad \forall (s,a,a') \in {\cal S} \times {\cal A} \times {\cal A},
\end{align*}
which means that it is an intersections of affine sets, and hence, it is a convex set. Moreover, we can prove that it is open because the inequalities are strict.

For the second statement, one can observe that every $Q\in {\cal D}$ satisfies
\begin{align*}
(FQ)(s,a,a') - Q(s,a) < 0,\quad \forall (s,a,a') \in {\cal S} \times {\cal A} \times {\cal A},
\end{align*}
which implies
\begin{align*}
(TQ)(s,a) - Q(s,a) < 0,\quad \forall (s,a) \in {\cal S} \times {\cal A}
\end{align*}
and vice versa. Next, because the Bellman operator is monotone \citep[Lemma 1.1.1]{bertsekas2012dynamic2}, i.e., for $Q' \ge Q$, $TQ' \ge TQ$, we have ${T^k}Q \le Q$.
Taking the limit $ \lim_{k \to \infty } T^k Q = Q^* \le Q$ leads to the conclusion that $\cal D$ is bounded below. Moreover, let
\begin{align*}
Q(s,a) = \frac{{{r_{\max }} + \varepsilon }}{{1 - \gamma }} > 0,\quad (s,a) \in {\cal S} \times {\cal A}
\end{align*}
with any $\varepsilon>0$. Then, one has
\begin{align*}
&R(s,a) + \gamma \sum\limits_{s' \in S} {P(s'|s,a)Q(s',a')}  - Q(s,a)\\
=& R(s,a) + \gamma \sum\limits_{s' \in S} {P(s'|s,a)\frac{{{r_{\max }} + \varepsilon }}{{1 - \gamma }}}  - \frac{{{r_{\max }} + \varepsilon }}{{1 - \gamma }}\\
=& R(s,a) + (\gamma  - 1)\frac{{{r_{\max }} + \varepsilon }}{{1 - \gamma }}\\
=& R(s,a) - {r_{\max }} - \varepsilon  < 0
\end{align*}
Therefore, $Q$ is strictly feasible. Because $\varepsilon>0$ is arbitrary, $\cal D$ is unbounded above, and we can conclude the proof.
\end{proof}

The properties of the domain $\cal D$ established in the preceding lemma provide not only conceptual insight but also play a crucial role in the proofs of several subsequent results.

\section{Appendix:~\cref{lemma:bounded-levelset}}\label{sec:app:property-level1}

Several useful properties of the level set ${\cal L}_c$ defined as
\begin{align*}
{\cal L}_c: = \{ Q \in {\cal D}:{f_\eta }(Q) \le c\}
\end{align*}
are summarized in the following lemma.
\begin{lemma}\label{lemma:bounded-levelset}
For any given $c>0$, the following statements hold true:
\begin{enumerate}

\item ${\cal L}_c$ is convex and closed.

\item ${\cal L}_c$ is bounded.

\item there exists a constant $m_c>0$ such that
\[Q(s,a) - (FQ)(s,a,a') \ge m_c,\quad \forall (s,a,a') \in {\cal S} \times {\cal A} \times {\cal A},\quad \forall Q\in {\cal L}_c\]
\end{enumerate}
\end{lemma}
\begin{proof}
First of all, it is known that the log-barrier function is convex~\citep[Chapter 11.2.1]{Boyd2004}, and the composition of a convex function and affine function is convex~\citep[Chapter 3.2.2]{Boyd2004}. Since a sublevel set of a convex function is convex~\citep[Chapter 3.1.6]{Boyd2004}, ${\cal L}_c$ is convex.

Moreover, let us consider any sequence $Q_k \in {\cal L}_c$ such that $\lim_{x \to \infty } {Q_k} = \bar Q \in {\cal D}$. We want to prove that $\bar Q \in {\cal L}_c$ so that ${\cal L}_c$ is closed. This is true because $f_\eta$ is continuous in its domain (because it is convex in the domain), i.e.,
\begin{align*}
\mathop{\lim }\limits_{x \to \infty } {f_\eta }({Q_k}) = {f_\eta }(\mathop{\lim }\limits_{x \to \infty } {Q_k}) = {f_\eta }(\bar Q) \le c.
\end{align*}
This completes the proof of the first statement.

For the second statement, we note that by~\cref{lemma:bounded-D}, every $Q\in {\cal D}$ is bounded below. Therefore, ${\cal L}_c \subset {\cal D}$ is bounded below as well. To prove the upper bound, let us assume that ${\cal L}_c$ is unbounded from above. In this case, we can find a sequence $(Q_k)_{k=0}^\infty$ with $Q_k \in {\cal L}_c$ such that
\[\mathop{\lim }\limits_{k \to \infty } {\max _{(s,a) \in {\cal S} \times {\cal A}}}{Q_k}(s,a) = \infty \]
To proceed, let us define
\begin{align*}
(s_k,a_k): = \argmax_{(s,a) \in {\cal S} \times {\cal A}} {Q_k}(s,a).
\end{align*}

Then, we have
\begin{align*}
0 <& {Q_k}(s,a) - (F{Q_k})(s,a,a')\\
\le& {Q_k}({s_k},{a_k}) + {\max _{(s,a) \in {\cal S} \times {\cal A}}}|R(s,a)| + \gamma \sum\limits_{s' \in {\cal S}} {P(s'|s,a)} {Q_k}({s_k},{a_k})\\
\le& (1 + \gamma ){Q_k}({s_k},{a_k}) + {r_{\max }},\quad \forall (s,a,a') \in {\cal S} \times {\cal A} \times {\cal A}.
\end{align*}
Because $-\ln$ is a nonincreasing function, it follows that
\begin{align*}
- \ln \left\{ {{Q_k}(s,a) - (F{Q_k})(s,a,a')} \right\} \ge  - \ln \left\{ {(1 + \gamma ){Q_k}({s_k},{a_k}) + {r_{\max }}} \right\},\quad \forall (s,a,a') \in {\cal S} \times {\cal A} \times {\cal A}.
\end{align*}
Using this inequality, the objective function is bounded as
\begin{align*}
f_\eta(Q_k) =& \sum\limits_{(s,a) \in {\cal S} \times {\cal A}} {Q_k(s,a)\rho (s,a)}  - \eta \sum\limits_{(s,a,a') \in {\cal S} \times {\cal A} \times {\cal A}} {w(s,a,a')\ln \left[ {(F{Q_k})(s,a,a') - Q_k(s,a)} \right]}\\
\ge& \sum\limits_{(s,a) \in {\cal S} \times {\cal A}} {Q_k(s,a)\rho (s,a)}  - \eta \sum\limits_{(s,a,a') \in {\cal S} \times {\cal A} \times {\cal A}} {w(s,a,a')\ln \left[ {(1 + \gamma ){Q_k}(s_k,a_k) + r_{\max}} \right]}\\
\ge& Q_k(s_k,a_k)\rho ({s_k},{a_k}) + \sum\limits_{(s,a) \in ({\cal S} \times {\cal A})\backslash (s_k,a_k)} {Q^*(s,a)\rho (s,a)} \\
&- \eta \sum\limits_{(s,a,a') \in {\cal S} \times {\cal A} \times {\cal A}} {w(s,a,a')\ln \left[ {(1 + \gamma ){Q_k}(s_k,a_k) + r_{\max }} \right]}\\
\ge& Q_k(s_k,a_k){\min _{(s,a) \in {\cal S} \times {\cal A}}}\rho (s,a)\\
&+ (|{\cal S} \times {\cal A}| - 1){\min _{(s,a) \in {\cal S} \times {\cal A}}}Q^*(s,a)\rho (s,a)\\
&- \eta \sum\limits_{(s,a,a') \in {\cal S} \times {\cal A} \times {\cal A}} {w(s,a,a')\ln \left[ {(1 + \gamma ){Q_k}(s_k,a_k) + r_{\max }} \right]}.
\end{align*}

By taking the limit, we have
\begin{align*}
\mathop{\lim }\limits_{k \to \infty } {f_\eta }(Q_k) \ge& \mathop{\lim }\limits_{k \to \infty } {Q_k}(s_k,a_k){\min _{(s,a) \in {\cal S} \times {\cal A}}}\rho (s,a)\\
& + (|{\cal S} \times {\cal A}| - 1){\min _{(s,a) \in {\cal S} \times {\cal A}}}Q^*(s,a)\rho (s,a)\\
& - \eta \sum\limits_{(s,a,a') \in {\cal S} \times {\cal A} \times {\cal A}} {w(s,a,a')\mathop{\lim }\limits_{k \to \infty } \ln \left[ {(1 + \gamma ){Q_k}(s_k,a_k) + r_{\max }} \right]} \\
=& \infty
\end{align*}
which is a contradiction. Therefore, ${\cal L}_c$ is bounded.

Next, we prove the last statement. Since ${\cal L}_c$ is bounded, for any $Q\in {\cal L}_c$, the first term in $f_\eta$ is bounded as follows:
\begin{align*}
\left| {\sum\limits_{(s,a) \in {\cal S} \times {\cal A}} {Q(s,a)\rho (s,a)} } \right| \le B_1,\quad \forall Q \in {\cal L}_c,
\end{align*}
where $B_1>0$ is some constant. By contradiction, let us assume that
\[{\inf _{Q \in {\cal L}_c}}{\min _{(s,a,a') \in {\cal S} \times {\cal A} \times {\cal A}}}\{ Q(s,a) - (FQ)(s,a,a')\}  = 0\]
Then, there exists a sequence $Q_k \in {\cal L}_c$ such that
\[\mathop{\lim }\limits_{k \to \infty } {\min _{(s,a,a') \in {\cal S} \times {\cal A} \times {\cal A}}}\{ Q_k(s,a) - (F{Q_k})(s,a,a')\} = 0. \]

Therefore, one gets
\begin{align*}
{f_\eta }(Q_k) =& \sum\limits_{(s,a) \in {\cal S} \times {\cal A}} {Q_k(s,a)\rho (s,a)}  - \eta \sum\limits_{(s,a,a') \in {\cal S} \times {\cal A} \times {\cal A}} {w(s,a,a')\ln \left\{ {Q_k(s,a) - (F{Q_k})(s,a,a')} \right\}}  \\
\ge&  - {B_1} - \eta \sum\limits_{(s,a,a') \in S \times A \times A} {w(s,a,a')\ln \left\{ {{Q_k}(s,a) - (F{Q_k})(s,a,a')} \right\}}.
\end{align*}
By taking the limit
\begin{align*}
c \ge& \mathop{\lim }\limits_{k \to \infty } {f_\eta }(Q_k)\\
 \ge&  - {B_1} - \eta \mathop{\max }\limits_{(s,a,a') \in {\cal S} \times {\cal A} \times {\cal A}} w(s,a,a')|{\cal S}||{\cal A}|^2 \mathop{\lim }\limits_{k \to \infty } \ln \left[ {{{\min }_{(s,a,a') \in {\cal S} \times {\cal A} \times {\cal A}}}\left\{ Q_k(s,a) - (F Q_k)(s,a,a') \right\}} \right]\\
  =& \infty.
\end{align*}
which is a contradiction.
Therefore, there should exist $m_c >0$ such that
\[{\inf _{Q \in {\cal L}_c}}{\min _{(s,a,a') \in {\cal S} \times {\cal A} \times {\cal A}}}\{ Q(s,a) - (FQ)(s,a,a')\}  \geq m_c.\]
This completes the overall proof.
\end{proof}

\section{Appendix:~\cref{lemma:properties1}}\label{sec:app:property-level2}

The following summarizes several properties of the objective function $f_\eta$ related to its convexity and smoothness.
\begin{lemma}\label{lemma:properties1}
For any given $c>0$, we have
\begin{enumerate}
\item ${\nabla _Q}{f_\eta }(Q)$ is bounded in ${\cal L}_c$

\item $f_\eta$ is $\mu$-strongly convex and $L$-smooth in ${\cal L}_c$

\item Moreover, $f_\eta$ is strictly convex in the domain $\cal D$.
\end{enumerate}
\end{lemma}
\begin{proof}[Proof of the first statement]
By~\cref{lemma:bounded-levelset}, there exists a constant $m_c>0$ such that
\[Q(s,a) - (FQ)(s,a,a') \ge {m_c},\quad \forall (s,a,a') \in {\cal S} \times {\cal A} \times {\cal A}\]
for any $Q \in {\cal L}_c$.
Therefore, we can derive the following bounds
\begin{align*}
|({\nabla _Q}f(Q))(s,a)| \le& 1 + \gamma \sum\limits_{(s',a') \in {\cal S} \times {\cal A}} {P(s|s',a')|{\lambda _\eta }(s',a',a)|}  + \sum\limits_{a' \in {\cal A}} {|{\lambda _\eta }(s,a,a')|} \\
\le& 1 + \gamma \sum\limits_{(s',a') \in {\cal S} \times {\cal A}} {P(s|s',a')|{\lambda _\eta }(s',a',a)|}  + \sum\limits_{a' \in {\cal A}} {|{\lambda _\eta }(s,a,a')|}\\
\le& 1 + \gamma \sum\limits_{(s',a') \in {\cal S} \times {\cal A}} {P(s|s',a')\frac{{\eta w(s',a',a)}}{{Q(s',a') - (FQ)(s',a',a)}}}\\
&  + \sum\limits_{a' \in {\cal A}} {\frac{{\eta w(s,a,a')}}{{Q(s,a) - (FQ)(s,a,a')}}} \\
\le& 1 + \gamma \mathop{\max }\limits_{(s,a,a') \in {\cal S} \times {\cal A} \times {\cal A}} w(s,a,a')\frac{{\eta |{\cal S}||{\cal A}|}}{m_c} + \mathop{\max }\limits_{(s,a,a') \in {\cal S} \times {\cal A} \times {\cal A}} w(s,a,a')\frac{{\eta |{\cal A}|}}{m_c}
\end{align*}
which completes the proof of the first statement.
\end{proof}

\begin{proof}[Proof of the second statement]
For the second statement, since the log-barrier function is strongly convex in a convex compact set, $f_\eta$ is also strongly convex in the set because it is a composition of a linear function and strongly convex function. To be more precise, for all $(s,a,a')\in {\cal S}\times {\cal A}\times {\cal A}$, let us define
\[{v_{s,a,a'}}: = {\nabla _Q}(Q(s,a) - (FQ)(s,a,a')) = {e_{s,a}} - \gamma \sum\limits_{s' \in {\cal S}} {P(s'|s,a){e_{s',a'}}} \]
where ${e_{s,a}} = e_s \otimes e_a \in {\mathbb R}^{|{\cal S}||{\cal A}|}$, and $e_s \in {\mathbb R}^{|{\cal S}|}$ and $e_a \in {\mathbb R}^{|{\cal A}|}$ are $s$th basis vector (all components
are $0$ except for the $s$th component which is $1$) and $a$th basis vector, respectively. Then, the gradient can be written as
\[{\nabla _Q}{f_\eta }(Q) = \rho  - \eta \sum\limits_{(s,a,a') \in {\cal S} \times {\cal A} \times {\cal A}} {\frac{w(s,a,a')}{{Q(s,a) - (FQ)(s,a,a')}}{v_{s,a,a'}}}  \in {\mathbb R}^{|{\cal S}||{\cal A}|}\]
and the Hessian is
\[\nabla _Q^2{f_\eta }(Q) = \eta \sum\limits_{(s,a,a') \in {\cal S} \times {\cal A} \times {\cal A}} {\frac{{w(s,a,a')}}{{{{\{ Q(s,a) - (FQ)(s,a,a')\} }^2}}}{v_{s,a,a'}}v_{s,a,a'}^T}  \in {\mathbb R}^{|{\cal S}||{\cal A}| \times |{\cal S}||{\cal A}|}\]
By~\cref{lemma:bounded-levelset}, we have
\[{m_c} \le Q(s,a) - (FQ)(s,a,a') \le {M_c},\quad \forall Q \in {\cal L}_c,(s,a,a') \in {\cal S} \times {\cal A} \times {\cal A}\]
for some $m_c, M_c>0$, where the bounds come from the fact that ${\cal L}_c$ is bounded in~\cref{lemma:bounded-levelset}. Then, one has
\begin{align*}
{x^T}\nabla _Q^2{f_\eta }(Q)x =& \eta \sum\limits_{(s,a,a') \in {\cal S} \times {\cal A} \times {\cal A}} {\frac{{w(s,a,a'){{(v_{s,a,a'}^Tx)}^2}}}{{{{\{ Q(s,a) - (FQ)(s,a,a')\} }^2}}}} \\
\ge& \frac{{\eta \mathop{\min }\limits_{(s,a,a') \in {\cal S} \times {\cal A} \times {\cal A}} w(s,a,a')}}{{M_c^2}}\sum\limits_{(s,a,a') \in {\cal S} \times {\cal A} \times {\cal A}} {{{(v_{s,a,a'}^Tx)}^2}} \\
=& \frac{{\eta \mathop{\min }\limits_{(s,a,a') \in {\cal S} \times {\cal A} \times {\cal A}} w(s,a,a')}}{{M_c^2}}{x^T}Sx,
\end{align*}
where
\[S: = \sum\limits_{(s,a,a') \in {\cal S} \times {\cal A} \times {\cal A}} {{v_{s,a,a'}}v_{s,a,a'}^T}  \succeq 0.\]
Next, we prove that the above $S$ is strictly positive definite.
By contradiction, let us suppose that $S$ is not strictly positive definite.
Then, there exists a nonzero $x$ such that
\begin{align*}
x^T \sum\limits_{(s,a,a') \in {\cal S} \times {\cal A} \times {\cal A}} {{v_{s,a,a'}}v_{s,a,a'}^T} x = \sum\limits_{(s,a,a') \in {\cal S} \times {\cal A} \times {\cal A}} {{{(v_{s,a,a'}^Tx)}^2}}  = 0.
\end{align*}
This implies
\[v_{s,a,a'}^Tx = 0,\quad \forall (s,a,a') \in {\cal S} \times {\cal A} \times {\cal A},\]
which can be equivalently written as
\[e_{s,a}^Tx - \gamma \sum\limits_{s' \in {\cal S}} {P(s'|s,a)e_{s',a'}^Tx}  = x(s,a) - \gamma \sum\limits_{s' \in {\cal S}} {P(s'|s,a)x(s',a')}  = 0,\quad \forall (s,a,a') \in {\cal S} \times {\cal A} \times {\cal A}.\]

Now, let us fix $a=a'\in {\cal A}$ and
\[{(I - \gamma {P_a})^T}{x_a} = 0\]
where
\[{x_a} = \left[ {\begin{array}{*{20}{c}}
{x(1,a)}\\
 \vdots \\
{x(|S|,a)}
\end{array}} \right] \in {\mathbb R}^{|{\cal S}|},\quad {P_a} = \left[ {\begin{array}{*{20}{c}}
{P(1|1,a)}& \cdots &{P(|{\cal S}||1,a)}\\
 \vdots & \ddots & \vdots \\
{P(1||S|,a)}& \cdots &{P(|{\cal S}|||{\cal S}|,a)}
\end{array}} \right] \in {\mathbb R}^{|{\cal S}| \times |{\cal S}|}\]

Since $I - \gamma {P_a}$ is nonsingular for any $a\in {\cal A}$, $x_a =0$ for all $a\in {\cal A}$. This is a contradiction. Therefore, $S \succ 0$ and
\begin{align*}
\nabla _Q^2{f_\eta }(Q) \succeq \frac{\eta \min\limits_{(s,a,a') \in {\cal S} \times {\cal A} \times {\cal A}} w(s,a,a')\lambda _{\min }(S)}{M_c^2}I,
\end{align*}
where $\lambda _{\min }(\cdot)$ denotes the minimum eigenvalue.
This means that in ${\cal L}_c$, $f_\eta$ is $\mu$-strongly convex~\citep[Theorem~2.1.11]{nesterov2018lectures} with
\begin{align*}
\mu  = \frac{\eta \min\limits_{(s,a,a') \in {\cal S} \times {\cal A} \times {\cal A}} w(s,a,a'){\lambda _{\min }}(S)}{M_c^2}.
\end{align*}

Similarly, we can easily show that in ${\cal L}_c$
\begin{align*}
{x^T}\nabla _Q^2{f_\eta }(Q)x =& \eta \sum\limits_{(s,a,a') \in {\cal S} \times {\cal A} \times {\cal A}} {\frac{{w(s,a,a'){(v_{s,a,a'}^Tx)^2}}}{{\{ Q(s,a) - (FQ)(s,a,a')\}^2}}} \\
\le& \frac{\eta \max \limits_{(s,a,a') \in {\cal S} \times {\cal A} \times {\cal A}} w(s,a,a')}{{m_c^2}}\sum\limits_{(s,a,a')\in {\cal S}\times {\cal A}\times {\cal A}} {(v_{s,a,a'}^Tx)^2}\\
=& \eta \frac{1}{m_c^2} x^T Sx\\
\le& \frac{\eta  \max \limits_{(s,a,a') \in {\cal S} \times {\cal A} \times {\cal A}} w(s,a,a'){\lambda _{\max }}(S)}{m_c^2}\left\| x \right\|_2^2,
\end{align*}
where $\lambda _{\max}(\cdot)$ denotes the maximum eigenvalue.
This implies
\begin{align*}
{\left\| {{\nabla _Q}{f_\eta }(Q) - {\nabla _Q}{f_\eta }(Q')} \right\|_2} \le L{\left\| {Q - Q'} \right\|_2}
\end{align*}
with
\begin{align*}
L = \frac{{\eta \mathop{\max }\limits_{(s,a,a') \in {\cal S} \times {\cal A} \times {\cal A}} w(s,a,a'){\lambda _{\max }}(S)}}{m_c^2}.
\end{align*}
by the sufficient condition in~\citet[Lemma~1.2.2]{nesterov2018lectures}.
Therefore, $f_\eta$ is $L$-smooth in ${\cal L}_c$.
\end{proof}

\begin{proof}[Proof of the third statement]
In the domain $Q\in {\cal D}$ which is unbounded above, by~\cref{lemma:bounded-levelset}, we have
\[{m_c} \le Q(s,a) - (FQ)(s,a,a'),\quad \forall Q \in {\cal D},(s,a,a') \in {\cal S} \times {\cal A} \times {\cal A}\]
Therefore, the Hessian is bounded as
\begin{align*}
{x^T}\nabla _Q^2{f_\eta }(Q)x \ge \frac{{\eta \mathop{\min }\limits_{(s,a,a') \in {\cal S} \times {\cal A} \times {\cal A}} w(s,a,a')}}{{M{{(Q)}^2}}}{x^T}Sx,
\end{align*}
where
\begin{align*}
{\max _{(s,a,a') \in {\cal S} \times {\cal A} \times {\cal A}}}\{ Q(s,a) - (FQ)(s,a,a')\}  = M(Q) > 0,\quad \forall Q \in {\cal D}.
\end{align*}
Since $S$ is strictly positive definite, the above result implies that the Hessian is strictly positive definite
\begin{align*}
\nabla _Q^2{f_\eta }(Q) \succ 0,\quad \forall Q \in {\cal D}.
\end{align*}
Therefore, by the second-order condition in~\citet[Chapter~3.1.4]{Boyd2004}, $f_\eta$ is strictly convex in $\cal D$.
\end{proof}

\section{Appendix:~\cref{thm:convergence-GD}}\label{sec:app:gradient1}
When gradient descent is applied to the objective function $f_\eta$, we obtain the following results on its convergence rate. An interesting observation is that, although $f_\eta$ is strictly convex within the domain $\cal D$ but not strongly convex, the convergence result coincides with that of a globally strongly convex function defined on the domain. The reason is that, while $f_\eta$ is only strictly convex over $\cal D$, it is both strongly convex and $L$-smooth within its sublevel sets ${\cal L}_c$. Exploiting this property, we can prove that the convergence rate for the gradient descent iterates of $f_\eta$ matches that of a strongly convex function on the domain. This constitutes one of the key results of this paper.

\begin{theorem}\label{thm:convergence-GD}
Let us consider the gradient descent iterates:
\begin{align*}
Q_{k + 1} = Q_k - \alpha {\left. {{\nabla _Q}f_\eta (Q)} \right|_{Q = Q_k}},
\end{align*}
for $k=0,1,\ldots$ with a strictly feasible initial point $Q_0 \in {\cal D}$.
With the step-size $0<\alpha \leq \frac{2}{L_c+\mu_c}$, the iterates satisfies
\[\left\| {{Q_k} - {{\tilde Q}_\eta }} \right\|_2^2 \le {\left( {1 - \frac{{2\alpha {\mu _c}{L_c}}}{{\mu _c + L_c}}} \right)^k}\left\| {{Q_0} - {{\tilde Q}_\eta }} \right\|_2^2\]
where $\mu_c>0$ is the coefficient of the strong convexity and $L_c >0$ is the Lipschitz constant of the gradient of $f_\eta$ within the level set ${\cal L}_c$ where $c > f_\eta (Q_0)$,
\end{theorem}

To prove the convergence, we first need the following basic lemma.
\begin{lemma}[Theorem 2.1.5 in~\citet{nesterov2018lectures} ]\label{lemma:descent-lemma}
If $f:{\mathbb R}^n \to {\mathbb R}$ is $L$-smooth, then
\begin{align*}
f(y) \le f(x) + \left\langle {\nabla f(x),y - x} \right\rangle  + \frac{L}{2}\left\| y - x \right\|_2^2,\quad \forall x,y \in {\mathbb R}^n.
\end{align*}
\end{lemma}
By inserting $y = x - \alpha \nabla f(x)$ in~\cref{lemma:descent-lemma}, we have
\begin{align*}
f(x - \alpha \nabla f(x)) \le& f(x) - \alpha \left\langle {\nabla f(x),\nabla f(x)} \right\rangle  + \frac{L}{2}\left\| {\alpha \nabla f(x)} \right\|_2^2\\
=& f(x) - \alpha \left( {1 - \frac{\alpha L}{2}} \right)\left\| {\nabla f(x)} \right\|_2^2,\quad \forall x \in {\mathbb R}^n.
\end{align*}
Based on this observation, we now prove~\cref{thm:convergence-GD} as follows.
\begin{proof}[Proof of~\cref{thm:convergence-GD}]
First of all, let us choose any $Q \in {\cal D}$ such that ${\nabla _Q}{f_\eta }(Q) \ne 0$. Then, $Q \in {\cal L}_c$ for some $c>0$, which is any number such that $c > f_\eta(Q)$. Therefore, by~\cref{lemma:properties1}, $f_\eta$ is $L_c$-smooth and $\mu_c$-strongly convex in ${\cal L}_c$ for some real numbers $L_c, \mu_c >0$. We can choose a ball $B_r(Q)$ centered at $Q$ with radius $r>0$ such that $B_r(Q) \subset {\cal L}_c$.
Let us assume that
\[0 < \alpha  \le \frac{2}{{{L_c} + {\mu _c}}}\]
and define the next iterate
\begin{align*}
 Q^+_\alpha := Q - \alpha \nabla f_\eta(Q).
\end{align*}
Because $B_r(Q) \subset {\cal L}_c$ we can choose a sufficiently small $\alpha>0$ such that $ Q^+_\alpha \in {\cal L}_c$. Because $Q,Q_\alpha^+ \in {\cal L}_c$, we can apply~\cref{lemma:descent-lemma} so that
\begin{align*}
{f_\eta }(Q_\alpha^+) =&
{f_\eta }(Q - \alpha {\nabla _Q}{f_\eta }(Q))\\ \le& {f_\eta }(Q) - \alpha \left\langle {{\nabla _Q}{f_\eta }(Q),{\nabla _Q}{f_\eta }(Q)} \right\rangle  + \frac{L_c}{2}\left\| {\alpha {\nabla _Q}{f_\eta }(Q)} \right\|_2^2\\
=& {f_\eta }(Q) - \alpha \left( {1 - \frac{{\alpha {L_c}}}{2}} \right)\left\| {{\nabla _Q}{f_\eta }(Q)} \right\|_2^2
\end{align*}
which implies $Q_\alpha ^ +  \in {\cal L}_c$. Now, we can increase $\alpha$ so that $\alpha  = \frac{2}{L_c + \mu _c}$ or $Q_\alpha ^ +  \in \partial {\cal L}_c$, where $\partial {\cal L}_c$ is the boundary of ${\cal L}_c$ noting that ${\cal L}_c$ is a closed set (\cref{lemma:bounded-levelset}). Next, we want to prove that $\alpha$ reaches $\alpha  = \frac{2}{L_c + \mu _c}$ until $Q_\alpha ^ +  \in \partial {\cal L}_c$. By contradiction, let us assume that $Q_\alpha ^ +  \in \partial {\cal L}_c$ until $\alpha$ reaches $\alpha  = \frac{2}{{\cal L}_c + \mu _c}$.
For any $\varepsilon>0$, one can show that
\[Q_{\alpha  + \varepsilon }^ +  \notin {\cal L}_c\]
and
\[f_\eta(Q_{\alpha  + \varepsilon }^ + ) > c\]
Howver, we know that
\begin{align*}
f_\eta(Q_\alpha ^ + ) \le& f_\eta(Q) - \alpha \left( {1 - \frac{\alpha L_c}{2}} \right)\left\| {{\nabla _Q}{f_\eta }(Q)} \right\|_2^2\\
\le& c - \alpha \left( {1 - \frac{{\alpha {L_c}}}{2}} \right)\left\| {{\nabla _Q}{f_\eta }(Q)} \right\|_2^2\\
\le& {f_\eta }(Q_{\alpha  + \varepsilon }^ + ) - \alpha \left( {1 - \frac{{\alpha {L_c}}}{2}} \right)\left\| {{\nabla _Q}{f_\eta }(Q)} \right\|_2^2
\end{align*}
Taking $\varepsilon \to 0$ leads to
\[f_\eta(Q_\alpha ^ + ) < f_\eta(Q_\alpha ^ + ) + \alpha \left( {1 - \frac{{\alpha {L_c}}}{2}} \right)\left\| {{\nabla _Q}{f_\eta }(Q)} \right\|_2^2 \le \mathop{\lim}\limits_{\varepsilon  \to 0} {f_\eta }(Q_{\alpha  + \varepsilon }^ + )\]
which contradicts with the fact that $f_\eta$ is continuous in $\cal D$.
In summary, for any $0 < \alpha  \le \frac{2}{{{L_c} + {\mu _c}}}$ and $Q\in {\cal L}_c$, we have
\[f_\eta(Q_\alpha ^ + ) \le f_\eta(Q) - \alpha \left( {1 - \frac{\alpha L_c}{2}} \right)\left\| {{\nabla _Q}{f_\eta }(Q)} \right\|_2^2\]
and $Q_\alpha ^ +  \in {\cal L}_c$.
Therefore, if $Q_0 \in {\cal D}$, then
\[Q_k \in {\cal L}_c,\quad k = 0,1,2, \ldots \]
for some $c>0$. Next, standard results in convex optimization (Theorem 2.1.15 in~\citet{nesterov2018lectures}) can be applied to conclude the proof.
\end{proof}

\begin{theorem}\label{thm:convergence-GD2}
Let us consider the gradient descent iterates:
\begin{align*}
Q_{k + 1} = Q_k - \alpha \left. {\nabla _Q}f_\eta (Q) \right|_{Q = Q_k},
\end{align*}
for $k=0,1,\ldots$ with a strictly feasible initial point $Q_0 \in {\cal D}$ and
\[\left\| Q_0 \right\|_\infty  \le \frac{\eta \mathop{\sum}\limits_{(s,a,a') \in {\cal S} \times {\cal A} \times {\cal A}} w(s,a,a') }{\min_{(s,a) \in {\cal S} \times {\cal A}}\rho (s,a)} + \frac{r_{\max} }{1 - \gamma }.\]
Then, with the step-size $0<\alpha \leq \frac{2}{L_c+\mu_c}$, the iterates satisfies
\[\left\| Q_k - \tilde Q_\eta  \right\|_\infty \le {\left( \sqrt {1 - \frac{2\alpha {\mu _c}{L_c}}{\mu _c + L_c}} \right)^k}2\sqrt {|{\cal S} \times {\cal A}|} \left( \frac{{\eta \mathop{\sum}\limits_{(s,a,a') \in {\cal S} \times {\cal A} \times {\cal A}} w(s,a,a') }}{\min_{(s,a) \in {\cal S} \times {\cal A}}\rho (s,a)} + \frac{r_{\max }}{1 - \gamma } \right),\]
where $\mu_c>0$ is the coefficient of the strong convexity and $L_c >0$ is the Lipschitz constant of the gradient of $f_\eta$ within the level set ${\cal L}_c$ where $c > f_\eta (Q_0)$.

Moreover, to achieve ${\left\| Q_k - {{\tilde Q}_\eta } \right\|_\infty } \le \varepsilon$, we need at least the following number of iterations:
\[O\left( \ln \left( \frac{{\sqrt {|{\cal S} \times {\cal A}|} }}{\varepsilon }\left\{ \frac{\eta \mathop{\sum}\limits_{(s,a,a') \in {\cal S} \times {\cal A} \times {\cal A}} w(s,a,a')}{\min_{(s,a) \in {\cal S} \times {\cal A}}\rho (s,a)} + \frac{r_{\max}}{1 - \gamma} \right\} \right) \right).\]
\end{theorem}
\begin{proof}
By~\cref{thm:convergence-GD} and~\cref{lemma:bound-on-Q-eta}, we can obtain
\begin{align*}
{\left\| {Q_k - {{\tilde Q}_\eta }} \right\|_2} \le& {\left( {\sqrt {1 - \frac{{2\alpha {\mu _c}{L_c}}}{{{\mu _c} + {L_c}}}} } \right)^k}{\left\| {Q_0 - {{\tilde Q}_\eta }} \right\|_2}\\
\le& {\left( {\sqrt {1 - \frac{{2\alpha \mu_c L_c}}{{\mu _c + L_c}}} } \right)^k}\sqrt {|{\cal S} \times {\cal A}|} {\left\| {Q_0 - {{\tilde Q}_\eta }} \right\|_\infty }\\
\le& {\left( {\sqrt {1 - \frac{{2\alpha \mu _c L_c}}{{\mu _c + L_c}}} } \right)^k}2\sqrt {|{\cal S} \times {\cal A}|} \left( {\frac{{\eta \sum\limits_{(s,a,a') \in {\cal S} \times {\cal A} \times {\cal A}} w(s,a,a') }}{{{{\min }_{(s,a) \in {\cal S} \times {\cal A}}}\rho (s,a)}} + \frac{r_{\max }}{1 - \gamma}} \right),
\end{align*}
where the second inequality comes from the inequality $\|\cdot\|_2\le \sqrt{|{\cal S}\times {\cal A}|}\|\cdot\|_\infty$, the last inequality is due to~\cref{lemma:bound-on-Q-eta} and the bounded assumption on $Q_0$. Rearranging terms and taking log function on both sides lead to
\begin{align*}
k\ln \left( {\sqrt {1 - \frac{{2\alpha \mu_c L_c}}{\mu_c + L_c}} } \right) \le \ln \left( {\varepsilon \frac{1}{{2\sqrt {|{\cal S} \times {\cal A}|} }}\frac{1}{{\frac{{\eta \sum\limits_{(s,a,a') \in {\cal S} \times {\cal A} \times {\cal A}} {w(s,a,a')} }}{{{{\min }_{(s,a) \in {\cal S} \times {\cal A}}}\rho (s,a)}} + \frac{r_{\max }}{1 - \gamma }}}} \right).  \end{align*}
Rearranging and simplifying terms again lead to
\begin{align*}
k \ge \frac{{\ln \left( {\frac{{2\sqrt {|{\cal S} \times {\cal A}|} }}{\varepsilon }\left\{ {\frac{{\eta \sum\limits_{(s,a,a') \in {\cal S} \times {\cal A} \times {\cal A}} {w(s,a,a')} }}{{{{\min }_{(s,a) \in {\cal S} \times {\cal A}}}\rho (s,a)}} + \frac{{{r_{\max }}}}{{1 - \gamma }}} \right\}} \right)}}{{ - \ln \left( {\sqrt {1 - \frac{{2\alpha \mu_c L_c}}{\mu_c + L_c}} } \right)}}.    \end{align*}

\end{proof}

We note that if $w$ is a probability distribution, then the sum in the numerator becomes one. On the other hand, the term ${{\min }_{(s,a) \in {\cal S} \times {\cal A}}}\rho (s,a)$ depends on the initial state distribution $\rho$, and this term can be made arbitrarily small. For instance, when the initial state distribution is uniform, i.e., $\rho (s,a) = \frac{1}{|{\cal S}||{\cal A}|}$ and $w$ is a probability distribution, then
\[O\left( {\ln \left( {\frac{\eta |{\cal S} \times {\cal A}|^{3/2}}{\varepsilon } + \frac{{|{\cal S} \times {\cal A}|^{1/2}}}{\varepsilon }\frac{r_{\max}}{1 - \gamma }} \right)} \right)\]

For computational cose of the gradient computation, in tabular log-barrier method, each gradient update scales with the number of state–action pairs (i.e., linear in $|{\cal S}\times {\cal A}|$) because the objective and its gradient must be evaluated over the relevant domain.

%===========experiment explanation for GD===========
\section{Appendix: Gradient descent example}\label{sec:app:gradient2}

% --- FIGURE BLOCK (REVISED) ---
% Note: The \subfigure command is obsolete. For a single image, \includegraphics is sufficient.
% Math notation in the caption has been corrected.
\begin{figure}[h!]
    \centering
    \includegraphics[width=14cm, height=7cm]{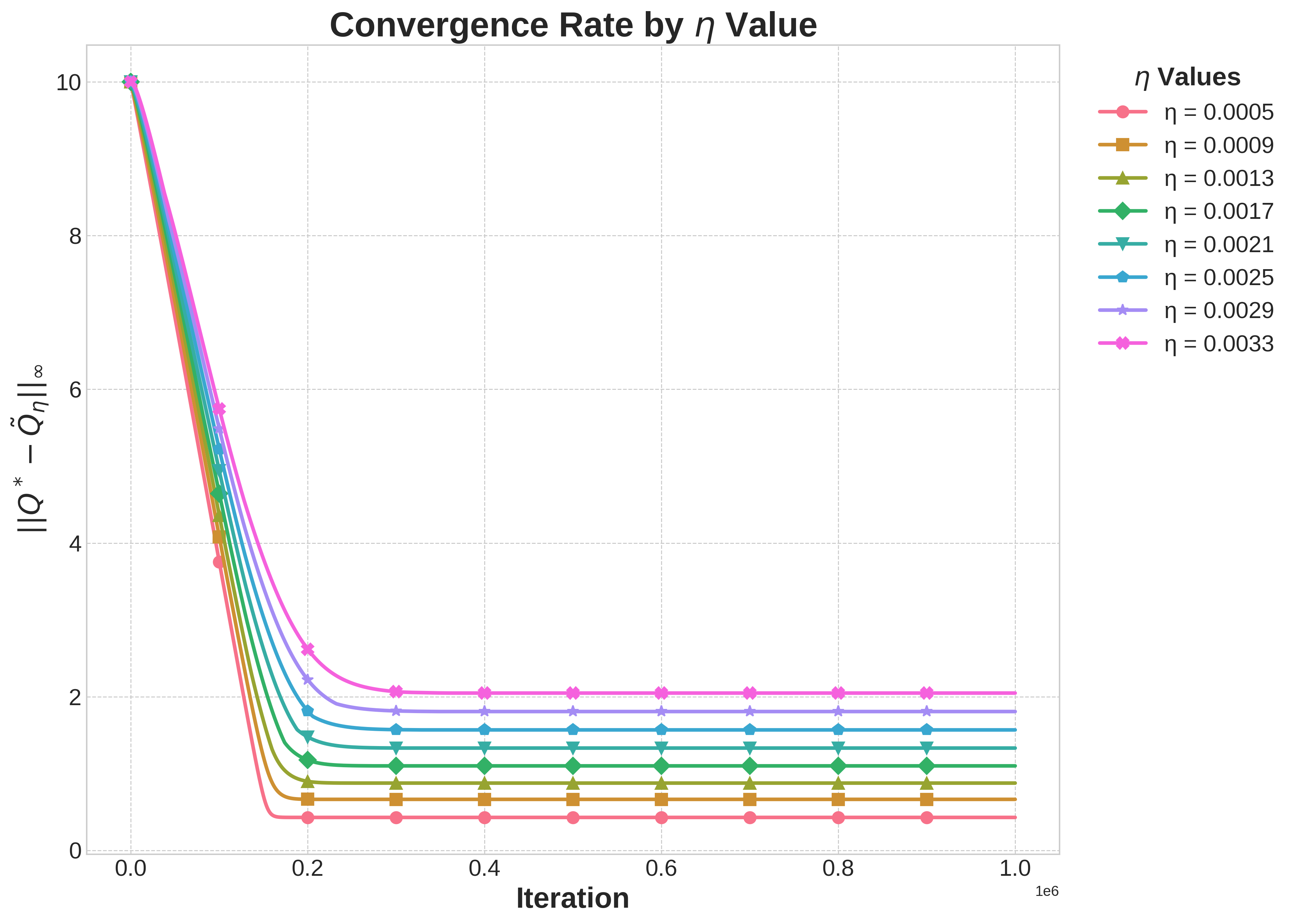}
    \caption{
        Comparison of the log-barrier coefficient $\eta$ on the convergence rate. The plot shows the max-norm error, $\|Q^* - \tilde{Q}_\eta\|_{\infty}$, between the learned and optimal Q-functions versus the number of training iterations. These results were obtained from experiments on a 6$\times$6 FrozenLake-v1 environment, where the ground-truth $Q^*$ was pre-computed using value iteration.
    }
    \label{fig:comparison1}
\end{figure}

We conducted an experiment with two primary objectives: first, to demonstrate that the gradient expression proposed in~\cref{lem:gradient} can serve as a valid gradient within a gradient descent framework, and second, to analyze how the log-barrier coefficient $\eta$ affects the performance. For this purpose, we used the FrozenLake-v1 environment from Gymnasium with a custom 6$\times$6 grid in a tabular setting. We represented the Q-function in a tabular form, for which the optimal Q-function, $Q^*$, can be pre-computed with guaranteed convergence using value iteration~\citep{puterman2014markov}. Then, we utilized the gradient descent method with the learning rate $\alpha =0.01$ for all experiments and measured the max-norm error between $Q^*$ and $ \tilde{Q}_\eta$. The overall convergence of the algorithm, as shown in~\cref{fig:comparison1}, validates our first objective by confirming that the expression from~\cref{lem:gradient} serves as a functional gradient.

% Paragraph 2: Focuses entirely on the second objective.
Our second objective is to analyze the impact of the coefficient $\eta$, which serves as an approximation factor in the log-barrier method. A higher value of $\eta$ corresponds to a relaxed approximation of the feasible set's boundary, leading to a smoother optimization landscape. Conversely, a lower value of $\eta$ yields a more precise approximation but creates a steeper barrier near the boundary. This theoretical property is clearly demonstrated in~\cref{fig:comparison1}, where the different error floors for each curve confirm the role of $\eta$ as an approximation factor.

\begin{figure}[h!]
    \centering
    \includegraphics[width=14cm, height=7cm]{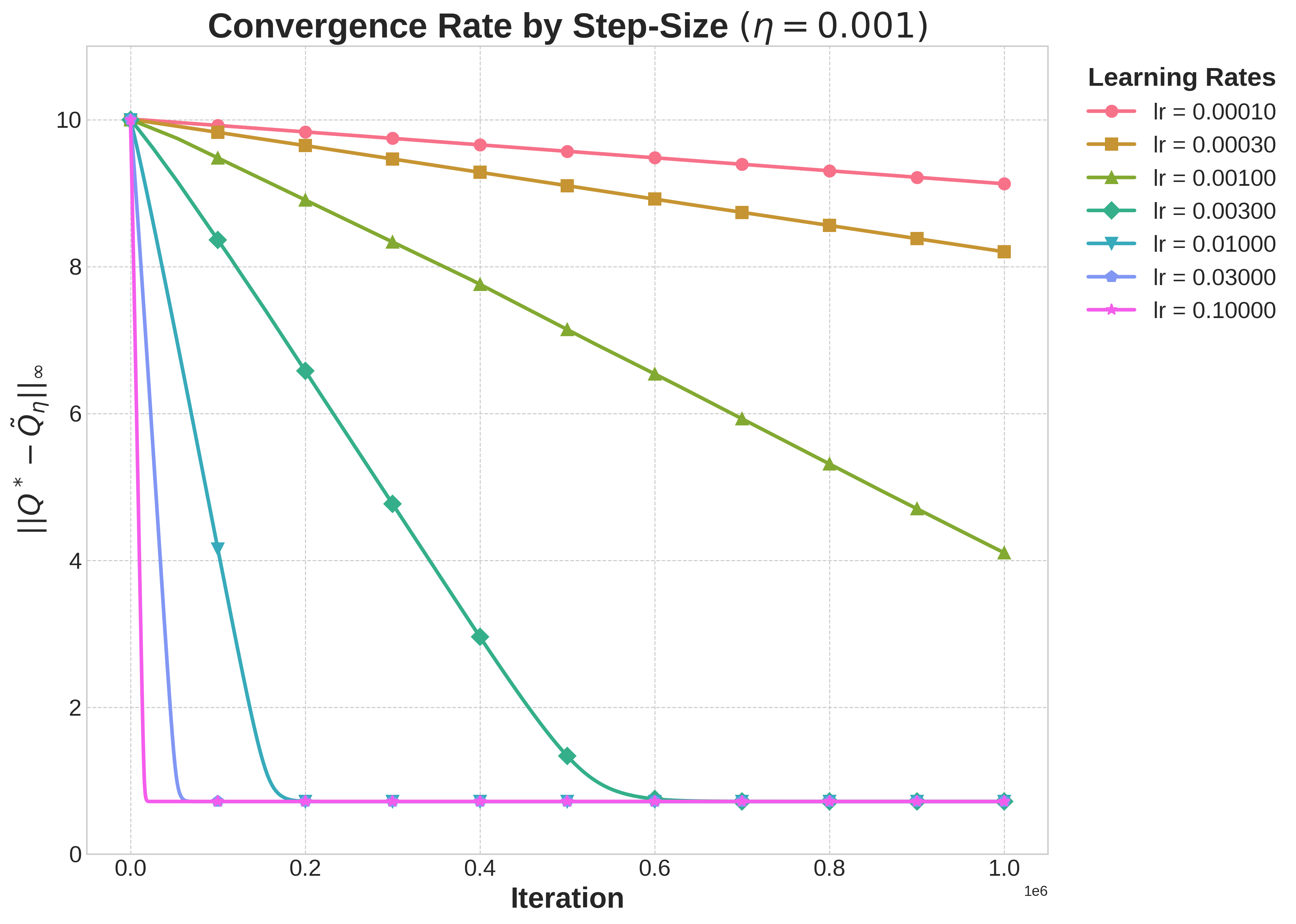}
    \caption{
        Comparison of the error evolutions with $\eta = 0.001$ for different learning rates. The plot shows the max-norm error, $\|Q^* - \tilde{Q}_\eta\|_{\infty}$, between the learned and optimal Q-functions versus the number of training iterations. These results were obtained from experiments on a 6$\times$6 FrozenLake-v1 environment, where the ground-truth $Q^*$ was pre-computed using value iteration.
    }
    \label{fig:comparison2}
\end{figure}

\begin{figure}[h!]
    \centering
    \includegraphics[width=14cm, height=7cm]{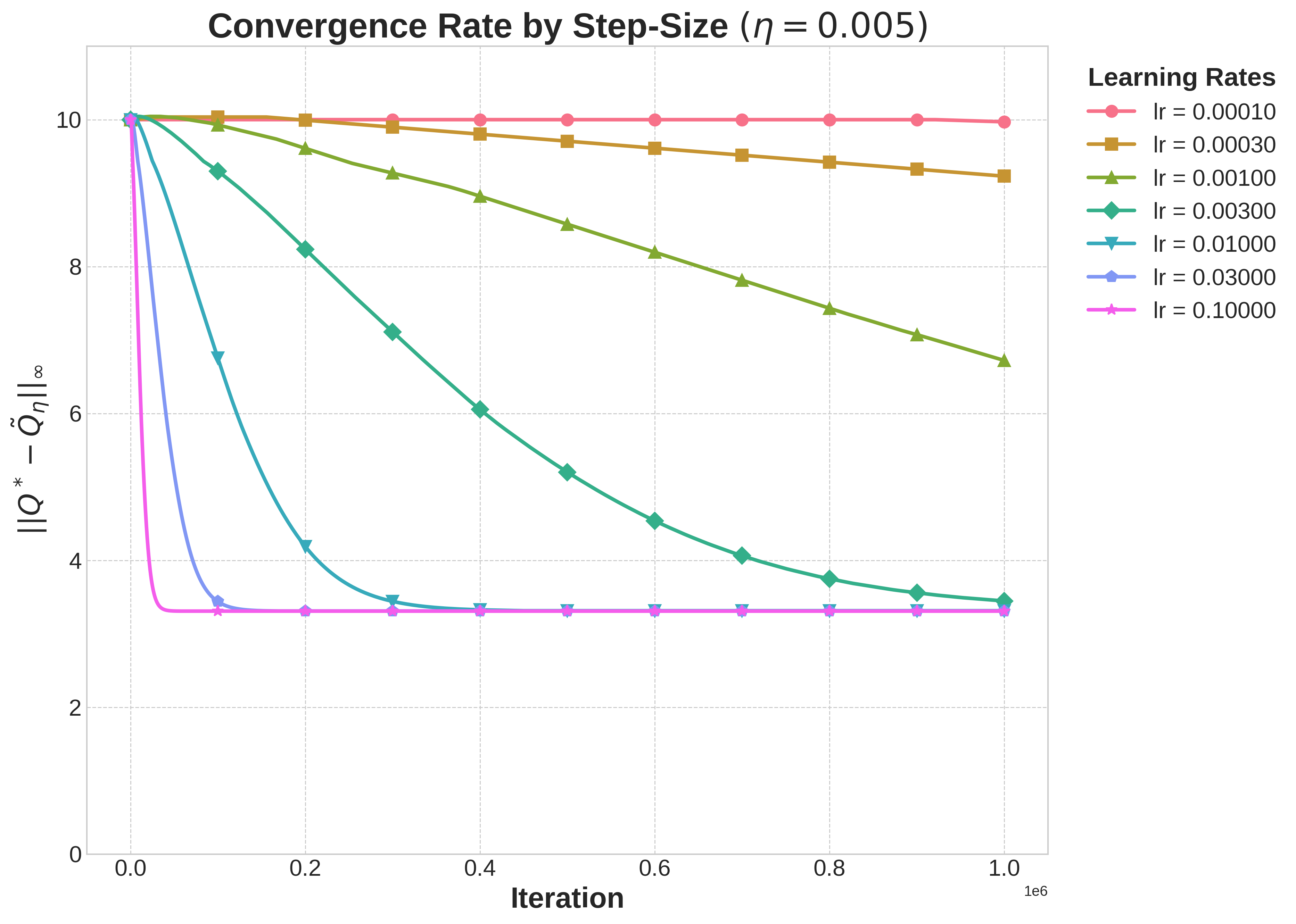}
    \caption{
        Comparison of the error evolutions with $\eta = 0.005$ for different learning rates. The plot shows the max-norm error, $\|Q^* - \tilde{Q}_\eta\|_{\infty}$, between the learned and optimal Q-functions versus the number of training iterations. These results were obtained from experiments on a 6$\times$6 FrozenLake-v1 environment, where the ground-truth $Q^*$ was pre-computed using value iteration.
    }
    \label{fig:comparison3}
\end{figure}
\cref{fig:comparison2} and \cref{fig:comparison3} show the error evolutions for different learning rates while keeping the barrier weight $\eta$ fixed. \cref{fig:comparison2} corresponds to $\eta = 0.001$ and \cref{fig:comparison3} to $\eta = 0.005$. As the plots show, larger learning rates lead to faster convergence in these experiments.

\begin{figure}[h!]
    \centering
    \includegraphics[width=13cm, height=7cm]{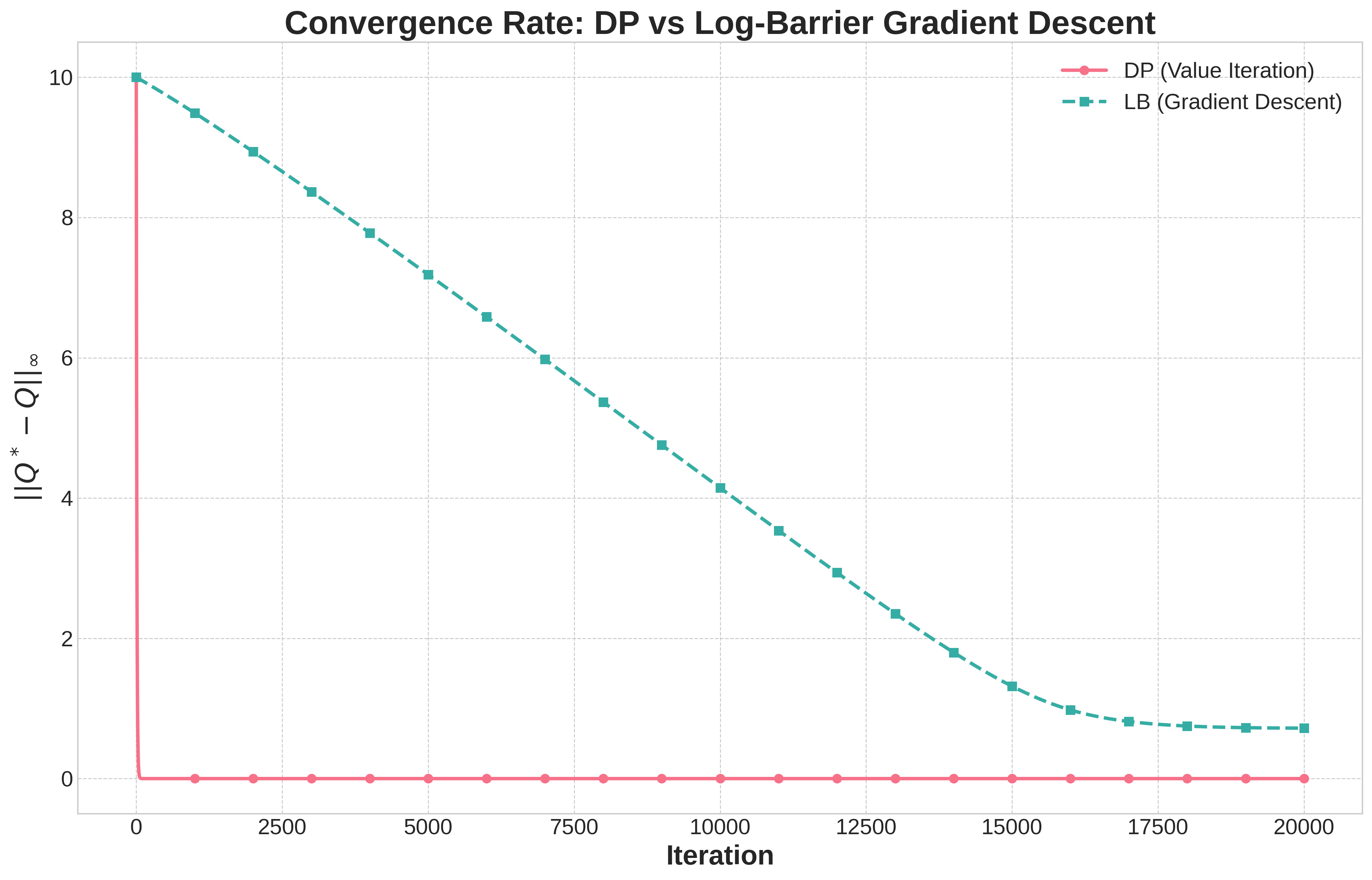}
    \caption{
        Comparison of the error evolutions of the classical dynamic programming and the gradient descent of the log-barrier objective with fixed $\eta$ and learning rate. The plot shows the max-norm error, $\|Q^* - \tilde{Q}_\eta\|_{\infty}$, between the learned and optimal Q-functions versus the number of training iterations. These results were obtained from experiments on a 6$\times$6 FrozenLake-v1 environment, where the ground-truth $Q^*$ was pre-computed using value iteration.
    }
    \label{fig:comparison4}
\end{figure}

\cref{fig:comparison4} compares the convergence of classical dynamic programming (specifically, value iteration) with gradient descent on the log-barrier objective. For the log-barrier method we used a fixed learning rate and fixed $\eta$; the learning rate was tuned to produce the best possible performance for that method. As~\cref{fig:comparison4} shows, value iteration converges faster than the gradient-descent log-barrier approach in this experiment.

Remark on the choice of $\eta$: Theoretically, as $\eta$ goes to zero, the log-barrier solution $\tilde Q_{\eta}$ approaches the true $Q^*$. This relationship is established in~\cref{thm:bounds2}. In practice, however, taking $\eta$ too small can cause numerical issues. For example, with very small $\eta$ the gradient-based estimate $\tilde Q_{\eta}$ may step outside the domain that satisfies the inequality constraints due to finite-precision and optimization error. The $\eta$ values reported in~\cref{tab:hyperparameters_simple} and~\cref{tab:hyperparameters_ddpg} were obtained by empirical tuning; in our experiments, we typically used moderate values of $\eta$ (see~\cref{tab:hyperparameters_simple} and~\cref{tab:hyperparameters_ddpg} for the exact settings). At present, a principled criterion for selecting $\eta$ has not been established, and we view the development of such a selection rule as important future work.

\section{Appendix: proof of~\cref{thm:bounds2}}\label{sec:app:proof:bound1}

For convenience of the reader, the statements of~\cref{thm:bounds2} are summarized as follows:
\begin{enumerate}
\item $\eta \min\limits_{(s,a,a') \in {\cal S} \times {\cal A} \times {\cal A}} w(s,a,a') < {\left\| {{{\tilde Q}_\eta } - {Q^*}} \right\|_\infty } \le \frac{{\eta \sum\limits_{(s,a,a') \in {\cal S} \times {\cal A} \times {\cal A}} {w(s,a,a')} }}{{{{\min }_{(s,a) \in {\cal S} \times {\cal A}}}\rho (s,a)}}$

\item $\eta (1 - \gamma )\min\limits_{(s,a,a') \in {\cal S} \times {\cal A} \times {\cal A}} w(s,a,a') < {\left\| {{{\tilde Q}_\eta } - T{{\tilde Q}_\eta }} \right\|_\infty } \le \frac{{(1 + \gamma )\eta \sum\limits_{(s,a,a') \in {\cal S} \times {\cal A} \times {\cal A}} {w(s,a,a')} }}{{\min_{(s,a) \in {\cal S} \times {\cal A}}}\rho (s,a)}$
\end{enumerate}

To prove~\cref{thm:bounds2}, we first need to prove the following lemma.
\begin{lemma}\label{lemma:bounds1}
Let us define
\begin{align*}
{{\tilde \lambda }_\eta }(s,a,a'): =  - \frac{{\eta w(s,a,a')}}{{(F{{\tilde Q}_\eta })(s,a,a') - {{\tilde Q}_\eta }(s,a)}},\quad (s,a,a') \in {\cal S} \times {\cal A} \times {\cal A}.
\end{align*}
Then, we have the following results:
\begin{enumerate}
\item $\eta (1 - \gamma ) \mathop{\min}\limits_{(s,a,a') \in {\cal S} \times {\cal A} \times {\cal A}} w(s,a,a') < {\left\| \tilde Q_\eta - T\tilde Q_\eta \right\|_\infty }$

\item $\eta \mathop{\min}\limits_{(s,a,a') \in {\cal S} \times {\cal A} \times {\cal A}} w(s,a,a') < {\tilde Q_\eta }(s,a) - Q^*(s,a),\quad \forall (s,a) \in {\cal S} \times {\cal A}$.

\item $\eta \mathop{\min}\limits_{(s,a,a') \in {\cal S} \times {\cal A} \times {\cal A}} w(s,a,a') < {\left\| {\tilde Q_\eta - {Q^*}} \right\|_\infty }$
\end{enumerate}

\end{lemma}
\begin{proof}
By~\cref{prop:dual-property1}, we have
\begin{align*}
0 < \frac{{\eta w(s,a,a')(1 - \gamma )}}{{{\tilde Q}_\eta(s,a) - (F{{\tilde Q}_\eta })(s,a,a')}} < 1,\quad \forall (s,a,a') \in {\cal S} \times {\cal A} \times {\cal A},
\end{align*}
which implies
\begin{align*}
\eta \mathop{\min }\limits_{(s,a,a') \in {\cal S} \times {\cal A} \times {\cal A}} w(s,a,a')(1 - \gamma ) < {\tilde Q}_\eta (s,a) - (F{{\tilde Q}_\eta })(s,a,a'),\quad \forall (s,a,a') \in {\cal S} \times {\cal A} \times {\cal A},
\end{align*}
and hence,
\begin{align*}
\eta (1 - \gamma )\mathop{\min }\limits_{(s,a,a') \in {\cal S} \times {\cal A} \times {\cal A}} w(s,a,a') < {{\tilde Q}_\eta }(s,a) - (T{{\tilde Q}_\eta })(s,a),\quad \forall (s,a) \in {\cal S} \times {\cal A}.
\end{align*}
By taking the absolute value on the right-hand side and taking the maximum over $(s,a)\in {\cal S}\times {\cal A}$, we can completes the proof of the first statement.
Next, for the second statement, we first note the following bounds:
\begin{align*}
&\eta (1 - \gamma ) \mathop{\min }\limits_{(s,a,a') \in {\cal S} \times {\cal A} \times {\cal A}} w(s,a,a')\\
<& \tilde Q_\eta(s,a) - (T\tilde Q_\eta)(s,a)\\
=& \tilde Q_\eta(s,a) - (T\tilde Q_\eta)(s,a) - {Q^*}(s,a) + (T Q^*)(s,a)\\
=& \tilde Q_\eta(s,a) - Q^*(s,a) - \gamma \sum_{s' \in {\cal S}} {P(s'|s,a)\{ {\max_{a' \in {\cal A}}}\tilde Q_\eta(s',a') - \max_{a' \in {\cal A}}Q^*(s',a')\} } \\
\le& \tilde Q_\eta(s,a) - Q^*(s,a) - \gamma \sum\limits_{s' \in {\cal S}} {P(s'|s,a)\{ \tilde Q_\eta(s',b) - {Q^*}(s',b)\} }\\
\le& {\tilde Q_\eta }(s,a) - {Q^*}(s,a) - \gamma \sum\limits_{s' \in {\cal S}} {P(s'|s,a){{\min }_{(s',a') \in {\cal S} \times {\cal A}}}\{ {{\tilde Q}_\eta }(s',a')}  - {Q^*}(s',a')\} \\
\le& \tilde Q_\eta(s,a) - Q^*(s,a) - \gamma {\min _{(s',a') \in {\cal S} \times {\cal A}}}\{ \tilde Q_\eta(s',a') - {Q^*}(s',a')\},\quad \forall (s,a)\in {\cal S}\times {\cal A},
\end{align*}
where in the fourth line, $b = \argmax _{a' \in {\cal A}} Q^*(s',a')$.
Taking the minimum over $(s,a)\in {\cal S}\times {\cal A}$ on both sides yields
\begin{align*}
(1-\gamma)\eta \mathop{\min }\limits_{(s,a,a') \in {\cal S} \times {\cal A} \times {\cal A}} w(s,a,a') <(1-\gamma) {\min _{(s',a') \in {\cal S} \times {\cal A}}}\{ \tilde Q_\eta(s',a') - Q^*(s',a')\},
\end{align*}
which is the second statement. Moreover, it also implies
\begin{align*}
\eta \mathop{\min }\limits_{(s,a,a') \in {\cal S} \times {\cal A} \times {\cal A}} w(s,a,a') <& {\min _{(s',a') \in {\cal S} \times {\cal A}}}\{ \tilde Q_\eta(s',a') - {Q^*}(s',a')\} \\
\le& {\min _{(s',a') \in {\cal S} \times {\cal A}}}|\tilde Q_\eta(s',a') - {Q^*}(s',a')|\\
\le& {\max _{(s',a') \in {\cal S} \times {\cal A}}}|\tilde Q_\eta(s',a') - {Q^*}(s',a')|\\
=& {\left\| {\tilde Q_\eta - Q^*} \right\|_\infty },
\end{align*}
which proves the last statement.
\end{proof}

Based on~\cref{lemma:bounds1}, we will now prove~\cref{thm:bounds2}.
\begin{proof}[Proof of the first statement of~\cref{thm:bounds2}]
Because the optimal solution $\tilde Q_\eta$ should be in the domain $\cal D$ of the log-barrier function, it should be strictly feasible for the primal LP. Moreover, $\tilde \lambda_\eta$ is feasible for the dual LP. Using these results, we can obtain the following lower bounds:
\begin{align}
&\sum\limits_{(s,a) \in {\cal S} \times {\cal A}} {\rho (s,a)Q^*(s,a)}\nonumber\\
=& \sum\limits_{(s,a,a') \in {\cal S} \times {\cal A} \times {\cal A}} {{\lambda^*}(s,a,a')R(s,a)}\nonumber\\
\ge& \sum\limits_{(s,a,a') \in {\cal S} \times {\cal A} \times {\cal A}} {\tilde \lambda_\eta (s,a,a')R(s,a)}\nonumber\\
=& \sum\limits_{(s,a,a') \in {\cal S} \times {\cal A} \times {\cal A}} {\tilde \lambda_\eta (s,a,a')R(s,a)}\nonumber\\
&  + \sum\limits_{(s,a) \in {\cal S} \times {\cal A}} {\tilde Q_\eta(s,a)\left\{ {\rho (s,a) + \gamma \sum\limits_{(s',a') \in {\cal S} \times {\cal A}} {\tilde \lambda_\eta (s',a',a)P(s|s',a')}  - \sum\limits_{a' \in {\cal A}} {\tilde \lambda_\eta (s,a,a')} } \right\}}\nonumber\\
=& L(\tilde Q_\eta,\tilde \lambda_\eta )\nonumber\\
=&\sum\limits_{(s,a) \in {\cal S} \times {\cal A}} {\rho (s,a){{\tilde Q}_\eta }(s,a)}  + \sum\limits_{(s,a,a') \in {\cal S} \times {\cal A} \times {\cal A}} {{{\tilde \lambda }_\eta }(s,a,a')\{ (F{{\tilde Q}_\eta })(s,a,a') - {{\tilde Q}_\eta }(s,a)\} }\nonumber\\
=& \sum\limits_{(s,a) \in {\cal S} \times {\cal A}} {\rho (s,a){{\tilde Q}_\eta }(s,a)}  + \sum\limits_{(s,a,a') \in {\cal S} \times {\cal A} \times {\cal A}} { - \eta \frac{{w(s,a,a')\{ (F{{\tilde Q}_\eta })(s,a,a') - {{\tilde Q}_\eta }(s,a)\} }}{{(F{{\tilde Q}_\eta })(s,a,a') - {{\tilde Q}_\eta }(s,a)}}} \nonumber\\
=& \sum\limits_{(s,a) \in {\cal S} \times {\cal A}} {\rho (s,a){{\tilde Q}_\eta }(s,a)}  - \eta \sum\limits_{(s,a,a') \in {\cal S} \times {\cal A} \times {\cal A}} {w(s,a,a')},\label{eq:appendix:1}
\end{align}
where the second line is due to the strong duality of the LP problem, the third and fourth lines are due to the fact that $\tilde \lambda_\eta$ is dual feasible, and the sixth line is obtained by rearranging terms.
On the other hand, since $\tilde Q_\eta \in {\cal D}$ is primal feasible, we have
\begin{align}
\sum\limits_{(s,a)\in {\cal S}\times {\cal A}} {\rho(s,a){Q^*}(s,a)}  \le \sum_{(s,a)\in {\cal S}\times {\cal A}} {\rho (s,a)\tilde Q_\eta(s,a)}.\label{eq:appendix:2}
\end{align}
Combining~\cref{eq:appendix:1} and~\cref{eq:appendix:2} leads to
\begin{align*}
0 \le \sum\limits_{(s,a) \in {\cal S} \times {\cal A}} {\rho (s,a)\{ {{\tilde Q}_\eta }(s,a) - {Q^*}(s,a)\} }  \le \eta \sum\limits_{(s,a,a') \in {\cal S} \times {\cal A} \times {\cal A}} {w(s,a,a')} .
\end{align*}
Using~\cref{lemma:bounds1}, we further derive
\begin{align*}
\sum\limits_{(s,a)\in {\cal S} \times {\cal A}} {\rho (s,a)\{ \tilde Q_\eta(s,a) - {Q^*}(s,a)\} }  \ge& {\min _{(s,a) \in {\cal S} \times {\cal A}}}\rho (s,a)\sum\limits_{(s,a) \in {\cal S} \times {\cal A}} {(\tilde Q_\eta(s,a) - {Q^*}(s,a))}\\
=& {\min _{(s,a) \in {\cal S} \times {\cal A}}}\rho (s,a){\left\| {\tilde Q_\eta - {Q^*}} \right\|_1}\\
\ge& {\min _{(s,a) \in {\cal S} \times {\cal A}}}\rho (s,a){\left\| {\tilde Q_\eta - {Q^*}} \right\|_\infty },
\end{align*}
where the last line comes from ${\left\|  \cdot  \right\|_1} \ge {\left\|  \cdot  \right\|_\infty }$.
By dividing both sides by ${{{\min }_{(s,a) \in {\cal S} \times {\cal A}}}\rho (s,a)} >0$, one gets
\begin{align*}
{\left\| {{{\tilde Q}_\eta } - {Q^*}} \right\|_\infty } \le \frac{{\eta \sum\limits_{(s,a,a') \in {\cal S} \times {\cal A} \times {\cal A}} {w(s,a,a')} }}{{{{\min }_{(s,a) \in {\cal S} \times {\cal A}}}\rho (s,a)}}.
\end{align*}
The lower bounds come from the second statement of~\cref{lemma:bounds1}.
\end{proof}

\begin{proof}[Proof of the second statement of~\cref{thm:bounds2}]
Next, using the reverse triangle inequality, one gets
\begin{align*}
{\left\| {\tilde Q_\eta - {Q^*}} \right\|_\infty } =& {\left\| {\tilde Q_\eta - T\tilde Q_\eta + T\tilde Q_\eta - {Q^*}} \right\|_\infty }\\
\ge& {\left\| {\tilde Q_\eta - T\tilde Q_\eta} \right\|_\infty } - {\left\| {T\tilde Q_\eta - T{Q^*}} \right\|_\infty }\\
\ge& {\left\| {\tilde Q_\eta - T\tilde Q_\eta} \right\|_\infty } - \gamma {\left\| {\tilde Q_\eta - {Q^*}} \right\|_\infty },
\end{align*}
where the last inequality is due to the contraction property of the Bellman operator with respect to $\infty$-norm.

Rearranging terms on both sides results in
\begin{align*}
{\left\| {{{\tilde Q}_\eta } - T{{\tilde Q}_\eta }} \right\|_\infty } \le (1 + \gamma ){\left\| {{{\tilde Q}_\eta } - Q^*} \right\|_\infty } \le \frac{{(1 + \gamma )\eta \sum\limits_{(s,a,a') \in {\cal S} \times {\cal A} \times {\cal A}} {w(s,a,a')} }}{{{{\min }_{(s,a) \in {\cal S} \times {\cal A}}}\rho (s,a)}}.
\end{align*}
The lower bounds come from the first statement of~\cref{lemma:bounds1}, and this completes the proof.
\end{proof}

Based on~\cref{thm:bounds2}, we can derive the following bound on ${\tilde Q}_\eta$.
\begin{lemma}[Bound on ${\tilde Q}_\eta$]\label{lemma:bound-on-Q-eta}
${\tilde Q}_\eta$ satisfies
\begin{align*}
\left\| \tilde Q_\eta \right\|_\infty  \le \frac{\eta \mathop{\sum}\limits_{(s,a,a') \in {\cal S} \times {\cal A} \times {\cal A}} w(s,a,a') }{\min_{(s,a) \in {\cal S} \times {\cal A}}\rho (s,a)} + \frac{r_{\max}}{1 - \gamma }.
\end{align*}
\end{lemma}
\begin{proof}
Using~\cref{thm:bounds2}, one has
\begin{align*}
\left\| \tilde Q_\eta  \right\|_\infty =& \left\| \tilde Q_\eta  - Q^* + Q^* \right\|_\infty \\
\le& \left\| \tilde Q_\eta - Q^* \right\|_\infty  + \left\| Q^* \right\|_\infty\\
\le& \frac{\eta \mathop{\sum}\limits_{(s,a,a') \in {\cal S} \times {\cal A} \times {\cal A}} w(s,a,a') }{\min_{(s,a) \in {\cal S} \times {\cal A}}\rho (s,a)} + \frac{r_{\max}}{1 - \gamma},
\end{align*}
where the last line is due to~\cref{thm:bounds2} and ${\left\| Q^* \right\|_\infty } \le \frac{r_{\max }}{1 - \gamma}$. This completes the proof.

\end{proof}

\section{Appendix: proof of~\cref{thm:bounds3}}\label{sec:app:proof:bound2}

For convenience of the reader, the statements of~\cref{thm:bounds3} are summarized as follows:
\begin{enumerate}
\item $J^{\pi ^*} - \eta \sum\limits_{(s,a,a') \in {\cal S} \times {\cal A} \times {\cal A}} {w(s,a,a')}  \le {J^{{\tilde \pi }_\eta }} \le J^{\pi ^*}$

\item ${J^{{\pi ^*}}} - \frac{{\eta (1 + \gamma )\sum\limits_{(s,a,a') \in {\cal S} \times {\cal A} \times {\cal A}} {w(s,a,a')} }}{{(1 - \gamma ){\min_{(s,a) \in {\cal S} \times {\cal A}}}\rho (s,a)}} \le {J^{{{\tilde \beta }_\eta }}} \le J^{\pi ^*}$

\item $-\frac{{\eta (1 + \gamma )\sum\limits_{(s,a,a') \in {\cal S} \times {\cal A} \times {\cal A}} {w(s,a,a')} }}{{(1 - \gamma ){{\min }_{(s,a) \in {\cal S} \times {\cal A}}}\rho (s,a)}} \le {J^{{\tilde \beta }_\eta }} - {J^{{\tilde \pi }_\eta}} \le \eta \sum\limits_{(s,a,a') \in {\cal S} \times {\cal A} \times {\cal A}} {w(s,a,a')}$
\end{enumerate}

\begin{proof}[Proof of the first statement]
Because the optimal solution $\tilde Q_\eta \in {\cal D}$ should be in the domain of the log-barrier function, it should be strictly feasible for the primal LP. Moreover, $\tilde \lambda_\eta$ is feasible for the dual LP. Using these results, we can derive the following inequalities:
\begin{align*}
&\sum\limits_{(s,a,a') \in {\cal S} \times {\cal A} \times {\cal A}} {{\lambda ^*}(s,a,a')R(s,a)}\\
\ge& \sum\limits_{(s,a,a') \in {\cal S} \times {\cal A} \times {\cal A}} {\tilde \lambda (s,a,a')R(s,a)}\\
=& \sum\limits_{(s,a,a') \in {\cal S} \times {\cal A} \times A} {\tilde \lambda (s,a,a')R(s,a)}\\
&  + \sum\limits_{(s,a) \in {\cal S} \times {\cal A}} {\tilde Q(s,a)\left\{ {\rho (s,a) + \gamma \sum\limits_{(s',a') \in {\cal S} \times {\cal A}} {\tilde \lambda (s',a',a)P(s|s',a')}  - \sum\limits_{a' \in {\cal A}} {\tilde \lambda (s,a,a')} } \right\}}\\
=& L(\tilde Q,\tilde \lambda )\\
=& \sum\limits_{(s,a) \in {\cal S} \times {\cal A}} {\rho (s,a)\tilde Q(s,a)}  - \eta \sum\limits_{(s,a,a') \in {\cal S} \times {\cal A} \times {\cal A}} {w(s,a,a')}\\
\ge& \sum\limits_{(s,a) \in {\cal S} \times {\cal A}} {\rho (s,a){Q^*}(s,a)}  - \eta \sum\limits_{(s,a,a') \in {\cal S} \times {\cal A} \times {\cal A}} {w(s,a,a')}\\
=& \sum\limits_{(s,a,a') \in {\cal S} \times {\cal A} \times {\cal A}} {{\lambda ^*}(s,a,a')R(s,a)}  - \eta \sum\limits_{(s,a,a') \in {\cal S} \times {\cal A} \times {\cal A}} {w(s,a,a')},
\end{align*}
where the first inequality is due to the fact that $\tilde \lambda_\eta$ is dual feasible, the second inequality is due to the fact that $\tilde Q_\eta$ is primal feasible, and the last line is due to the strong duality of the LP problem.
Summarizing the above results, one gets
\begin{align*}
\sum\limits_{(s,a,a') \in {\cal S} \times {\cal A} \times {\cal A}} {{\lambda ^*}(s,a,a')R(s,a)}  \ge& \sum\limits_{(s,a,a') \in {\cal S} \times {\cal A} \times {\cal A}} {\tilde \lambda (s,a,a')R(s,a)}\\
\ge& \sum\limits_{(s,a,a') \in {\cal S} \times {\cal A} \times {\cal A}} {{\lambda ^*}(s,a,a')R(s,a)}\\
&  - \eta \sum\limits_{(s,a,a') \in {\cal S} \times {\cal A} \times {\cal A}} {w(s,a,a')}.
\end{align*}
The proof of the first statement is then completed using~\cref{prop:dual-property1}.

\end{proof}

\begin{proof}[Proof of the second statement]
For the proof of the second statement, from~\cref{lemma:bounds1} and~\cref{thm:bounds2}, we have
\begin{align*}
&\eta (1 - \gamma )\mathop{\min }\limits_{(s,a,a') \in {\cal S} \times {\cal A} \times {\cal A}} w(s,a,a')\\
 \le& {{\tilde Q}_\eta }(s,a) - \left\{ {R(s,a) + \gamma {\max_{a'\in {\cal A}}}\sum\limits_{s' \in {\cal S}} {P(s'|s,a){{\tilde Q}_\eta }(s',a')} } \right\}\\
\le& \frac{\eta (1 + \gamma )\sum\limits_{(s,a,a') \in {\cal S} \times {\cal A} \times {\cal A}} w(s,a,a')}{{{\min_{(s,a) \in {\cal S} \times {\cal A}}}\rho (s,a)}},\quad \forall (s,a,a') \in {\cal S} \times {\cal A} \times {\cal A}.
\end{align*}
By plugging the primal $\eta$-policy ${\tilde \beta }_\eta$ into $a$ in the above inequality results in
\begin{align*}
&\eta (1 - \gamma )\mathop{\min }\limits_{(s,a,a') \in {\cal S} \times {\cal A} \times {\cal A}} w(s,a,a')\\
 \le& {{\tilde Q}_\eta }(s,{{\tilde \beta }_\eta }(s)) - R(s,{{\tilde \beta }_\eta }(s)) - \gamma \sum\limits_{s' \in {\cal S}} {P(s'|s,{{\tilde \beta }_\eta }(s)){{\tilde Q}_\eta }(s',{{\tilde \beta }_\eta }(s'))} \\
\le& \frac{{\eta (1 + \gamma )\sum\limits_{(s,a,a') \in {\cal S} \times {\cal A} \times {\cal A}} {w(s,a,a')} }}{{{{\min }_{(s,a) \in {\cal S} \times {\cal A}}}\rho (s,a)}},\forall (s,a) \in {\cal S} \times {\cal A}.
\end{align*}
Next, taking the expectation on both sides of the above inequality leads to
\begin{align*}
&\eta (1 - \gamma )\mathop{\min }\limits_{(s,a,a') \in {\cal S} \times {\cal A} \times {\cal A}} w(s,a,a')\\
 \le& {\mathbb E}[{{\tilde Q}_\eta }(s_k,{{\tilde \beta }_\eta }(s_k))|{{\tilde \beta }_\eta },\rho ] - {\mathbb E}[R(s_k,{{\tilde \beta }_\eta }(s_k))|{{\tilde \beta }_\eta },\rho ] - \gamma {\mathbb E}[{{\tilde Q}_\eta }(s_{k+1},{{\tilde \beta }_\eta }(s_{k+1}))|{{\tilde \beta }_\eta },\rho ]\\
\le& \frac{{\eta (1 + \gamma )\sum\limits_{(s,a,a') \in {\cal S} \times {\cal A} \times {\cal A}} {w(s,a,a')} }}{{{\min_{(s,a) \in {\cal S} \times {\cal A}}}\rho (s,a)}}.
\end{align*}
where the expectation is take with respect to the state $s_k$ and the next state $s_{k+1}$ generated under ${\tilde \beta }_\eta$ and the initial state distribution $\rho$. Multiplying both sides of the above inequality by $\gamma ^k$ yields
\begin{align*}
&\eta (1 - \gamma ){\gamma ^k}\mathop{\min }\limits_{(s,a,a') \in {\cal S} \times {\cal A} \times {\cal A}} w(s,a,a')\\
\le& {\gamma ^k} {\mathbb E}[{{\tilde Q}_\eta }({s_k},{{\tilde \beta }_\eta }({s_k}))|{{\tilde \beta }_\eta },\rho ] - {\gamma ^k} {\mathbb E}[R({s_k},{{\tilde \beta }_\eta }({s_k}))|{{\tilde \beta }_\eta },\rho ] - {\gamma ^{k + 1}} {\mathbb E}[{{\tilde Q}_\eta }(s_{k + 1},{{\tilde \beta }_\eta }({s_{k + 1}}))|{{\tilde \beta }_\eta },\rho ]\\
\le& \frac{{\eta (1 + \gamma )\sum\limits_{(s,a,a') \in {\cal S} \times {\cal A} \times {\cal A}} {w(s,a,a')} }}{{{\min_{(s,a) \in {\cal S} \times {\cal A}}}\rho (s,a)}}{\gamma ^k}.
\end{align*}

Now, summing them over all $k =0,1,2,\ldots$ leads to
\begin{align*}
\eta \mathop{\min }\limits_{(s,a,a') \in {\cal S} \times {\cal A} \times {\cal A}} w(s,a,a') \le& {\mathbb E}[{\tilde Q}_\eta(s_0,{{\tilde \beta }_\eta }({s_0}))|{{\tilde \beta }_\eta },\rho ] - {J^{{{\tilde \beta }_\eta }}}\\
 \le&s \frac{{\eta (1 + \gamma )\sum\limits_{(s,a,a') \in {\cal S} \times {\cal A} \times {\cal A}} {w(s,a,a')} }}{{(1 - \gamma ){\min_{(s,a) \in {\cal S} \times {\cal A}}}\rho (s,a)}},
\end{align*}
where we use
\[{J^{{\tilde \beta }_\eta } = {\mathbb E} \left[ {\left. {\sum\limits_{k = 0}^\infty  {{\gamma ^k}R({s_k},{{\tilde \beta }_\eta }(s_k))} } \right|{{\tilde \beta }_\eta },\rho } \right] = \sum\limits_{k = 0}^\infty  {\gamma ^k} {\mathbb E}[R(s_k,{\tilde \beta }_\eta (s_k))|{{\tilde \beta }_\eta },\rho ]}. \]

Rearranging terms leads to
\begin{align}
{\mathbb E}[{\tilde Q}_\eta(s_0,{{\tilde \beta }_\eta }({s_0}))] - \frac{{\eta (1 + \gamma )\sum\limits_{(s,a,a') \in {\cal S} \times {\cal A} \times {\cal A}} {w(s,a,a')} }}{{(1 - \gamma ){{\min }_{(s,a) \in {\cal S} \times {\cal A}}}\rho (s,a)}} \le {J^{{{\tilde \beta }_\eta }}} \le J^{\pi ^*}.\label{eq:101}
\end{align}

Moreover, since ${{\tilde Q}_\eta }(s,a) \ge {Q^*}(s,a)$, we have
\begin{align}
\sum\limits_{s \in {\cal S}} {\rho (s){{\tilde Q}_\eta }(s,{{\tilde \beta }_\eta }(s))}  \ge \sum\limits_{s \in {\cal S}} {\rho (s){{\tilde Q}_\eta }(s,{\pi ^*}(s))}  \ge \sum\limits_{s \in {\cal S}} {\rho (s){Q^*}(s,{\pi ^*}(s))}  = {J^{\pi ^*}}.\label{eq:102}
\end{align}
Combining~\cref{eq:101} and~\cref{eq:102} leads to the second statement.
\end{proof}

\begin{proof}[Proof of the third statement]
The last statement can be easily obtained by combining the first and second statements. This completes the overall proof.
\end{proof}

\section{Appendix: policy evaluation problem}\label{sec:app:policy-evaluation}

In this section, we present a summary of several properties of the LP formulation associated with policy evaluation.
Similar to the policy design problem discussed in the main text, a number of corresponding properties can be established through analogous arguments.
For convenience of the reader, let us write the LP form of the policy evaluation problem again as follows:
\begin{align}
&\mathop{\min }\limits_{Q \in {\mathbb R}^{|{\cal S}||{\cal A}|}} \sum\limits_{(s,a) \in {\cal S} \times {\cal A}} {\rho (s,a)Q(s,a)}\label{eq:primal-LP2}\\
&{\rm subject\,\, to}\nonumber\\
&R(s,a) + \gamma \sum\limits_{(s',a') \in {\cal S} \times {\cal A}} {P(s'|s,a)\pi (a'|s')Q(s',a')}  \le Q(s,a),\quad (s,a) \in {\cal S} \times {\cal A},\nonumber
\end{align}
where $R(s,a)$ is the expected reward conditioned on $(s,a)\in {\cal S}\times {\cal A}$, $\rho$ denotes any probability distribution over ${\cal S} \times {\cal A}$ with strictly positive support.
For convenience, let us define the Bellman operator
\begin{align*}
({T^\pi }Q)(s,a): = R(s,a) + \gamma \sum\limits_{(s',a') \in {\cal S} \times {\cal A}} {P(s'|s,a)\pi (a'|s')Q(s',a')} ,\quad (s,a) \in {\cal S} \times {\cal A}.
\end{align*}

Analogous to the policy design problem, it can be established that the optimal solution of the above LP corresponds to the Q-function, $Q^\pi$.
\begin{lemma}\label{lemma:optimality2}
The optimal solution of the LP in~\cref{eq:primal-LP2} is unique and given by $Q^\pi$
\end{lemma}
\begin{proof}
Let us assume that the optimal solution of the above LP in~\cref{eq:primal-LP2} is $\hat Q$.
Since $\hat Q$ is feasible, it satisfies
\begin{align*}
(T^\pi \hat Q)(s,a) \le \hat Q(s,a),\quad (s,a) \in {\cal S} \times {\cal A}.
\end{align*}
Equivalently, it can be written as $T^\pi \hat Q \le \hat Q$.
Because $T$ is a monotone operator, we have
\begin{align*}
\mathop{\lim }\limits_{k \to \infty } (T^\pi)^k\hat Q \le \hat Q \Rightarrow {Q^\pi} \le \hat Q.
\end{align*}
Moreover, since $T^\pi{Q^\pi} = Q^\pi$, $Q^\pi$ is a feasible solution, and hence, $\hat Q \le Q^\pi$.
Therefore, $Q^\pi \le \hat Q \le Q^\pi$, which implies that $\hat Q = Q^\pi$.
\end{proof}

We next present the dual LP~\citep[Chapter~5]{Boyd2004} corresponding to the policy evaluation LP.
\begin{lemma}\label{lemma:dual-LP3}
The dual problem of the primal LP in~\cref{eq:primal-LP2} is given by
\begin{align*}
&{\max _{\lambda  \ge 0}}\sum\limits_{(s,a) \in {\cal S} \times {\cal A}} {\lambda (s,a)R(s,a)}\\
&{\rm subject\,\, to}\\
&\lambda (s,a) - \gamma \sum\limits_{(i,j) \in S \times A} {P(s|i,j)\pi (a|s)\lambda (i,j)}  - \rho (s,a) = 0,\quad \forall (s,a) \in {\cal S} \times {\cal A}.
\end{align*}
\end{lemma}
\begin{proof}
Let us consider the Lagrangian function~\citep[Chapter~5]{Boyd2004}
\begin{align*}
L(Q,\lambda) =& \sum\limits_{(s,a) \in {\cal S} \times {\cal A}} {\rho (s,a)Q(s,a)}\\
&+ \sum\limits_{(s,a) \in {\cal S} \times {\cal A}} {\lambda (s,a)(R(s,a) + \gamma \sum\limits_{(s',a') \in {\cal S} \times {\cal A}} {P(s'|s,a)\pi (a'|s')Q(s',a') - Q(s,a)} )},
\end{align*}
where $\lambda$ is the Lagrangian multiplier vector (or dual variable). Next, the function can be written by
\begin{align*}
L(Q,\lambda ) =& \sum\limits_{(s,a) \in {\cal S} \times {\cal A}} {\lambda (s,a)R(s,a)} \\
&+ \sum\limits_{(s,a) \in {\cal S} \times {\cal A}} {Q(s,a)\left\{ {\lambda (s,a) - \gamma \sum\limits_{(i,j) \in {\cal S} \times {\cal A}} {P(s|i,j)\pi (a|s)\lambda (i,j)}  - \rho (s,a)} \right\}}.
\end{align*}

One can observe that the dual problem ${\max _{\lambda  \ge 0}}{\min _{Q \in {\mathbb R}^{|{\cal S}||{\cal A}|}}}L(Q,\lambda )$ is finite only when
\begin{align*}
\lambda (s,a) - \gamma \sum\limits_{(i,j) \in {\cal S} \times {\cal A}} {P(s|i,j)\pi (a|s)\lambda (i,j)}  - \rho (s,a) = 0,\quad \forall (s,a) \in {\cal S} \times {\cal A}.
\end{align*}
Therefore, one gets the dual problem.
\end{proof}

In what follows, we investigate the properties of feasible solutions of the dual LP introduced above.
\begin{theorem}\label{prop:dual-property2}
Suppose that the marginal
\begin{align*}
\rho (s) = \sum\limits_{a \in {\cal A}} {\rho (s,a)}
\end{align*}
represents the initial state distribution and $\lambda (s,a),\forall (s,a) \in {\cal S} \times {\cal A}$ is any dual feasible solution (all solutions satisfying the dual constraints). Then, we obtain the following results:
\begin{enumerate}
\item $(1-\gamma)\lambda$ is a probability distribution, i.e.,
\[(1-\gamma)\lambda (s,a) \ge 0,\quad \forall (s,a) \in {\cal S} \times {\cal A},\quad \sum\limits_{(s,a) \in {\cal S} \times {\cal A}} {(1-\gamma)\lambda (s,a)}  = 1.\]

\item Moreover, let us define the marginal distributions
\begin{align*}
\lambda (s): = \sum\limits_{a \in {\cal A}} {\lambda (s,a)} ,\quad \rho (s): = \sum\limits_{a \in {\cal A}} {\rho (s,a)}.
\end{align*}
Then, we have
\begin{align*}
\lambda (s') - \gamma \sum\limits_{s \in {\cal S}} {P^\xi(s'|s)\lambda (s)}  = \rho (s'),
\end{align*}
where
\[\xi ( \cdot |s): = \left[ {\begin{array}{*{20}{c}}
{\frac{\lambda (s,1)}{\sum\limits_{a \in {\cal A}} {\lambda (s,a)} }}&{\frac{{\lambda (s,2)}}{{\sum\limits_{a \in {\cal A}} {\lambda (s,a)} }}}& \cdots &{\frac{\lambda (s,|{\cal A}|)}{\sum\limits_{a \in {\cal A}} {\lambda (s,a)} }}
\end{array}} \right]\]
and
\[P^\xi (s'|s): = \sum_{a \in {\cal A}} {P(s'|s,a)\xi (a|s)} \]
is the probability of the transition from $s$ to $s'$ under the policy $\xi$.

\item $(1-\gamma)\lambda ( \cdot )$ is the discounted state occupation probability distribution defined as
\[(1 - \gamma )\lambda (s) = {\mu ^\xi }(s) = (1 - \gamma )\sum\limits_{k = 0}^\infty  {\gamma ^k {\mathbb P}[s_k = s|\xi ,\rho ]}. \]

\item The dual objective function satisfies
\[\sum\limits_{(s,a) \in {\cal S} \times {\cal A}} {\lambda (s,a)R(s,a)}  = \frac{1}{{1 - \gamma }}\sum\limits_{s \in {\cal S}} {{V^\xi }(s)\rho (s)}  = \frac{1}{{1 - \gamma }}{J^\xi }\]
\end{enumerate}
\end{theorem}
\begin{proof}

For the first statement, summing the dual constraints over $(s,a)\in {\cal S}\times {\cal A}$, we get
\begin{align*}
\Rightarrow \sum\limits_{(s,a) \in {\cal S} \times {\cal A}} {\lambda (s,a)}  - \gamma \sum\limits_{(s,a) \in {\cal S} \times {\cal A}} {\lambda (s,a)}  = 1,
\end{align*}
which leads to
\begin{align*}
\sum\limits_{(s,a) \in {\cal S} \times {\cal A}} {\lambda (s,a)}=1.
\end{align*}
It proves the first statement.
For the second statement, summing the dual constraints in~\cref{eq:dual-constraint1} over $a\in {\cal A}$ leads to
\begin{align*}
\lambda (s,a) - \gamma \sum\limits_{(i,j) \in {\cal S} \times {\cal A}} {P(s|i,j)\lambda (i,j,a)}  = \rho (s,a).
\end{align*}
Moreover, summing the above equation over $a\in {\cal A}$ leads to
\begin{align*}
\lambda (s) - \gamma \sum\limits_{(i,j) \in {\cal S} \times {\cal A}} {P(s|i,j)\lambda (i,j)}  = \rho (s),
\end{align*}
which can be further written as
\begin{align}
&\lambda (s) - \gamma \sum\limits_{(i,j) \in {\cal S} \times {\cal A}} {P(s|i,j)\lambda (i,j)}\nonumber\\
=& \lambda (s) - \gamma \sum\limits_{(i,j) \in {\cal S} \times {\cal A}} {P(s|i,j)\frac{\lambda (i,j)}{\sum\limits_{a \in {\cal A}} {\lambda (i,a)} }\sum\limits_{a \in {\cal A}} {\lambda (i,a)} }\nonumber \\
 =& \lambda (s) - \gamma \sum\limits_{(i,j) \in {\cal S} \times {\cal A}} {P(s|i,j)\xi (j|i)\lambda (i)}\nonumber \\
=& \lambda (s) - \gamma \sum\limits_{i \in {\cal S}} {P^\xi(s|i)\lambda (i)} \nonumber\\
=& \rho (s),\label{eq:412}
\end{align}
and this proves the second statement.

To prove the third statement, we can first consider the vectorized form of~\cref{eq:412}
\begin{align*}
\lambda  - \gamma {(P^\xi)^T}\lambda  = \rho,
\end{align*}
where
\begin{align*}
\lambda  =& \left[ {\begin{array}{*{20}{c}}
{\lambda (1)}\\
 \vdots \\
{\lambda (|{\cal S}|)}
\end{array}} \right] \in {\mathbb R}^{|{\cal S}|},\quad \rho  = \left[ {\begin{array}{*{20}{c}}
{\rho (1)}\\
 \vdots \\
{\rho (|S|)}
\end{array}} \right] \in {\mathbb R}^{|{\cal S}|},\\
{P^\xi } =& \left[ {\begin{array}{*{20}{c}}
{{P^\xi }(1|1)}& \cdots &{{P^\pi }(|{\cal S}||1)}\\
 \vdots & \ddots & \vdots \\
{{P^\xi }(1||S|)}& \cdots &{{P^\xi }(|{\cal S}|||{\cal S}|)}
\end{array}} \right] \in {\mathbb R}^{|{\cal S}| \times |{\cal S}|},
\end{align*}
which leads to
\begin{align*}
{\lambda ^T} = {\rho ^T} + \gamma {\rho ^T}{P_\xi } + {\gamma ^2}{\rho ^T}P_\xi ^2 +  \cdots.
\end{align*}
The expression can be written as
\[\lambda (s) = \sum\limits_{k = 0}^\infty  {{\gamma ^k}{\mathbb P}[s_k = s|\pi ,\xi ]}  = \frac{1}{1 - \gamma}{\mu ^\xi}(s).\]
This completes the proof of the third statement.

The last statement can be proved through the following identities:
\begin{align*}
\sum\limits_{(s,a) \in {\cal S} \times {\cal A}} {\lambda (s,a)R(s,a)}  =& \sum\limits_{(s,a) \in {\cal S} \times {\cal A}} {\xi (a|s)R(s,a)} \lambda (s)\\
=& \sum\limits_{s \in {\cal S}} {{V^\xi }(s)\rho (s)} \\
=& J^\xi,
\end{align*}
where $R^\xi$ denotes the expected reward under the policy $\xi$, and $V^\xi$ denotes the value function under $\xi$.
\end{proof}

Based on~\cref{prop:dual-property2}, we can readily obtain the following corollary.
\begin{corollary}
Suppose that the marginal $\rho (s): = \sum_{a \in {\cal A}} {\rho (s,a)} ,s \in {\cal S}$ represents the initial state distribution,
\begin{align*}
\pi(a|s): = \frac{{\rho (s,a)}}{{\sum\limits_{a \in {\cal A}} {\rho (s,a)} }}
\end{align*}
represents $\pi$, $Q^\pi$ is the primal optimal solution, and $\lambda^\pi$ is the dual optimal solution.
Then, the policy $\xi$ defined in the following is identical to $\pi$ in the sense that their corresponding MDP objective values are identical:
\begin{align*}
\xi( \cdot |s): = \left[ {\begin{array}{*{20}{c}}
{\frac{\lambda^\pi (s,1)}{\sum_{a \in {\cal A}} {\lambda^\pi (s,a)} }}&{\frac{\lambda^\pi (s,2)}{\sum_{a \in {\cal A}} {\lambda^\pi (s,a)} }}& \cdots &{\frac{\lambda^\pi (s,|{\cal A}|)}{\sum_{a \in {\cal A}} \lambda^\pi (s,a) }}
\end{array}} \right]
\end{align*}
which is identical to the MDP objective function $J^\pi$.
\end{corollary}
\begin{proof}
We first note that the optimal primal solution $Q^\pi$ does not depend on $\rho$, while the optimal dual solution $\lambda^\pi$ depends on $\rho$.
Keeping this in mind, let us suppose that the marginal
\begin{align*}
\rho (s) = \sum\limits_{a \in {\cal A}} {\rho (s,a)}
\end{align*}
represents the initial state distribution and
\begin{align*}
\pi(a|s): = \frac{{\rho (s,a)}}{{\sum\limits_{a \in {\cal A}} {\rho (s,a)} }}
\end{align*}
represents the policy $\pi$. The optimal primal objective function value is written as
\begin{align*}
\sum\limits_{(s,a) \in {\cal S} \times {\cal A}} {\rho (s,a){Q^\pi }(s,a)}  =& \sum\limits_{(s,a) \in {\cal S} \times {\cal A}} {\rho (s)\pi (a|s){Q^\pi }(s,a)}\\
 =& \sum\limits_{s \in {\cal S}} {\rho (s){V^\pi }(s)} \\
=& {J^\pi }
\end{align*}
On the other hand, the optimal dual objective function value is
\begin{align*}
\sum\limits_{(s,a) \in {\cal S} \times {\cal A}} {\lambda ^\pi(s,a)R(s,a)}=& \sum\limits_{(s,a) \in {\cal S} \times {\cal A}} {\xi (a|s)R(s,a)} {\lambda ^\pi }(s)\\
=& \sum\limits_{s \in {\cal S}} {R^\xi(s)} {\lambda ^\pi }(s)\\
=& \sum\limits_{s \in {\cal S}} {V^\xi(s)\rho (s)}\\
=& J^\xi.
\end{align*}
By the strong dualty, we have $J^\xi = J^\pi$, which means that the policy $\xi$ is identical to the policy $\pi$.
\end{proof}

Let us consider the objective function with log-barrier function
\[f_\eta^\pi (Q): = \sum\limits_{(s,a) \in {\cal S} \times {\cal A}} {Q(s,a)\rho (s,a)}  + \eta \sum\limits_{(s,a) \in {\cal S} \times {\cal A}} {w(s,a)\varphi \left( {({T^\pi }Q)(s,a) - Q(s,a)} \right)} \]
where $\eta >0$ is a weight parameter and $w(s,a)>0,(s,a)\in {\cal S} \times {\cal A}$ are weight parameters of the inequality constraints.
Then, we can prove that the corresponding optimal solution $\tilde Q_\eta^\pi: = \argmin_{Q \in {\cal D}}f_\eta^\pi(Q)$ approximates $Q^\pi$. The following lemma establishes such an error bound between $\tilde Q_\eta$ and $Q^*$, and, in addition, presents a bound on the Bellman error corresponding to
$\tilde Q_\eta$.
\begin{theorem}\label{thm:bounds4}
We have
\begin{enumerate}
\item $\eta \mathop{\min }\limits_{(s,a) \in {\cal S} \times {\cal A}} w(s,a) < {\left\| {{\tilde Q}^\pi_\eta } - Q^\pi \right\|_\infty } \le \frac{{\eta \sum\limits_{(s,a) \in {\cal S} \times {\cal A} } {w(s,a)} }}{{{\min_{(s,a) \in {\cal S} \times {\cal A}}}\rho (s,a)}}$

\item $\eta (1 - \gamma )\mathop{\min }\limits_{(s,a) \in {\cal S} \times {\cal A}} w(s,a) < {\left\| {{\tilde Q}^\pi_\eta - T^\pi {\tilde Q}^\pi_\eta } \right\|_\infty } \le \frac{{(1 + \gamma )\eta \sum\limits_{(s,a) \in {\cal S} \times {\cal A} } {w(s,a)} }}{{\min_{(s,a) \in {\cal S} \times {\cal A}}}\rho (s,a)}$
\end{enumerate}
\end{theorem}

To prove~\cref{thm:bounds4}, we first need to prove the following lemma.
\begin{lemma}\label{lemma:bounds4}
Let us define
\begin{align*}
{\tilde \lambda }_\eta(s,a): =  - \frac{\eta w(s,a)}{{(T^\pi{\tilde Q}^\pi_\eta)(s,a) - {\tilde Q}^\pi_\eta (s,a)}},\quad (s,a) \in {\cal S} \times {\cal A}.
\end{align*}
Then, we have the following results:
\begin{enumerate}
\item $\eta (1 - \gamma ) \mathop{\min }\limits_{(s,a) \in {\cal S} \times {\cal A}} w(s,a) < {\left\| \tilde Q^\pi_\eta - T^\pi\tilde Q^\pi_\eta \right\|_\infty }$

\item $\tilde Q^\pi_\eta (s,a) - Q^\pi(s,a) > \eta \mathop{\min }\limits_{(s,a) \in {\cal S} \times {\cal A}} w(s,a),\quad \forall (s,a)\in {\cal S} \times {\cal A}$

\item $\eta \mathop{\min }\limits_{(s,a) \in {\cal S} \times {\cal A}} w(s,a) < {\left\| {\tilde Q^\pi_\eta - Q^\pi} \right\|_\infty }$
\end{enumerate}

\end{lemma}
\begin{proof}
Following similar lines as in the proof of~\cref{prop:dual-property1}, we can derive
\begin{align*}
0 < \frac{{\eta w(s,a)(1 - \gamma )}}{{\tilde Q}^\pi_\eta(s,a) - (T^\pi{\tilde Q}^\pi_\eta)(s,a)} < 1,\quad \forall (s,a) \in {\cal S} \times {\cal A},
\end{align*}
which implies
\begin{align*}
\eta \mathop{\min }\limits_{(s,a) \in {\cal S} \times {\cal A}} w(s,a)(1 - \gamma ) < {\tilde Q}^\pi_\eta (s,a) - (T^\pi{\tilde Q}^\pi_\eta)(s,a),\quad \forall (s,a) \in {\cal S} \times {\cal A},
\end{align*}
By taking the absolute value on the right-hand side and taking the maximum over $(s,a)\in {\cal S}\times {\cal A}$, we can obtain the first statement.

Next, for the second statement, we first note the following bounds:
\begin{align*}
&\eta (1 - \gamma ) \mathop{\min }\limits_{(s,a) \in {\cal S} \times {\cal A} } w(s,a)\\
<& \tilde Q^\pi_\eta(s,a) - (T^\pi\tilde Q^\pi_\eta)(s,a)\\
=& \tilde Q_\eta(s,a) - (T^\pi\tilde Q_\eta)(s,a) - Q^\pi(s,a) + (T^\pi Q^\pi)(s,a)\\
=& \tilde Q_\eta ^\pi (s,a) - {Q^\pi }(s,a) - \gamma \sum\limits_{s' \in {\cal S}} {P(s'|s,a)\sum\limits_{a' \in {\cal A}} {\pi (a'|s')\{ } \tilde Q_\eta ^\pi (s',a') - {Q^\pi }(s',a')\} }\\
\le& {\tilde Q_\eta }(s,a) - {Q^*}(s,a) - \gamma \sum\limits_{s' \in {\cal S}} {P(s'|s,a){{\min }_{(s',a') \in {\cal S} \times {\cal A}}}\{ \tilde Q_\eta ^\pi (s',a')}  - {Q^\pi }(s',a')\}\\
\le& \tilde Q^\pi_\eta(s,a) - Q^\pi(s,a) - \gamma {\min _{(s',a') \in {\cal S} \times {\cal A}}}\{ \tilde Q^\pi_\eta(s',a') - Q^\pi(s',a')\},\quad \forall (s,a)\in {\cal S}\times {\cal A},
\end{align*}
where we use the Bellman equation ${Q^\pi } = {T^\pi }{Q^\pi }$ in the first equality.

Taking the minimum over $(s,a)\in {\cal S}\times {\cal A}$ on both sides yields
\begin{align*}
(1-\gamma) \eta \mathop{\min }\limits_{(s,a) \in {\cal S} \times {\cal A}} w(s,a) <(1-\gamma) {\min _{(s',a') \in {\cal S} \times {\cal A}}}\{ \tilde Q_\eta^\pi(s',a') - Q^\pi(s',a')\},
\end{align*}
which is the second statement. Moreover, it also implies
\begin{align*}
\eta \mathop{\min }\limits_{(s,a) \in {\cal S} \times {\cal A}} w(s,a) <& {\min _{(s',a') \in {\cal S} \times {\cal A}}}\{ \tilde Q_\eta^\pi(s',a') - Q^\pi(s',a')\} \\
\le& {\min _{(s',a') \in {\cal S} \times {\cal A}}}|\tilde Q_\eta^\pi(s',a') - Q^\pi(s',a')|\\
\le& {\max _{(s',a') \in {\cal S} \times {\cal A}}}|\tilde Q_\eta^\pi(s',a') - Q^\pi(s',a')|\\
=& {\left\| \tilde Q_\eta^\pi - Q^\pi \right\|_\infty },
\end{align*}
which proves the last statement.
\end{proof}

Based on~\cref{lemma:bounds4}, we will now prove~\cref{thm:bounds4}.
\begin{proof}[Proof of the first statement of~\cref{thm:bounds4}]
Because the optimal solution $\tilde Q_\eta^\pi$ should be in the domain of the log-barrier function, it should be strictly feasible for the primal LP. Moreover, $\tilde \lambda_\eta^\pi$ is feasible for the dual LP. Using these results, we can obtain the following lower bounds:
\begin{align}
&\sum\limits_{(s,a) \in {\cal S} \times {\cal A}} {\rho (s,a)Q^\pi(s,a)}\nonumber\\
=& \sum\limits_{(s,a) \in {\cal S} \times {\cal A}} {\lambda^\pi(s,a)R(s,a)}\nonumber\\
\ge& \sum\limits_{(s,a) \in {\cal S} \times {\cal A}} {\tilde \lambda_\eta^\pi (s,a)R(s,a)}\nonumber\\
=& \sum\limits_{(s,a) \in {\cal S} \times {\cal A}} {\tilde \lambda_\eta^\pi (s,a)R(s,a)}\nonumber\\
& + \sum\limits_{(s,a) \in {\cal S} \times {\cal A}} {{{\tilde Q}_\eta^\pi }(s,a)\left\{ {\rho (s,a) + \gamma \sum\limits_{(s',a') \in {\cal S} \times {\cal A}} {{{\tilde \lambda }_\eta^\pi }(s',a')P(s|s',a')}  - {{\tilde \lambda }_\eta^\pi }(s,a)} \right\}}\nonumber\\
=& L(\tilde Q_\eta^\pi,\tilde \lambda_\eta^\pi )\nonumber\\
=&\sum\limits_{(s,a) \in {\cal S} \times {\cal A}} {\rho (s,a){{\tilde Q}_\eta^\pi }(s,a)}  + \sum\limits_{(s,a) \in {\cal S} \times {\cal A}} {{\tilde \lambda }_\eta^\pi (s,a)\{ (T^\pi \tilde Q_\eta ^\pi )(s,a) - \tilde Q_\eta ^\pi (s,a)\} }\nonumber\\
=&\sum\limits_{(s,a) \in {\cal S} \times {\cal A}} {\rho (s,a){\tilde Q_\eta^\pi }(s,a)}  + \sum\limits_{(s,a) \in {\cal S} \times {\cal A}} { - \eta \frac{w(s,a)\{ (T^\pi \tilde Q_\eta ^\pi )(s,a) - \tilde Q_\eta ^\pi (s,a)\} }{(T^\pi \tilde Q_\eta ^\pi )(s,a,a') - \tilde Q_\eta ^\pi (s,a)}} \nonumber\\
=& \sum\limits_{(s,a) \in {\cal S} \times {\cal A}} {\rho (s,a){\tilde Q}_\eta^\pi (s,a)}  - \eta \sum\limits_{(s,a) \in {\cal S} \times {\cal A}}w(s,a),\label{eq:appendix:11}
\end{align}
where the second line is due to the strong duality of the LP problem, the third and fourth lines are due to the fact that $\tilde \lambda_\eta^\pi$ is dual feasible, and the sixth line is obtained by rearranging terms.
On the other hand, since $\tilde Q_\eta^\pi \in {\cal D}$ is primal feasible, we have
\begin{align}
\sum\limits_{(s,a)\in {\cal S}\times {\cal A}} {\rho(s,a){Q^*}(s,a)}  \le \sum_{(s,a)\in {\cal S}\times {\cal A}} {\rho (s,a)\tilde Q_\eta^\pi(s,a)}.\label{eq:appendix:12}
\end{align}
Combining~\cref{eq:appendix:11} and~\cref{eq:appendix:12} leads to
\begin{align*}
0 \le \sum\limits_{(s,a) \in {\cal S} \times {\cal A}} {\rho (s,a)\{ {{\tilde Q}_\eta^\pi }(s,a) - Q^\pi(s,a)\} }  \le \eta \sum\limits_{(s,a) \in {\cal S} \times {\cal A}} w(s,a).
\end{align*}
Using~\cref{lemma:bounds4} and ${\left\|  \cdot  \right\|_1} \ge {\left\|  \cdot  \right\|_\infty }$, we further derive
\begin{align*}
\sum\limits_{(s,a)\in {\cal S} \times {\cal A}} {\rho (s,a)\{ \tilde Q_\eta^\pi(s,a) - Q^\pi(s,a)\} }  \ge& {\min _{(s,a) \in {\cal S} \times {\cal A}}}\rho (s,a)\sum\limits_{(s,a) \in {\cal S} \times {\cal A}} {(\tilde Q_\eta^\pi(s,a) - Q^\pi(s,a))}\\
=& {\min _{(s,a) \in {\cal S} \times {\cal A}}}\rho (s,a){\left\| {\tilde Q_\eta^\pi - Q^\pi} \right\|_1}\\
\ge& {\min _{(s,a) \in {\cal S} \times {\cal A}}}\rho (s,a){\left\| {\tilde Q_\eta^\pi - Q^\pi} \right\|_\infty }.
\end{align*}
where the first inequality follows from $\tilde Q_\eta^\pi(s,a) - Q^\pi(s,a)>0$.
By dividing both sides by ${{{\min }_{(s,a) \in {\cal S} \times {\cal A}}}\rho (s,a)} >0$ leads to
\begin{align*}
{\left\| {{\tilde Q}_\eta^\pi } - Q^\pi \right\|_\infty } \le \frac{\eta \sum\limits_{(s,a) \in {\cal S} \times {\cal A}} w(s,a) }{{\min_{(s,a) \in {\cal S} \times {\cal A}}}\rho (s,a)}.
\end{align*}
The lower bounds come from the second statement of~\cref{lemma:bounds1}.
\end{proof}

\begin{proof}[Proof of the second statement of~\cref{thm:bounds4}]
Next, using the reverse triangle inequality, one gets
\begin{align*}
{\left\| {\tilde Q_\eta^\pi - Q^\pi} \right\|_\infty } =& {\left\| {\tilde Q^\pi_\eta - T^\pi\tilde Q_\eta^\pi + T^\pi \tilde Q_\eta^\pi - Q^\pi} \right\|_\infty }\\
\ge& {\left\| {\tilde Q_\eta^\pi - T^\pi \tilde Q_\eta^\pi} \right\|_\infty } - {\left\| {T^\pi\tilde Q_\eta^\pi - T^\pi Q^\pi} \right\|_\infty }\\
\ge& {\left\| {\tilde Q_\eta^\pi - T^\pi\tilde Q_\eta^\pi} \right\|_\infty } - \gamma {\left\| {\tilde Q_\eta^\pi - Q^\pi} \right\|_\infty },
\end{align*}
where the last inequality comes from the contraction of the Bellman operator.

Rearranging terms on both sides results in
\begin{align*}
{\left\| {\tilde Q^\pi_\eta } - T^\pi {\tilde Q}_\eta^\pi  \right\|_\infty } \le (1 + \gamma ){\left\| {\tilde Q_\eta^\pi  - Q^\pi} \right\|_\infty } \le \frac{{(1 + \gamma )\eta \sum\limits_{(s,a) \in {\cal S} \times {\cal A}} w(s,a) }}{{\min_{(s,a) \in {\cal S} \times {\cal A}}}\rho (s,a)}.
\end{align*}
The lower bounds come from the first statement of~\cref{lemma:bounds4}, and this completes the proof.
\end{proof}

\section{Appendix:~\cref{prop:upper-surrogate}}\label{sec:app:upper-surrogate}

In~\cref{sec:DeepRL}, we introduce a novel DQN algorithm inspired by the idea of standard DQN~\citep{mnih2015human} and consider the following loss function:
\begin{align*}
L(\theta ): = \frac{1}{{|B|}}\sum\limits_{(s,a,r,s') \in B,a' \in {\cal A}} {\left[ {{Q_\theta }(s,a) + \eta \varphi (r + \gamma {Q_\theta }(s',a') - {Q_\theta }(s,a))} \right]} ,
\end{align*}
where $B$ is a mini-batch uniformly sampled from the experience replay buffer $D$, $|B|$ is the size of the mini-batch, $Q_\theta$ is a deep neural network approximation of Q-function, $\theta\in {\mathbb R}^m$ is the parameter to be determined, and $(s,a,r,s')$ is the transition sample of the state-action-reward-next state. The loss function can be seen as a stochastic approximation of $f_\eta$, where $\rho$ and $w$ can be set to probability distributions corresponding to the replay buffer. However, this stochastic approximation is generally biased because it approximates a function in which the state transition probabilities appear outside the log-barrier function. In fact, by applying Jensen’s inequality, we can show that the above loss function is essentially an unbiased stochastic approximation of an upper bounding surrogate function of $f_\eta$. The related result is first introduced below.
\begin{proposition}\label{prop:upper-surrogate}
We have
\begin{align*}
{f_\eta }(Q) \le& \sum\limits_{(s,a) \in {\cal S} \times {\cal A}} {\rho (s,a)Q(s,a)} \\
&+ \eta \sum\limits_{(s,a,s',a') \in {\cal S} \times {\cal A} \times {\cal S} \times {\cal A}} {P(s'|s,a)w(s,a,a')\varphi \left( {r(s,a,s') + \gamma Q(s',a') - Q(s,a)} \right)}\\
=:& g_\eta(Q)
\end{align*}
\end{proposition}
\begin{proof}
It can be proved directly using Jensen's inequality.
\end{proof}

The above result employed Jensen’s inequality to derive the upper-bounding surrogate function $g_\eta$. Relating this to the loss function $L$, we observe that the surrogate function can be obtained by taking the expectation of $L$. In other words, under the assumption that the mini-batch samples are i.i.d., the loss function $L$ can be shown to be an unbiased stochastic approximation of the upper surrogate function.
Of course, as noted in the main text, if the environment is deterministic, then the original function $f_\eta$ and the upper surrogate function $g_\eta$ are identical. Therefore, $L$ is no longer a stochastic approximation of the upper surrogate function; it is now an exact representation.

\section{Appendix: log-barrier DQN algorithm}
\label{appendix:DQN algo}
Here, we discuss the implementation details of the deep RL variant of the proposed method. Since the loss function $L$
\begin{align*}
L(\theta ): = \frac{1}{{|B|}}\sum\limits_{(s,a,r,s') \in B,a' \in {\cal A}} {\left[ {Q_\theta(s,a) + \eta \varphi (r + \gamma {Q_\theta }(s',a') - {Q_\theta }(s,a))} \right]} ,
\end{align*}
relies on random samples of state–action pairs, applying gradient descent does not update the parameters for all state–action pairs simultaneously. As a result, there is no guarantee that $Q_\theta$ strictly satisfies the inequality constraints of the LP during the update process. In fact, $Q_\theta$ corresponding to state–action pairs that are not sampled may move outside the feasible region defined by the inequality constraints. Such behavior can significantly undermine the stability of learning. To address this issue, in practical implementations, we replace $L$ with the following alternative loss function $\tilde L$ as follows:
\[\tilde  L(\theta ): = \frac{1}{{|B|}}\sum\limits_{(s,a,r,s') \in B} {\left[ {{Q_\theta }(s,a) + \eta \sum\limits_{a' \in A} {h (r + \gamma {Q_\theta }(s',a') - {Q_\theta }(s,a))} } \right]} \]
where
\[h(x): = \left\{ {\begin{array}{*{20}{c}}
{\varphi (x - \varepsilon )\quad {\rm{if}}\quad x < 0}\\
{\nu x\quad {\rm{if}}\quad x \ge 0}
\end{array}} \right.\]
$\varepsilon>0$ is a small number, e.g. $\varepsilon = 10^{-6}$, added for numerical stability and $\nu >0$ is a large number, e.g. $\nu = 10^3$, in order to enforce $Q_\theta$ to be feasible when the TD error $r + \gamma {Q_\theta }(s',a') - {Q_\theta }(s,a)$ is nonnegative. Based on these details, the overall deep RL version is given in~\cref{algo:log-barrier-q}.
\begin{algorithm}[ht!]
\caption{Log-barrier deep Q-learning}
\begin{algorithmic}[1]
\State Initialize replay memory $D$ to capacity $|D|$
\State Initialize the parameter $\theta \in {\mathbb R}^m$ such that ${Q_\theta } \cong \kappa {{\bf{1}}_{|{\cal S}||{\cal A}|}}$ with large enough $\kappa>0 $
\State Set $k=0$

\For{Episode $i \in \{1,2,\ldots\}$}

\State Observe $s_0$

\For{$t \in\{0,1,\ldots, \tau-1\}$}

\State Take an action $a_t$ according to
\begin{align*}
a_t = \left\{ {\begin{array}{*{20}c}
   {\arg \max _{a \in {\cal A}} Q_{\theta} (s_t,a)\quad {\rm{with}}\,\,{\rm{probability}}\,\,1 - \varepsilon }  \\
   {a \sim {\rm{uniform}}({\cal A})\quad {\rm{with}}\,\,{\rm{probability}}\,\,\varepsilon }  \\
\end{array}} \right.
\end{align*}

\State Observe $r_{t+1},s_{t+1}$
\State Store the transition $(s_t ,a_t ,r_{t+1} ,s_{t + 1} )$ in $D$
\State Sample uniformly a random mini-batch $B$ of transitions $(s,a,r,s')$ from $D$

\State Perform few gradient descent steps
\[\theta  \leftarrow \theta  - \alpha {\nabla _\theta }\tilde L(\theta )\]

\State Set $k \leftarrow k+1$

\EndFor

\EndFor
\end{algorithmic}\label{algo:log-barrier-q}
\end{algorithm}
In~\cref{algo:log-barrier-q}, ${{\bf{1}}_{|{\cal S}||{\cal A}|}} \in {\mathbb R}^{|{\cal S}||{\cal A}|}$ is the vector with all entries one. The initialization, $Q_\theta \cong \kappa {{\bf{1}}_{|{\cal S}||{\cal A}|}}$, is motivated by the observation that a candidate strictly feasible solution to the LP formulation in~\cref{eq:primal-LP1} takes the form $\kappa {{\bf{1}}_{|{\cal S}||{\cal A}|}}$ with a large $\kappa>0$. To implement this initialization within a deep neural network, one may set all weights to zero while retaining only a constant bias term.

\section{Appendix: log-barrier DDPG algorithm}
\label{appendix:DDPG algo}
Here, we discuss the implementation details of the deep RL variant of the proposed method for policy evaluation. Since the loss function $L_{\text{critic}}(\theta ;\pi_w)$
\begin{align*}
L_{\text{critic}}(\theta ;\pi_w): = \frac{1}{{|B|}}\sum\limits_{(s,a,r,s') \in B} {\left[ {{Q_\theta }(s,a) + \eta \varphi (r + \gamma {Q_\theta }(s',{\pi _w}(s')) - {Q_\theta }(s,a))} \right]} ,
\end{align*}
for policy evaluation relies on random samples of state–action pairs, applying gradient descent does not update the parameters for all state–action pairs simultaneously. Similar to the previous section, in practical implementations, we replace $L$ with the following alternative loss function $\tilde L$ as follows:
\[\tilde L_{\text{critic}}(\theta;\pi_w ): = \frac{1}{{|B|}}\sum\limits_{(s,a,r,s') \in B} {\left[ {{Q_\theta }(s,a) + \eta h(r + \gamma {Q_\theta }(s',{\pi _w}(s')) - {Q_\theta }(s,a))} \right]} \]
where
\[h(x): = \left\{ {\begin{array}{*{20}{c}}
{\varphi (x - \varepsilon )\quad {\rm{if}}\quad x < 0}\\
{\nu x\quad {\rm{if}}\quad x \ge 0}
\end{array}} \right.\]
$\varepsilon>0$ is a small number, e.g. $\varepsilon = 10^{-6}$, added for numerical stability and $\nu >0$ is a large number, e.g. $\nu = 10^3$, in order to enforce $Q_\theta$ to be feasible when the TD error $r + \gamma {Q_\theta }(s',\pi_w(s')) - {Q_\theta }(s,a)$ is nonnegative. As usual, the actor loss function~\citep{lillicrap2015continuous} is given by
\begin{align*}
{L_{\rm actor}}(w): = \frac{1}{|B|}\sum\limits_{(s,a,r,s') \in B} {(-{Q_\theta }(s,{\pi _w}(s))}.
\end{align*}
The overall detailed algorithm is shown in~\cref{algo:log-barrier-ddpg}.
\begin{algorithm}[ht!]
\caption{Log-barrier DDPG}
\begin{algorithmic}[1]
\State Initialize replay memory $D$ to capacity $|D|$
\State Initialize the critic parameter $\theta \in {\mathbb R}^{m_1}$ such that ${Q_\theta } \cong \kappa {{\bf{1}}_{|S||A|}}$ with large enough $\kappa>0 $ and the target critic parameter $\theta' = \theta \in {\mathbb R}^{m_1}$
\State Initialize the online actor parameter $w \in {\mathbb R}^{m_2}$ and the target actor parameter $w'=w \in {\mathbb R}^{m_2}$

\State Set $k=0$

\For{Episode $i \in \{1,2,\ldots\}$}

\State Observe $s_0$

\For{$t \in\{0,1,\ldots, T-1\}$}

\State Take an action $a_t = \pi_w(s_t) + e_t$ where $e_t\sim {\cal N}(0,\sigma)$ is the exploration noise and ${\cal N}(0,\sigma)$ is the normal distribution.

\State Observe $r_{t+1},s_{t+1}$
\State Store the transition $(s_t ,a_t ,r_{t+1} ,s_{t + 1} )$ in $D$
\State Sample uniformly a random mini-batch $B$ of transitions $(s,a,r,s')$ from $D$

\State Critic update: perform few gradient descent steps
\[\theta  \leftarrow \theta  - \alpha_{\text{critic}} {\nabla _\theta } \tilde L_{\text{critic}}(\theta;\pi_w )\]

\State Actor update: perform a gradient descent step
\[w \leftarrow w + \alpha_{\text{actor}} {\nabla _w}L_{\text{actor}}(w)\]

\State Soft target network update
\[w' \leftarrow \tau w + (1 - \tau )w'\]
\[\theta' \leftarrow \tau \theta+ (1 - \tau )\theta'\]

\State Set $k \leftarrow k+1$

\EndFor

\EndFor
\end{algorithmic}\label{algo:log-barrier-ddpg}
\end{algorithm}

In~\cref{algo:log-barrier-ddpg}, we adopt settings that are consistent with the original DDPG paper~\citep{lillicrap2015continuous}. In particular, we employ target networks for both actor and critic and apply the soft target updates. Moreover,~\cref{algo:log-barrier-ddpg} incorporates exploration through Gaussian noise. Our main deviations lie in the structure of the critic loss and the initialization of the critic network. As presented in~\cref{appendix:hyperparameters}, we initialize the critic network bias with $\kappa = 100$, set the enforcement parameter to $\nu = 100$, and use a barrier parameter of $\eta = 0.015$ in the critic loss.

\section{Appendix: hyperparameters}\label{appendix:hyperparameters}
We summarize in the following tables the hyperparameter configurations employed in the experiments for the proposed log-barrier deep Q-learning and log-barrier DDPG. Specifically,~\cref{tab:hyperparameters_simple} lists the hyperparameters for the log-barrier deep Q-learning experiments, and~\cref{tab:hyperparameters_ddpg} lists those for the log-barrier DDPG experiments.
In the comparative experiments, the parameters of standard DQN and standard DDPG were tuned to achieve optimal performance. For fairness, common hyperparameters were controlled to be identical under the same experimental conditions.

\begin{table}[h!]
\centering
\caption{Hyperparameter settings of log-barrier deep Q-learning for the discrete control}
\label{tab:hyperparameters_simple}
\begin{tabular}{@{}ll@{}}
\toprule
\textbf{Hyperparameter} & \textbf{Value} \\ \midrule
Initial network bias ($\kappa$)  & 100 \\
Enforce parameter ($\nu$)        & 1000\\
Discount factor ($\gamma$)       & 0.99 \\
Batch size                       & 64 \\
Epsilon ($\epsilon$)             & 0.1 \\
Episodes                         & 200 (CartPole-v1, Acrobot-v1) \\
                                 & 500 (LunarLander-v3, \textcolor{red}{MountainCar-v0, Pendulum-v1}) \\
Barrier parameter ($\eta$)       & 7 (CartPole-v1), 0.25 (Acrobot-v1) \\
                                 & 2 (LunarLander-v3), \textcolor{red}{0.4, (MountainCar-v0), 1.5 (Pendulum-v1)} \\
Learning rate                    & 0.0002 (LunarLander-v3), 0.0005 (otherwise) \\
Replay buffer size               & 100,000 (CartPole-v1, Acrobot-v1) \\
                                 & 1,000,000 (LunarLander-v3, \textcolor{red}{MountainCar-v0, Pendulum-v1}) \\ \bottomrule
\end{tabular}
\end{table}

\begin{table}[h!]
\centering
\caption{Hyperparameter settings for log-barrier DDPG on MuJoCo tasks}
\label{tab:hyperparameters_ddpg}
\begin{tabular}{@{}ll@{}}
\toprule
\textbf{Hyperparameter} & \textbf{Value} \\ \midrule
Initial network bias ($\kappa$)   & 100 \\
Enforce parameter ($\nu$)        & 100\\

Discount factor ($\gamma$)      & 0.99 \\
Batch size                      & 256 \\
Epsilon ($\epsilon$)             & 0.01 \\
Barrier parameter ($\eta$)       & 0.015 \\

Exploration noise ($\sigma$)    & 0.1 (Gaussian) \\
Total timesteps                        & 2M steps equivalent (Walker2d, Ant, Hopper, HalfCheetah, Humanoid) \\
Target update rate ($\tau$)     & 0.005 \\
Actor learning rate             & 0.001 \\
Critic learning rate            & 0.001 \\
Replay buffer size              & 1,000,000 \\
Hidden units per layer          & 256 \\
Number of hidden layers          & 2 \\
\bottomrule
\end{tabular}
\end{table}

%added 11.17
\newpage

\section{General Remarks}\label{app:remarks}
Several discussions that are not included in the main text due to page limits are summarized below.
\begin{enumerate}
\item {\bf Analysis under deep neural network function approximation}: When a deep neural network is used, the model becomes a nonlinear function of the parameters $\theta$, and the precise theoretical results derived for the tabular setting no longer apply directly. Nevertheless, the tabular analysis provides useful intuition: it clarifies the role of the log-barrier formulation, the effect of the barrier coefficient eta, and the qualitative behavior one should expect when using function approximation, and thus offers approximate predictions for the deep-RL case.

\item {\bf Sensitivity of learning performance on $\eta$}: We found the proposed methods to be generally highly sensitive to the barrier coefficient $\eta$. Accordingly, we performed careful hyperparameter tuning for the experiments, and we acknowledge that this sensitivity is an important issue that requires further investigation. We can try an annealing scheme that gradually reduces $\eta$ during training, and this often yields improved results. However, annealing itself introduces additional tuning choices: if $\eta$ is decreased via a schedule or learned online, one must still decide the initial value, the decay rate (how quickly to reduce $\eta$, and the final value, and experimental outcomes can vary substantially depending on these choices). Nonetheless, these auxiliary techniques frequently enhance empirical performance at the cost of increase tunning efforts. To free the method from excessive tuning, we plan to develop improved algorithms that are less dependent on hyperparameter selection. For example, mirror-descent–style updates that operate directly in the space satisfying the inequality constraints could maintain feasibility without heavy reliance on $\eta$. We view these directions as promising topics for future work.

\item {\bf Effect of small $\eta$}: Theoretically, as $\eta$ goes to zero, the log-barrier solution $\tilde Q_{\eta}$ approaches the true $Q^*$. This relationship is established in~\cref{thm:bounds2}. In practice, however, taking $\eta$ too small can cause numerical issues. For example, with very small $\eta$, the gradient-based estimate of $\tilde Q_{\eta}$ may easily step outside the domain that satisfies the inequality constraints with relevantly small step-sizes.

\item {\bf Potential applicability to constrainted RL}: The proposed approach can be integrated very naturally into constrained or safe RL by, for example, incorporating a log-barrier term. However, this introduces additional tuning issue, namely the need to carefully set additional additional barrier coefficients. Consequently, we argue that a more principled framework, with stronger theoretical guarantees and improved algorithmic mechanisms for handling such hyperparameters, is necessary.

\end{enumerate}

\section{Appendix: additional experiments of log-barrier Q-learning approach with linear function approximation}\label{sec:app_FAsimul}

In this appendix section, we present additional experiments comparing standard Q-learning and a modified version of our log-barrier Q-learning in linear function approximation settings.

\subsection{Environment}
\subsubsection{Markov Decision process setting}To make the comparison concrete, we consider a toy Markov decision process (MDP) with $|{\cal S}|=4$ states and $|{\cal A}|=2$ actions. The expected reward $R(s,a)$ used in our simulations is defined as below:
\begin{align*}
R(s,a) =
\begin{pmatrix}
0.5 & 1.0 \\
0.2 & 0.4 \\
0.8 & 0.3 \\
1.5 & 0.1
\end{pmatrix},
\end{align*}
where each row corresponds to a state $s \in \{0,1,2,3\}$ and each column corresponds to an action $a \in \{0,1\}$.
The transition probability $P(s' \mid s,a)$ is described by two state–transition matrices, one for each action.
For action $a = 0$, we have
\[
P(s' \mid s, a = 0) =
\begin{pmatrix}
0.7 & 0.3 & 0.0 & 0.0 \\
0.0 & 0.6 & 0.0 & 0.4 \\
0.5 & 0.0 & 0.5 & 0.0 \\
0.6 & 0.4 & 0.0 & 0.0
\end{pmatrix},
\]
and for action $a = 1$, the transition matrix is
\[
P(s' \mid s, a = 1) =
\begin{pmatrix}
0.0 & 0.4 & 0.6 & 0.0 \\
0.0 & 0.0 & 0.8 & 0.2 \\
0.0 & 0.0 & 0.0 & 1.0 \\
0.0 & 0.0 & 0.0 & 1.0
\end{pmatrix}.
\]
\cref{fig:toymdpfig} provides a clear visual illustration of the environment described above.
\begin{figure}[t]
    \centering
    \includegraphics[width=0.9\linewidth]{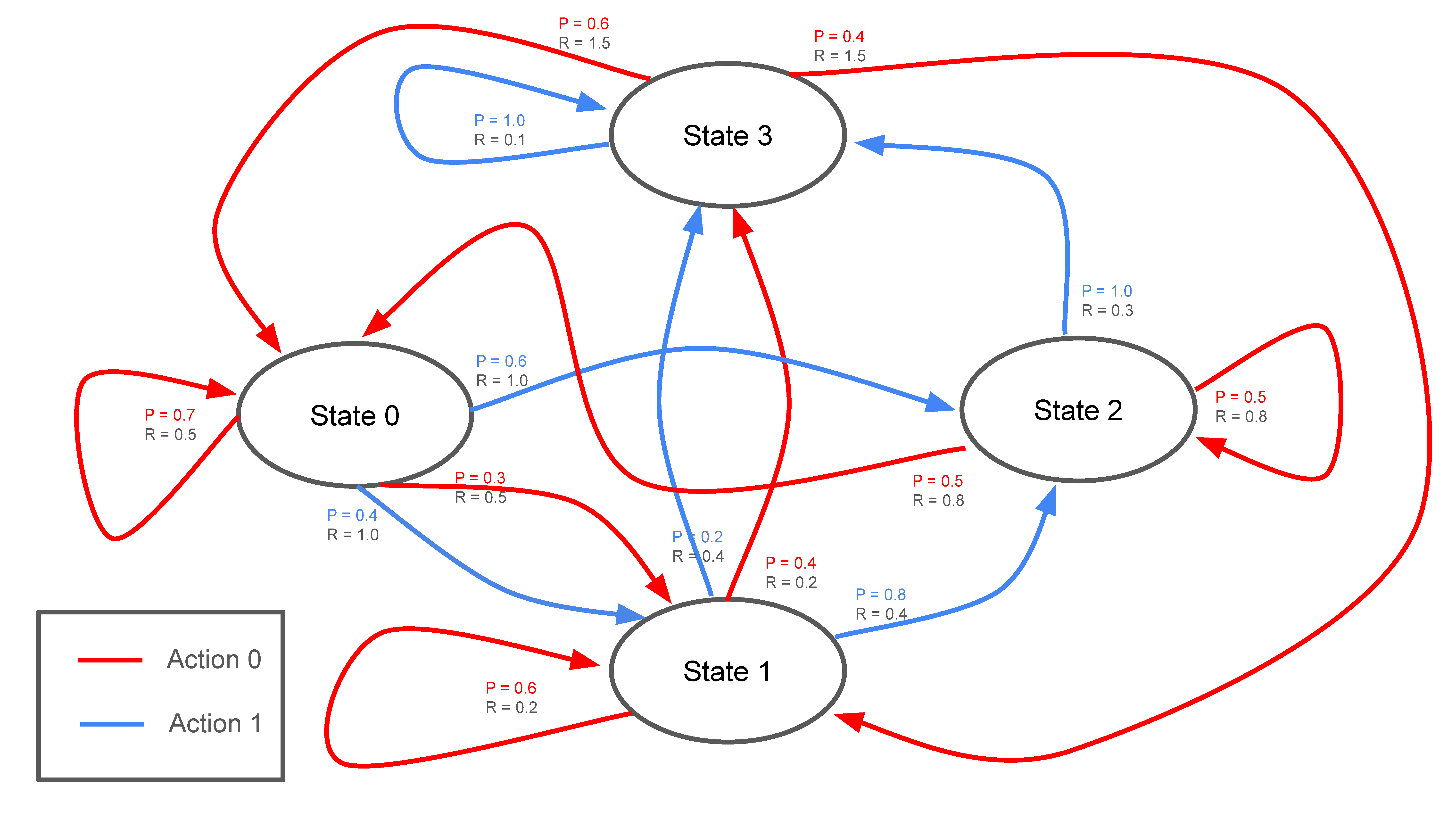}
    \caption{Visualization of the transition structure and reward layout of the toy MDP.
    Transitions under action 0 are shown as red arrows, while those under action 1 are shown as blue arrows.
    The corresponding reward values are indicated in black text.}
    \label{fig:toymdpfig}
\end{figure}

\subsubsection{Linear function approximation}
We next construct a linear function approximator with a fixed basis matrix of rank~6. This basis is deliberately designed so that several state–action pairs share identical or nearly collinear feature vectors, which allows us to highlight the well-known instability and under-performance of linear function approximation under off-policy Q-learning updates.
The basis matrix $\Phi \in \mathbb{R}^{8 \times 6}$, where each row corresponds to a state–action pair $(s,a)$, is defined as
\[
\Phi =
\begin{pmatrix}
1.0 & 0.0 & 0.0 & 0.1 & 0.0 & 0.0 \\   % (0,0)
0.1 & 0.0 & 0.0 & 1.0 & 0.0 & 0.0 \\   % (0,1)
1.0 & 0.0 & 0.0 & 0.1 & 0.0 & 0.0 \\   % (1,0)
0.1 & 0.0 & 0.0 & 1.0 & 0.0 & 0.0 \\   % (1,1)
0.0 & 1.0 & 0.0 & 0.0 & 0.1 & 0.0 \\   % (2,0)
0.0 & 0.1 & 0.0 & 0.0 & 1.0 & 0.0 \\   % (2,1)
0.0 & 0.0 & 1.0 & 0.0 & 0.0 & 0.1 \\   % (3,0)
0.0 & 0.0 & 0.1 & 0.0 & 0.0 & 1.0      % (3,1)
\end{pmatrix},
\]
where rows of $\Phi$ are ordered as $(s,a) = (0,0), (0,1), (1,0), (1,1), (2,0), (2,1), (3,0), (3,1)$.
Given this basis, the linear Q-function is represented as
\[
Q_w(s,a) = \phi(s,a)^\top w, \qquad w \in \mathbb{R}^{6},
\]
where $\phi(s,a)$ denotes the corresponding row of $\Phi$.
Because states $s=0$ and $s=1$ share identical feature patterns, the resulting approximation is intentionally limited in expressive power, making it a suitable scenario for examining divergence phenomena in both standard Q-learning and our log-barrier-adjusted updates.

To satisfy the theoretical precondition of the log-barrier approach, namely that the initial Q-values are sufficiently large, we initialize the parameter vector $w$ so that all state–action values start from a common high level. Concretely, let $q_{\mathrm{init}} \in \mathbb{R}$ be a prescribed initialization constant and define
\[
\mathbf{q}_0 \in \mathbb{R}^{|S||A|},
\qquad
\mathbf{q}_0 := q_{\mathrm{init}} \, \mathbf{1},
\]
where $\mathbf{1} \in {\mathbb R}^{|{\cal S}||{\cal A}|}$ denotes the all-ones vector of dimension $|{\cal S}||{\cal A}|$. We then choose the initial parameter $w_0$ by solving the least-squares problem
\[
w_0
= \arg\min_{w \in \mathbb{R}^{6}} \bigl\| \Phi w - \mathbf{q}_0 \bigr\|_2^2.
\]
In practice, this corresponds to computing the least-squares solution
\[
\Phi w_0 \approx \mathbf{q}_0,
\]
so that $Q_{w_0}(s,a) \approx q_{\mathrm{init}}$ holds for all $(s,a) \in {\cal S} \times {\cal A}$.

\subsubsection{Standard Q-learning} The standard Q-learning algorithm serves as a baseline method.
Following the classical formulation of semi-gradient Q-learning with linear function approximation~\citep{sutton1998reinforcement},
the update rules are given as follows:
\begin{align*}
&\text{Loss:} \quad
L(w) = \frac{1}{2}{\left( {{\rm{no\_grad}}[r(s,a,s') + \gamma {{\max }_{a'}}{Q_w}(s',a')] - {Q_w}(s,a)} \right)^2}\\[4pt]
&\text{Gradient:} \quad
\nabla_w L(w) = -\delta\, \phi(s,a),
\qquad
\delta = r(s,a,s') + \gamma \max_{a'\in {\cal A}} Q_{w}(s',a') - Q_{w}(s,a),  \\[4pt]
&\text{Update:} \quad
w \leftarrow w + \alpha\, \delta\, \phi(s,a).
\tag{R.1}\label{appeq:q_learning_update}
\end{align*}
where $(s,a,s',a')$ means the current state, current action, next state, next action, and $\rm{no\_grad}$ means that the corresponding term is set as a constant.

\subsubsection{Log-barrier Q-learning} Next, we consider our proposed method, the log-barrier Q-learning approach.
Although the implementation in Appendix~\ref{appendix:DQN algo} employs a deep neural network $Q_\theta$ to parameterize the Q-function, here we replace it with a linear function approximation of the form $Q_w = \Phi w$. Under this parametrization, we can derive a gradient-based update rule by directly computing the loss and its corresponding parameter update as:

Loss:
\begin{align*}
&L^{LB}(w) =
\begin{cases}
Q_w(s,a) - \eta \,\log\bigl(Q_w(s,a) - (TQ_w)(s,a,s') + \varepsilon\bigr),
& \text{if } Q_w(s,a) - (TQ_w)(s,a,s') > 0,\\[6pt]
Q_w(s,a) + \eta\, \nu\, \bigl((TQ_w)(s,a,s') - Q_w(s,a)\bigr),
& \text{if } Q_w(s,a) - (TQ_w)(s,a,s') \le 0,
\end{cases}
\end{align*}
where \((TQ_w)(s,a,s') = r(s,a,s') + \gamma \max_{u\in {\cal A}}Q_w(s',u) \).

Gradient:
\begin{align*}
&\nabla_w L^{LB}(w) =
\begin{cases}
\phi(s,a) - \displaystyle \eta\,\frac{\phi(s,a) - \gamma\, \phi(s',a^*)}{Q_w(s,a,s') - (TQ_w)(s,a,s') + \varepsilon},
& \text{if } Q_w(s,a) - (TQ_w)(s,a,s') > 0,\\[12pt]
\phi(s,a) -
\eta\,\nu\,
   \bigl(\phi(s,a) - \gamma\,\phi(s',a^*)\bigr),
& \text{if } Q_w(s,a) - (TQ_w)(s,a,s') \le 0,
\end{cases}
\end{align*}
where $a^* := \argmax_{u\in {\cal A}}Q_w(s',u)$.

Update:
\begin{align*}
&w \leftarrow w - \alpha\, \nabla_w L^{LB}(w),
\tag{R.2}\label{appeq:lb_q_learning_update}
\end{align*}

\subsubsection{Algorithm}
Both the standard Q-learning and the log-barrier Q-learning method share the same training loop  and differ only in how the parameter update is computed. The common algorithmic skeleton is summarized in~\cref{alg:base-loop}.

\begin{algorithm}[t]
\caption{Common training loop under linear function approximation with $\pi_b$}
\label{alg:base-loop}
\begin{algorithmic}[1]

\State \textbf{Input:} MDP environment, initial parameter vector $w_0$, basis matrix $\Phi$,
behavioral policy $\pi_b$, horizon $T$
\State Initialize state $s_0$ and set $w \leftarrow w_0$

\For{$t = 0,1,\dots,T-1$}

    \State Sample action $a_t \sim \pi_b(\cdot \mid s_t, w)$ \Comment{behavioral policy}
    \State Observe transition $(s_{t+1}, r_{t+1})$

    \State Compute gradient:
    \[
        \nabla L_t(w) =
        \begin{cases}
            \text{Standard Q-learning: } \nabla_w L(w), & \text{(Eq.~\ref{appeq:q_learning_update})} \\[4pt]
            \text{Log-barrier Q-learning: } \nabla_w L^{\text{LB}}(w), & \text{(Eq.~\ref{appeq:lb_q_learning_update})}
        \end{cases}
    \]

    \State Parameter update:
    \[
        w \leftarrow w - \alpha\, \nabla L_t(w)
    \]

\EndFor

\State \textbf{return} $w$, $Q = \Phi \cdot w$ .

\end{algorithmic}
\end{algorithm}

%-----

In this experiments, the behavioral policy $\pi_b$ is required only to have
\emph{full support} over the action set, i.e.,
\[
\pi_b(a \mid s) > 0 \qquad \text{for all } (s,a),
\]
so that every action remains reachable under the data-collection process. Any stochastic policy satisfying this condition can be used in the
common training loop (\cref{alg:base-loop}), and no particular structure
is assumed beyond full support.

With this flexibility in mind, and as we will later demonstrate,
we deliberately consider two contrasting choices of $\pi_b$:
the standard $\epsilon$-greedy policy and the
$\epsilon$-\emph{reverse}-greedy variant.
The former tends to favor the maximizer of the current value estimate,
whereas the latter intentionally favors the minimizer.
Both policies maintain full support (as long as $\epsilon>0$), but they induce
substantially different degrees of off-policy mismatch.
This design allows us to amplify and expose the classical instability
mechanisms associated with the deadly triad in linear Q-learning, while at the
same time enabling us to observe how the proposed log-barrier approach behaves
under such deliberately adverse learning conditions.

\subsection{Experiment Results}\label{app:FAsimul}
We consider two behavioral policies, each defined by a simple selection rule.
For any state \(s\), the standard \(\epsilon\)-greedy policy selects
\[
a =
\begin{cases}
\mathrm{Uniform}({\cal A}), & \text{with probability } \epsilon,\\[3pt]
\displaystyle \argmax_{a' \in {\cal A}} Q_w(s,a'), & \text{with probability } 1-\epsilon,
\end{cases}
\]
whereas the \(\epsilon\)-reverse-greedy policy selects
\[
a =
\begin{cases}
\mathrm{Uniform}({\cal A}), & \text{with probability } \epsilon,\\[3pt]
\displaystyle \argmin_{a' \in {\cal A}} Q_w(s,a'), & \text{with probability } 1-\epsilon.
\end{cases}
\]
Both policies maintain full support as long as \(\epsilon > 0\), but they differ significantly
in the extent to which they induce off-policy sampling.

As described above, we conducted two sets of experiments that differ only in the choice of the behavioral policy.
The first set employs the standard \(\epsilon\)-greedy strategy, while the second adopts the \(\epsilon\)-reverse-greedy variant.
These two policies induce notably different degrees of off-policy mismatch, enabling us to observe how each update rule behaves under progressively more adverse sampling conditions.
The corresponding results are presented in~\cref{fig:eps_greedy_result}
and~\cref{fig:eps_reverse_greedy_result}, respectively.

\begin{figure}[H]
    \centering
    \includegraphics[width=0.95\linewidth]{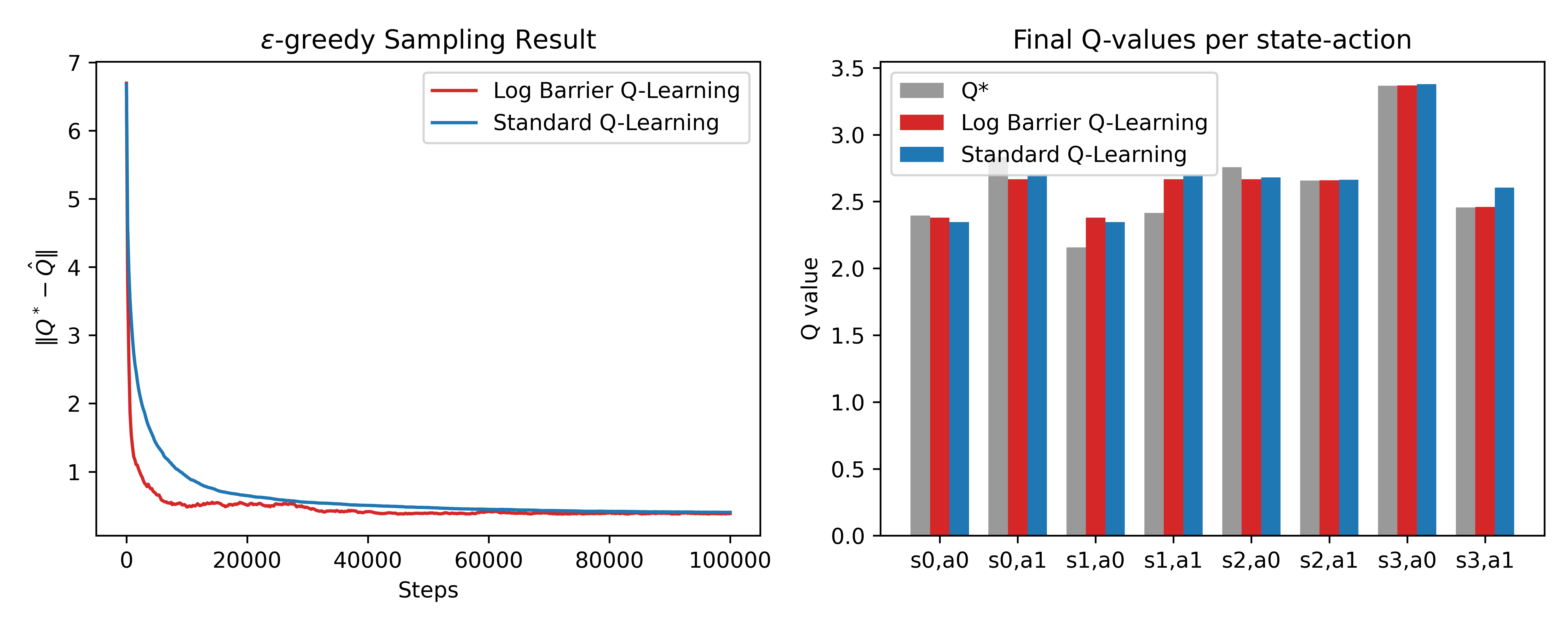}
    \caption{Experiment result under the $\epsilon$-greedy behavioral policy.}
    \label{fig:eps_greedy_result}
\end{figure}

\begin{figure}[H]
    \centering
    \includegraphics[width=0.95\linewidth]{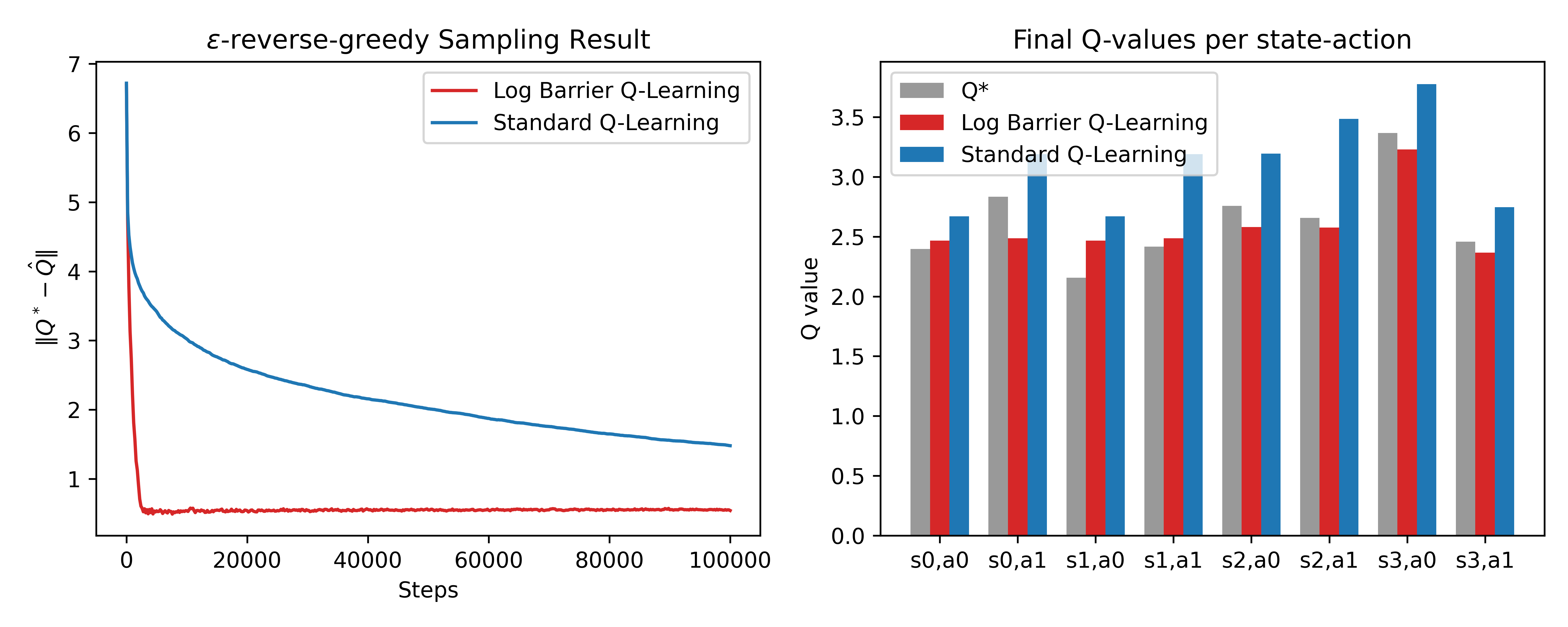}
    \caption{Experiment result under the $\epsilon$-reverse-greedy behavioral policy.}
    \label{fig:eps_reverse_greedy_result}
\end{figure}

From these results, by comparing~\cref{fig:eps_greedy_result} and~\cref{fig:eps_reverse_greedy_result},
we clearly see that standard Q-learning is highly sensitive to the sampling distribution induced by the behavioral policy.
When the distribution of the data deviates from that implied by the optimal policy, the resulting off-policy mismatch
amplifies error propagation—a phenomenon well documented in~\citep{munos2008finite}. In our experiments, deliberately modifying the behavioral policy (e.g., through the $\epsilon$-reverse-greedy strategy)
creates a more adverse off-policy regime. Under this setting, standard Q-learning (blue) exhibits slower convergence and larger approximation errors,
whereas the milder $\epsilon$-greedy regime leads to comparatively more stable learning outcomes, as expected.

In contrast, our log-barrier Q-learning method (red) exhibits a notably different behavior.
As long as the behavioral policy maintains full support, changes in its sampling distribution
do not substantially degrade performance. The algorithm consistently converges rapidly and produces more accurate value estimates. This can be clearly seen in~\cref{fig:eps_reverse_greedy_result}:
even under the adverse $\epsilon$-reverse-greedy policy, the final estimated $Q$-values remain
close to $Q^*$, and the corresponding error curve $\lVert Q^* - \hat{Q} \rVert$ exhibits
substantially reduced sensitivity to the off-policy sampling conditions.

However, we note that these results reflect idealized settings obtained through careful
hyperparameter tuning, and in practice each experiment required different choices of the barrier coefficient ($\eta$) and the infeasible enforce parameter ($\nu)$ to achieve stable performance. The precise coefficients used in each setting are summarized in~\cref{apptab:hyperparams}. In the next subsection, we discuss what phenomena arise during this tuning process and how
these parameters can be adjusted in a controlled and systematic manner to maintain stability
across different sampling conditions.

\begin{table}[t]
    \centering
    \caption{Hyperparameter settings for the toy MDP experiments under two behavioral policies.}
    \label{apptab:hyperparams}
    \begin{tabular}{lcc}
        \toprule
        \textbf{Parameter}
            & \textbf{$\epsilon$-greedy}
            & \textbf{$\epsilon$-reverse-greedy} \\
        \midrule
        \multicolumn{3}{c}{\textbf{General training settings}} \\
        \midrule
        $\gamma$ (discount factor)
            & $0.7$ & $0.7$ \\
        $\epsilon$ (exploration rate)
            & $0.3$ & $0.3$ \\
        $T$ (total training steps)
            & $100{,}000$ & $100{,}000$ \\
        $\alpha$ (learning rate)
            & $0.2$ & $0.2$ \\
        $Q_{\text{init}}$ (initial Q-value used for LS warm-start)
            & $5.0$ & $5.0$ \\
        \midrule
        \multicolumn{3}{c}{\textbf{Log-barrier coefficients}} \\
        \midrule
        $\eta$ (barrier parameter)
            & $5\times 10^{-5}$ & $5\times 10^{-5}$ \\
        $\varepsilon$ (feasibility margin)
            & $10^{-3}$ & $10^{-3}$ \\
        $\nu$ (infeasible-region enforce parameter)
            & $1.2\times 10^{5}$ & $3.2\times 10^{5}$ \\
        \bottomrule
    \end{tabular}
\end{table}

\subsection{Practical Considerations for Hyperparameter Selection}
\label{app:hyperparamselectguide}
\begin{figure}[htp!]
    \centering
    \includegraphics[width=0.95\linewidth]{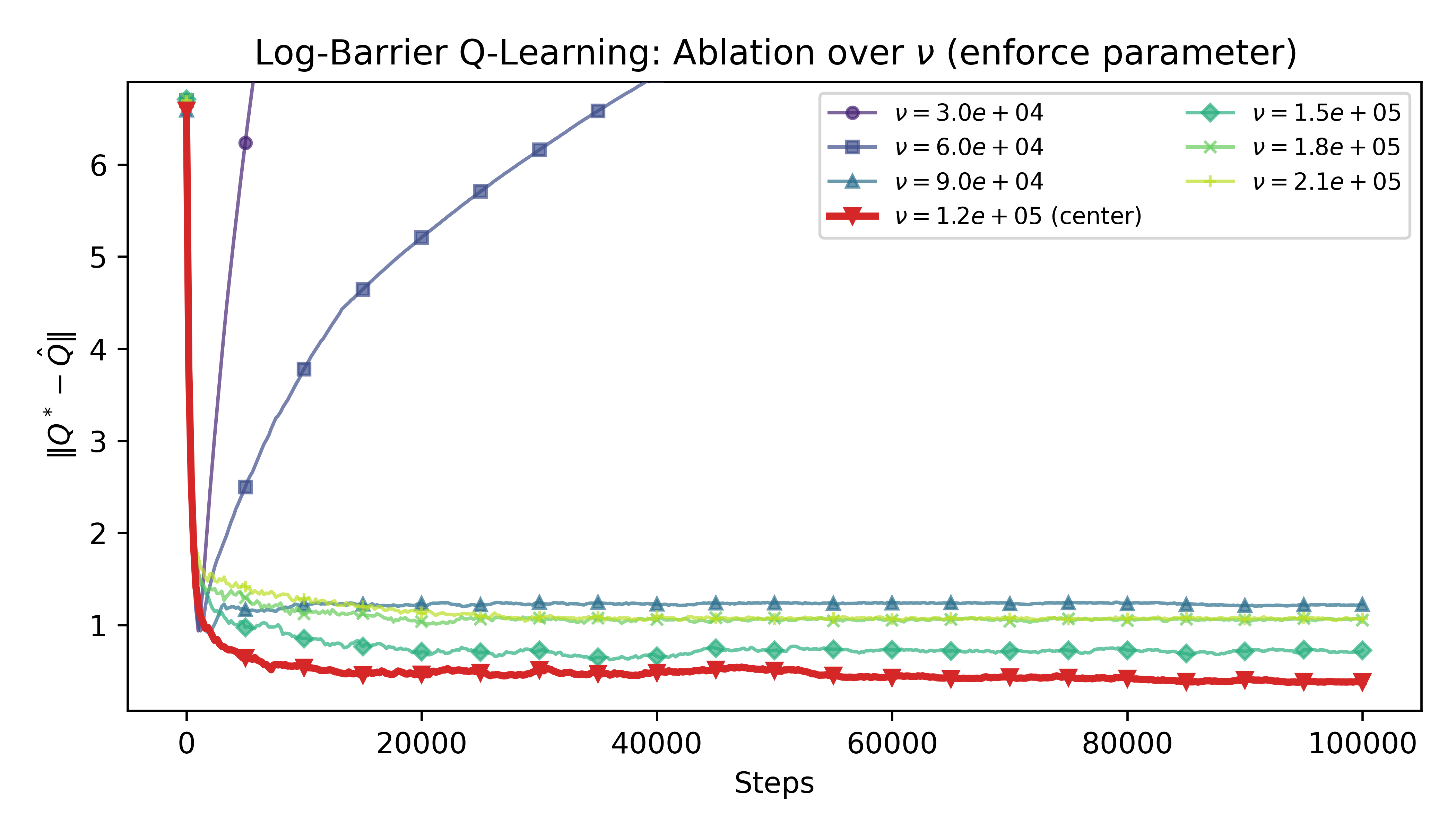}
    \caption{Ablation results of the log-barrier Q-learning method under the
$\epsilon$-greedy behavioral policy for various choices of the
enforce parameter $\nu$. The $\eta$ value is fixed as $5\times 10^{-5}$.}
    \label{appfig:nu_ablation}
\end{figure}

While tuning the hyperparameters, we observed several important patterns in how our approach behaves under different settings. The discussion in this subsection therefore provides practical guidance for selecting appropriate hyperparameters. In this subsection, we present an ablation study that varies the enforcement parameter $\nu$ while keeping the barrier coefficient $\eta$ fixed, evaluated under the $\epsilon$-greedy behavioral policy. The key observation is that, for a fixed value of $\eta$, the optimal choice of $\nu$ is tightly coupled with it. As shown clearly in~\cref{appfig:nu_ablation}, the best-performing hyperparameter pair appears at $\nu = 1.2 \times 10^5$ and $\eta = 5 \times 10^{-5}$ (highlighted by the red curve).

When the enforcement parameter $\nu$ is decreased below this optimal value, the error curve initially descends to a shallow minimum but then increases sharply, indicating a rapid degradation in stability. This phenomenon is not because the alrogithm suffers from overestimation. Rather, the pressure in the minimization step drives the estimated $\hat{Q}$-values down aggressively, leading to a sharp increase in the approximation error. This occurs because reducing $\nu$ weakens the penalty in the infeasible region, allowing the algorithm to accept the violation and push the estimated $\hat{Q}$-values further downward, ultimately destabilizing the error dynamics. This behavior is clearly illustrated in~\cref{appfig:nu_low}.

\begin{figure}[H]
    \centering
    \includegraphics[width=0.95\linewidth]{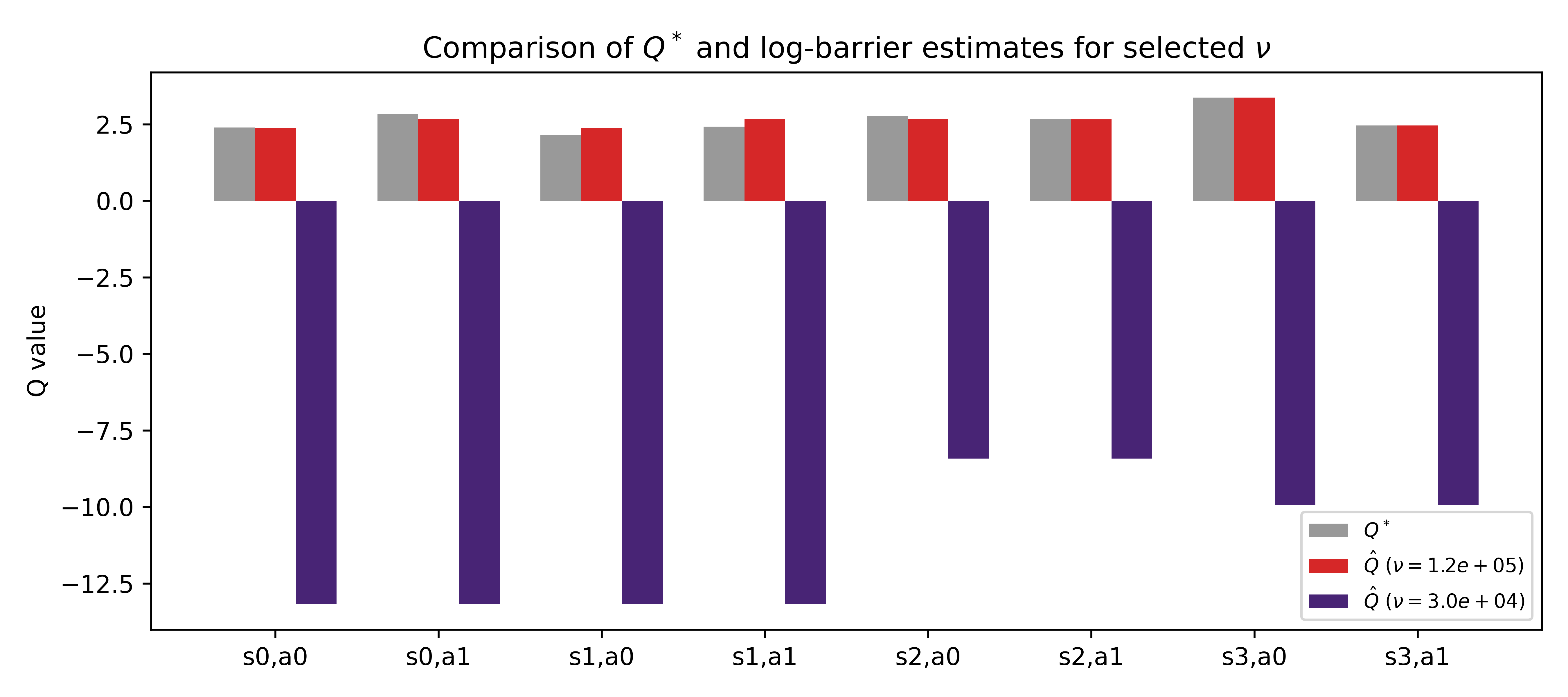}
    \caption{Comparison of $\hat{Q}$-values between the optimal and a smaller enforcement parameter $\nu$ setting.}
    \label{appfig:nu_low}
\end{figure}

In contrast, when $\nu$ is set larger than its optimal value, the framework still converges stably but exhibits a noticeable bias, which appears as an overestimation of the $Q$-values.
This overestimation does not stem from theoretical defect; rather, it arises because an excessively large enforcement parameter imposes strong penalties whenever the update violates the infeasible region.
The influence of this penalty persists even within the feasible region, effectively discouraging the algorithm from approaching the boundary closely and preventing it from accurately approximating the true $Q$-values. This behavior is clearly illustrated in~\cref{appfig:nu_high}.

\begin{figure}[H]
    \centering
    \includegraphics[width=0.95\linewidth]{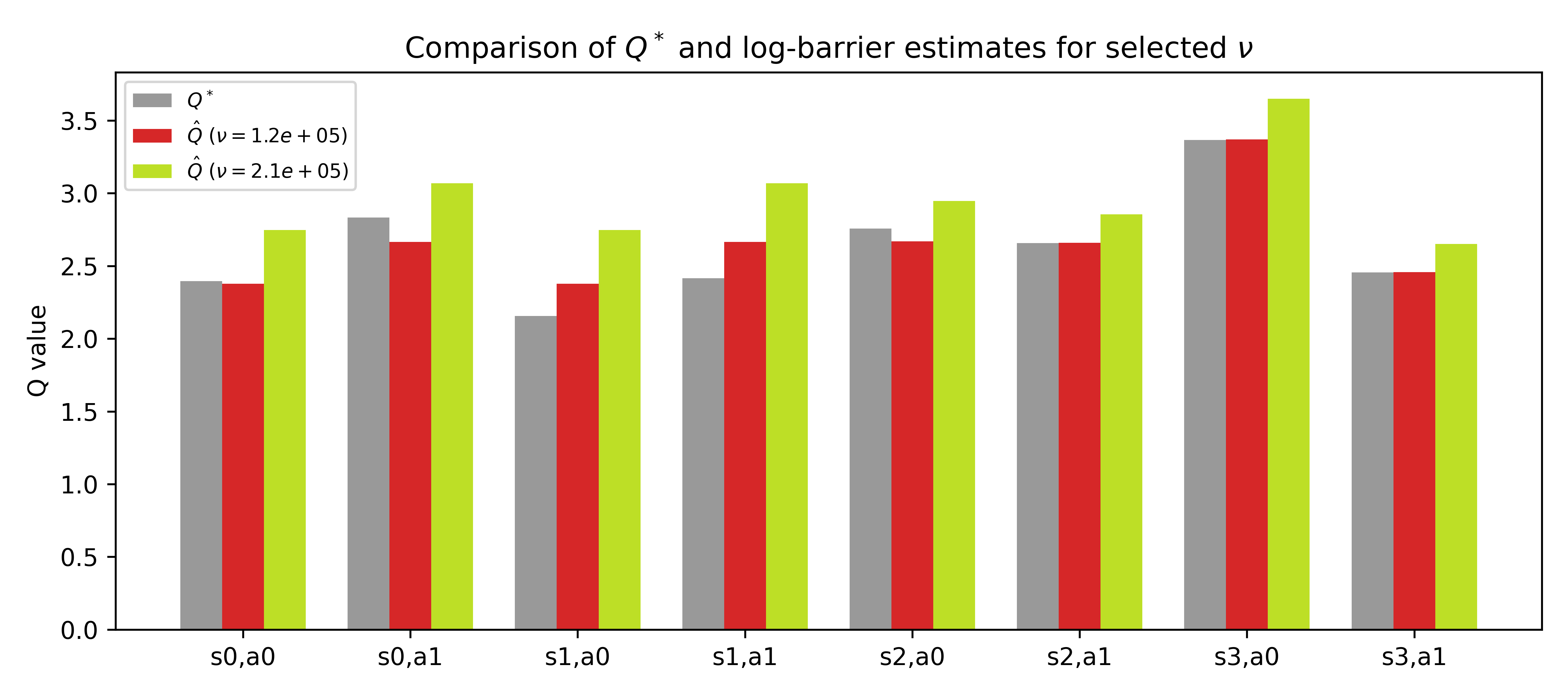}
    \caption{Comparison of $\hat{Q}$-values between the optimal enforcement parameter and a larger $\nu$ setting.}
    \label{appfig:nu_high}
\end{figure}

In practice, our empirical observations indicate that when the underlying MDP configuration is altered
(e.g., the discount factor $\gamma$, the sampling policy distribution, or the reward function $r$), the optimal values of the hyperparameters $\nu$ and $\eta$ also shift accordingly. Nevertheless, with appropriately tuned hyperparameter pairs, the algorithm consistently achieves estimates that remain close to the true optimal value function $Q^*$. According to our theoretical result in Theorem~\ref{thm:bounds2}, annealing the barrier coefficient $\eta$ during training may appear to be a reasonable strategy. However, such an approach inherently requires an adaptive adjustment of the enforcement parameter $\nu$,
which complicates the optimization process. For this reason, based on our empirical findings, it is generally more stable to fix one parameter (either $\eta$ or $\nu$) and tune the other accordingly.

\subsection{Tabular Setting under Stochastic sampling}
Using the experimental setup described in Appendix~\ref{sec:app_FAsimul},
our analysis can be directly extended to the tabular setting as a special case under the same stochastic sampling conditions.
In particular, when the identity feature map is adopted, the function approximation architecture becomes an exact tabular representation of the state–action value function.
Under this specialization, the tabular log-barrier Q-learning algorithm corresponds to the instance of~\cref{alg:base-loop} obtained by setting $\Phi = I$ and selecting log-barrier Q-learning as the learning framework, and its explicit form is presented in~\cref{alg:lb_ql_table}.

\begin{algorithm}[h!]
\caption{Log-barrier Q-learning (tabular case)}
\label{alg:lb_ql_table}
\begin{algorithmic}[1]
\State Initialize $Q(s,a)$ as a table.
\State Set step size $\alpha > 0$, discount $\gamma \in (0,1)$, log-barrier parameter $\eta > 0$, enforce parameter $\nu > 0$, and margin $\varepsilon > 0$
\For{each episode}
    \State $s \leftarrow s_{0}$
    \While{$s$ is not terminal}
        \State Choose $a$ at state $s$ by $\pi_b$  $w.r.t.\ Q(s,\cdot)$
        \State Take action $a$, observe reward $r$ and next state $s'$
        \If{$s'$ is terminal}
            \State $(TQ)(s,a,s') = r$
            \State $a^*$ undefined
        \Else
            \State $a^* = \argmax_{u\in {\cal A}} Q(s',u)$
            \State $(TQ)(s,a,s') = r + \gamma Q(s',a^*)$
        \EndIf
        \State Define the Bellman gap
        \Statex \[
            \Delta(s,a) := Q(s,a) - (TQ)(s,a,s').
        \]

        \State Compute the gradient at $(s,a)$:
        \Statex \[
        \nabla_{Q(s,a)} L(Q) \;=\;
        \begin{cases}
        1 - \displaystyle \eta\,\dfrac{1}{\Delta(s,a) + \varepsilon},
        & \text{if } \Delta(s,a) > 0, \\[10pt]
        1 - \eta\,\nu,
        & \text{else if } \Delta(s,a) \le 0.
        \end{cases}
        \]

        \State Compute the gradient at $(s',a^*)$ (full gradient into the target):
        \Statex \[
        \nabla_{Q(s',a^*)} L(Q) \;=\;
        \begin{cases}
            0,
            & \text{if } s' \text{ is terminal}, \\[10pt]
            \displaystyle \eta\,\gamma\,\frac{1}{\Delta(s,a) + \varepsilon},
            & \text{else if } \Delta(s,a) > 0, \\[12pt]
            \eta\,\nu\,\gamma,
            & \text{else if } \Delta(s,a) \le 0.
        \end{cases}
        \]

        \State Update current state--action value:
        \Statex \[
            Q(s,a) \;\leftarrow\; Q(s,a) - \alpha\,\nabla_{Q(s,a)} L(Q).
        \]
        \State Update next greedy state--action value:
        \Statex \[
            Q(s',a^*) \;\leftarrow\; Q(s',a^*) - \alpha\,\nabla_{Q(s',a^*)} L(Q).
        \]

        \State $s \leftarrow s'$
    \EndWhile
\EndFor
\end{algorithmic}
\end{algorithm}

From a theoretical perspective, standard Q-learning is known to converge in the tabular case, as established in classical results~\citep{jaakkola1993convergence, tsitsiklis1994asynchronous}. In contrast, our log-barrier framework, as shown in~\cref{thm:bounds2} and~\cref{thm:convergence-GD}, also converges but exhibits a systematic bias governed by the barrier coefficient $\eta$. Conducting experiments in the tabular setting therefore provides a clear and practical illustration of these theoretical properties.

\subsubsection{$\epsilon$-greedy Sampling}
Therefore, we conducted an experiment in the $\epsilon$-greedy setting under the tabular case to compare standard Q-learning with log-barrier Q-learning. The results of this comparison are presented in~\cref{appfig:tab-greedy}.

\begin{figure}[H]
    \centering
    \includegraphics[width=0.95\linewidth]{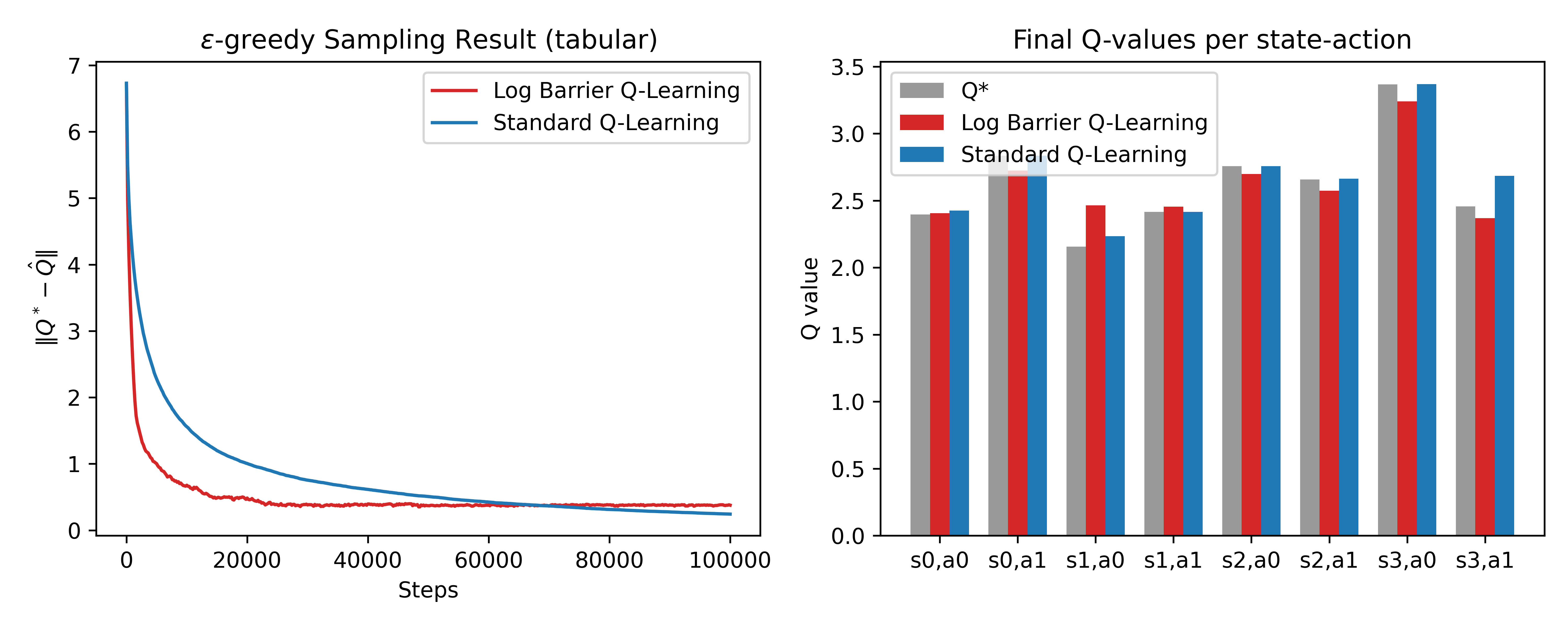}
    \caption{Experiment result under tabular setting using $\epsilon$-greedy sampling.}
    \label{appfig:tab-greedy}
\end{figure}

According to~\cref{appfig:tab-greedy}, the experimental results align well with our theoretical expectations.
While the standard Q-learning algorithm exhibits stable convergence in the tabular setting,
our log-barrier variant likewise converges stably but retains a small, systematic bias,
consistent with the analysis discussed above.
Nevertheless, the log-barrier method offers a clear advantage in terms of sampling efficiency:
it converges noticeably faster than standard Q-learning.
This highlights a natural trade-off between convergence speed and approximation bias,
which may motivate the choice of the log-barrier framework in scenarios where rapid learning is
preferable. In this experiment, the $\eta$ and $\nu$ are tuned as $\eta = 5 \times 10^{-5}$ and $\nu = 1.3 \times 10^{5}$, and any other hyperparameters are same as~\cref{apptab:hyperparams}.

\subsubsection{$\epsilon$-reverse-greedy sampling}

Similarly, we conducted an $\epsilon$-reverse-greedy sampling experiment to examine the behavior of both algorithms under tabular setting. The corresponding results are shown in~\cref{appfig:tab-reverse-greedy}.

\begin{figure}[H]
    \centering
    \includegraphics[width=0.95\linewidth]{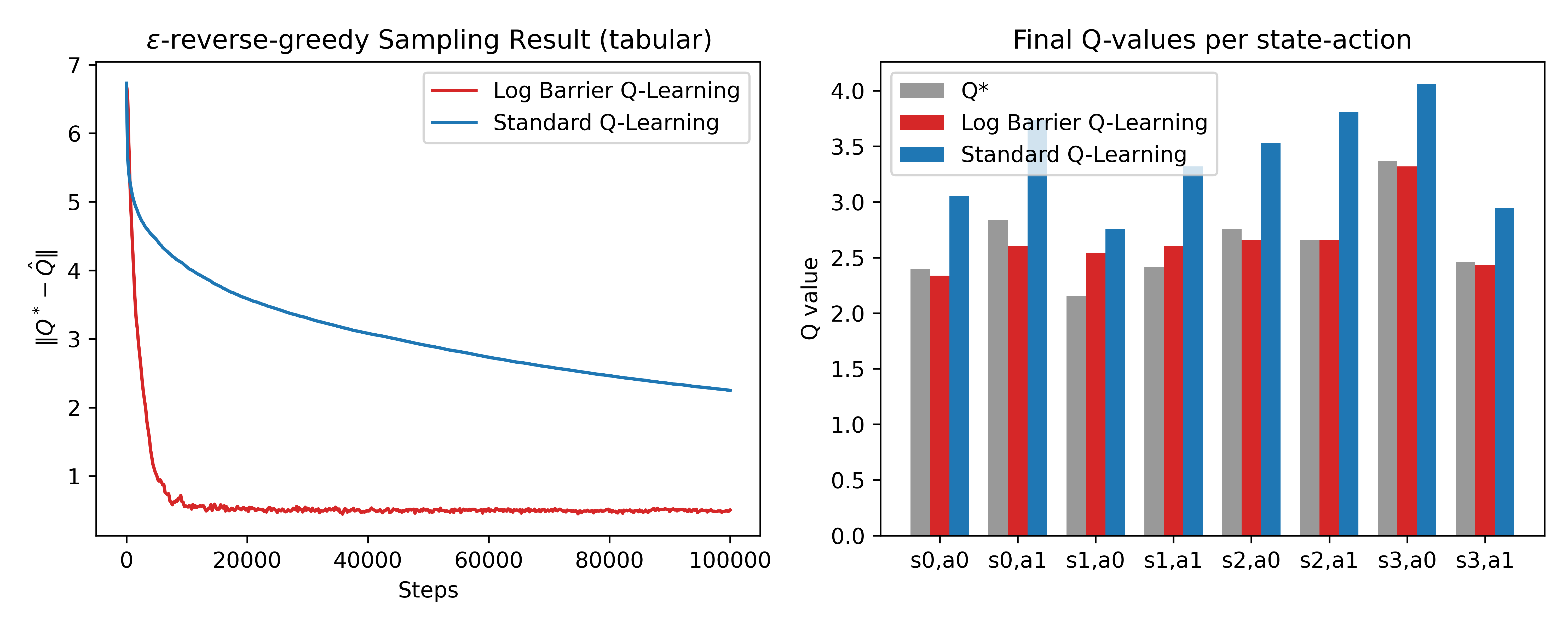}
    \caption{Experiment result under tabular setting using $\epsilon$-greedy sampling.}
    \label{appfig:tab-reverse-greedy}
\end{figure}

Under this setting, although standard Q-learning is guaranteed to converge, its convergence is noticeably slow due to the adverse sampling conditions. In contrast, our log-barrier approach exhibits coherent behavior, converging more rapidly while maintaining stable learning dynamics. In this experiment, we set the hyperparameters to $\eta = 5 \times 10^{-5}$ and $\nu = 6.0 \times 10^{5}$, and any other hyperparameters are same as~\cref{apptab:hyperparams}.

\section{Appendix: comparison of log-barrier Q-learning and primal-dual Q-learning}

In this short section, we briefly compare the performance of the proposed log-barrier Q-learning and a conventional stochastic primal–dual Q-learning method in a simple tabular setting. We use the OpenAI Gym \texttt{FrozenLake} environment with the \texttt{is\_slippery} option set to \texttt{True}, which induces stochasticity in the environment and makes it challenging to accurately compute the action–value function. The primal–dual Q-learning algorithm follows the approach of~\citet{lee2018stochastic}, which solves the MDP’s LP formulation using a first-order primal–dual method; this algorithm is applicable only in the tabular case. \cref{appfig:SPD-Q_vs_log-barrior-Q} below shows the return curves for the two algorithms.
\begin{figure}[H]
    \centering
    \includegraphics[width=0.95\linewidth]{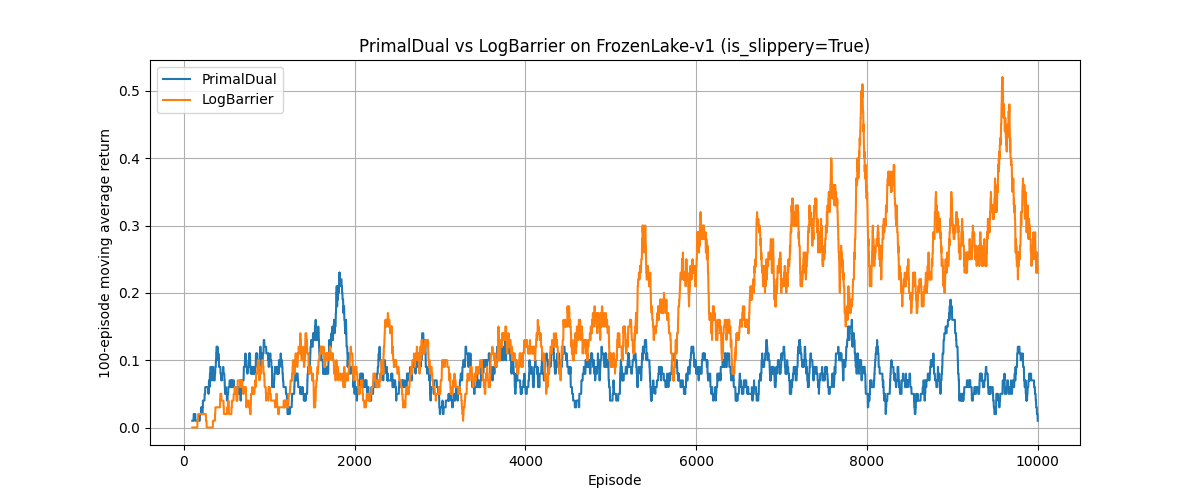}
    \caption{Return vs episodes plot of the proposed log-barrier Q-learning and stochastic primal-dual Q-learning~\citep{lee2018stochastic}}
    \label{appfig:SPD-Q_vs_log-barrior-Q}
\end{figure}
From the figure, we observe that log-barrier Q-learning achieves superior learning performance compared with the first-order primal–dual method. Note that the hyperparameters for each method were carefully tuned to give the best possible performance for that method.

\section{Appendix: ablation and additional experiments of deep learning variants}
As discussed in~\cref{app:remarks}, we found the proposed methods to be highly sensitive to the barrier coefficient~$\eta$. In this section, we present an ablation study over different choices of~$\eta$ and analyze the sensitivity of learning performance to this coefficient. In addition, we provide additional DDPG experiments that were omitted from the main text due to page limitations.

\subsection{Ablation study on deep learning variants}
To investigate the sensitivity of the barrier coefficient ($\eta$) in the log-barrier DQN and DDPG algorithms, we conducted a series of ablation studies by varying the value of $\eta$. For each Gymnasium environment used for log-barrier DQN, we performed a grid search over $\eta$ using task-specific ranges and step sizes rather than a single fixed grid across all tasks. For MuJoCo environments used for log-barrier DDPG, we test the coefficient $\eta$ ranged from 0.005 to 0.05 in increments of 0.005. As presented in~\cref{appfig:dqnab} and~\cref{appfig:ddpgab}, we found the barrier coefficient ($\eta$) to be a highly sensitive hyperparameter, particularly in MuJoCo environments.

\begin{figure}[H]
    \centering
    \includegraphics[width=0.95\linewidth]{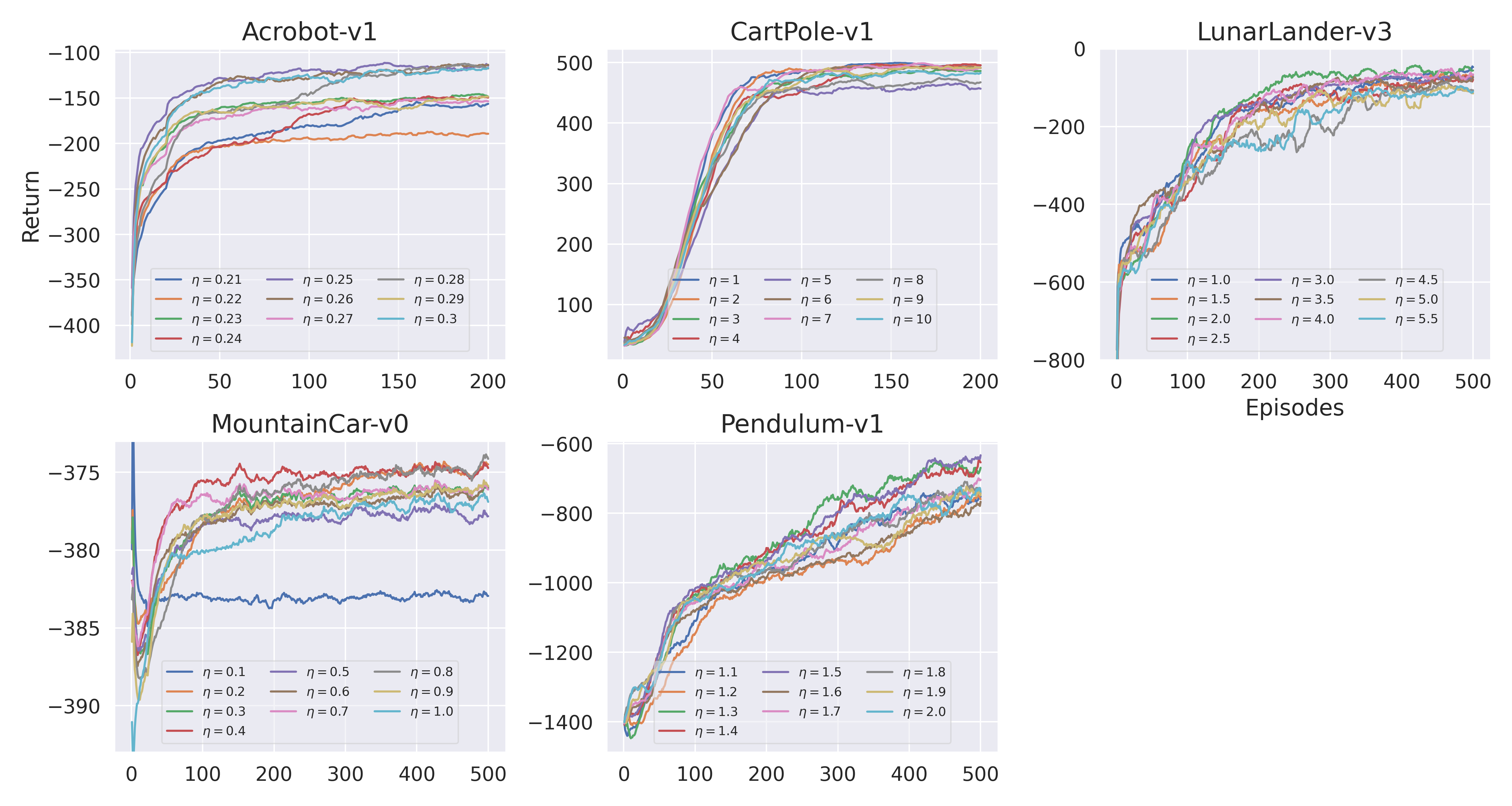}
    \caption{Ablation study of log-barrier DQN over the log-barrier coefficient $\eta$ across different environments.}
    \label{appfig:dqnab}
\end{figure}

\begin{figure}[H]
    \centering
    \includegraphics[width=0.95\linewidth]{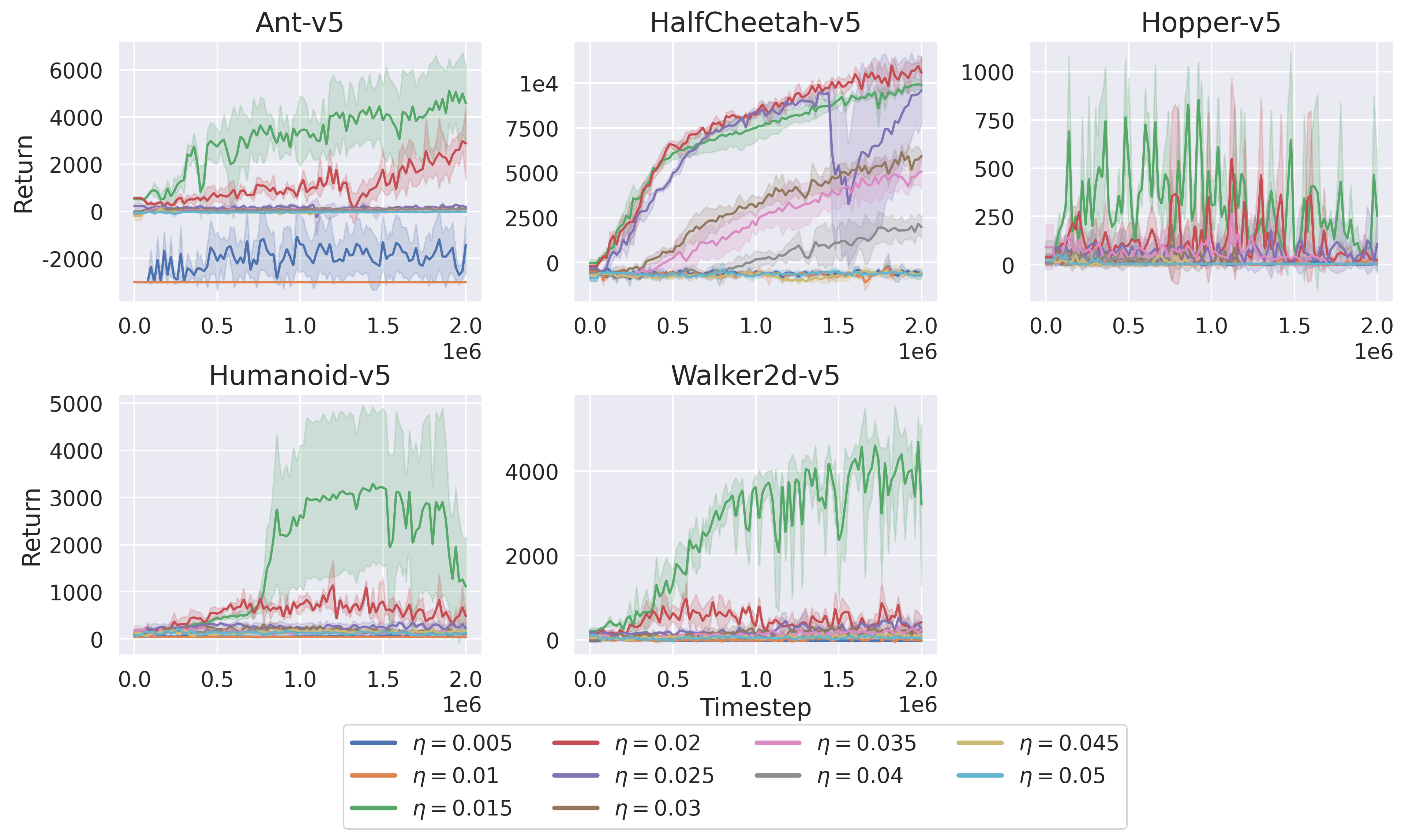}
    \caption{Ablation study of log-barrier DDPG over the log-barrier coefficient $\eta$ across different environments.}
    \label{appfig:ddpgab}
\end{figure}

\subsection{Additional experiments of log-barrier DDPG}
To further demonstrate the performance of log-barrier DDPG, we conducted additional experiments beyond those reported in Section~\ref{sec:experiments} and present the extended results below. We evaluate and compare the original DDPG with our proposed log-barrier DDPG across several Gymnasium-Robotics environments, using five different random seeds for each algorithm. The resulting performance curves are shown in~\cref{appfig:add_ddpg}, and the list of environments along with their final returns is summarized in~\cref{apptab:addDDPG}. Overall, log-barrier DDPG performs slightly better or at least comparably to the original DDPG in most environments.

From the results~\cref{appfig:add_ddpg} and~\cref{apptab:addDDPG}, the fact that \texttt{FetchPickAndPlaceDense}, \texttt{FetchPushDense}, and \texttt{FetchSlideDense} exhibit nearly identical returns for both algorithms may initially appear unusual. However, this behavior is expected once we consider that these environments compute rewards solely based on the distance between the movable object (“puck”) and the goal position. Since both algorithms completely fail to manipulate the object in these challenging tasks, the puck remains near its initial position, resulting in identical returns for both DDPG and log-barrier DDPG.

\begin{figure}[h]
    \centering
    \includegraphics[width=0.95\linewidth]{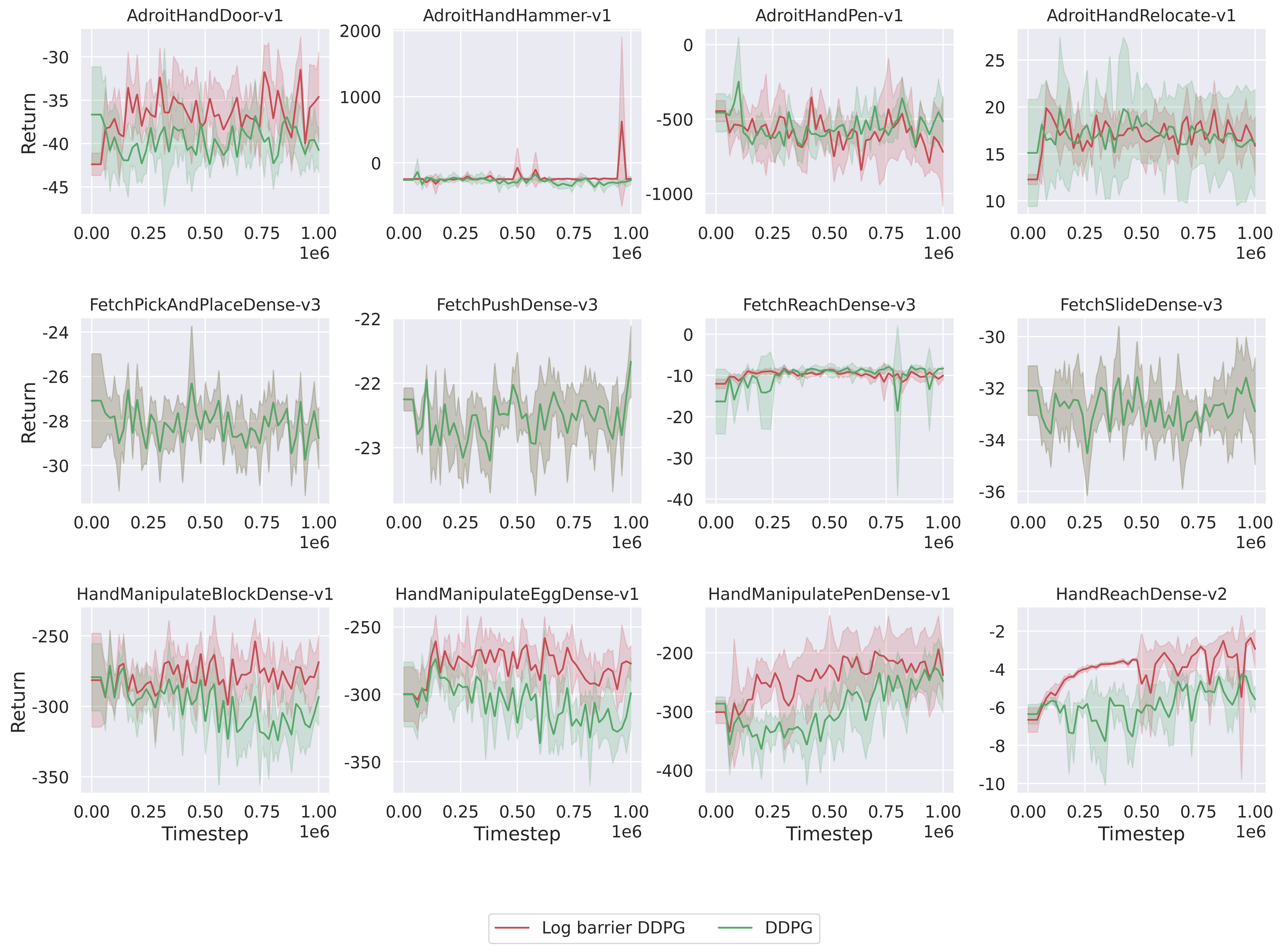}
    \caption{Additional experiments conducted in various Gymnasium-robotics environments.}
    \label{appfig:add_ddpg}
\end{figure}

\begin{table}[h]
    \centering
    \begin{tabular}{l c c}
        \toprule
        Environment & DDPG & Log-barrier DDPG \\
        \midrule
        AdroitHandDoor-v1 & -39.3 $\pm$ 1.3 & \textbf{-36.0 $\pm$ 2.0} \\
        AdroitHandHammer-v1 & -303.1 $\pm$ 22.0 & \textbf{-154.8 $\pm$ 129.4} \\
        AdroitHandPen-v1 & \textbf{-519.1 $\pm$ 56.1} & -635.9 $\pm$ 125.0 \\
        AdroitHandRelocate-v1 & 16.5 $\pm$ 5.0 & \textbf{17.3 $\pm$ 0.8} \\
        FetchPickAndPlaceDense-v3 & \textbf{-28.3 $\pm$ 0.4} & \textbf{-28.3 $\pm$ 0.4} \\
        FetchPushDense-v3 & \textbf{-22.7 $\pm$ 0.1} & \textbf{-22.7 $\pm$ 0.1} \\
        FetchReachDense-v3 & \textbf{-9.3 $\pm$ 1.0} & -10.2 $\pm$ 0.7 \\
        FetchSlideDense-v3 & \textbf{-32.5 $\pm$ 0.2} & \textbf{-32.5 $\pm$ 0.2} \\
        HandManipulateBlockDense-v1 & -308.7 $\pm$ 14.1 & \textbf{-278.8 $\pm$ 5.4} \\
        HandManipulateEggDense-v1 & -318.2 $\pm$ 14.4 & \textbf{-285.2 $\pm$ 5.5} \\
        HandManipulatePenDense-v1 & -248.8 $\pm$ 21.5 & \textbf{-223.1 $\pm$ 51.2} \\
        HandReachDense-v2 & -5.1 $\pm$ 0.5 & \textbf{-3.2 $\pm$ 0.6} \\
        \bottomrule
    \end{tabular}
    \caption{DDPG vs. Log-barrier DDPG Performance comparison across different environments.}
    \label{apptab:addDDPG}
\end{table}

\newpage
\section{Appendix: experiments of maximization bias}

In this section, we demonstrate that our log-barrier approach can effectively avoid overestimation, a phenomenon from which iteration-based algorithms such as Q-learning often suffer. To illustrate this, we employ the well-known maximization bias example introduced in~\citet{sutton1998reinforcement} where there exist two non-terminal states $\mathsf{A}$ and $\mathsf{B}$ at which two actions, $\mathsf{left}$ and $\mathsf{right}$, are admitted as illustrated in~\cref{fig:Sutton_example}. There also are two terminal states denoted with grey boxes and $\mathsf{A}$ is the initial state. Executing the action $\mathsf{right}$ leads to the termination of an episode with no rewards, while the action $\mathsf{left}$ leads to state $\mathsf{B}$ with no rewards.
There exist many actions at state $\mathsf{B}$, which will all immediately terminate the episode with a reward drawn from a normal distribution ${\cal N}(\mu,1)$. In our experiments, we consider $r\sim {\cal N}(\mu,1)$ with $\mu=-0.1$ .

\begin{figure}[h]
\centering
\includegraphics[width=0.6\linewidth]{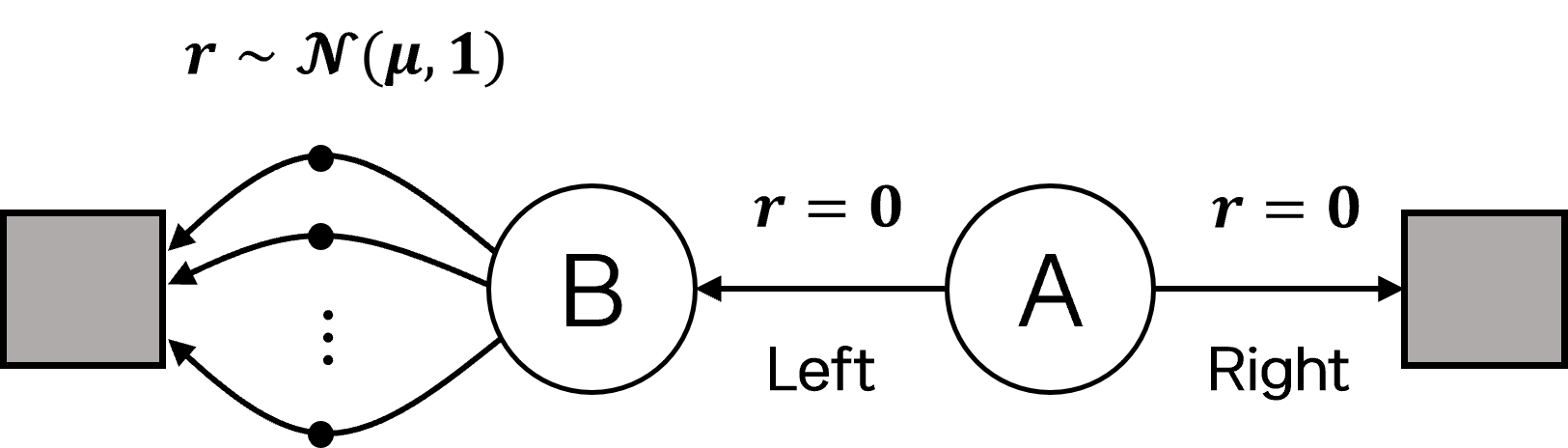}
\caption{Sutton's example for demonstrating overestimation bias}
\label{fig:Sutton_example}
\end{figure}

The implementations of Q-learning and double Q-learning follow the descriptions in~\citet{sutton1998reinforcement}, both using an $\epsilon$-greedy exploration strategy with $\epsilon = 0.1$. In this experiments, we adopt~\cref{alg:lb_ql_table} for the tabular log-barrier Q-learning method and configure its $\epsilon$-greedy behavior policy to match the same $\epsilon = 0.1$, ensuring a consistent exploration scheme across all methods.
For this experiment, we carefully tuned the hyperparameters of log-barrier Q-learning to $\eta = 0.0005$, $\nu = 2000$, and $\varepsilon = 0.01$. We emphasize that the parameter $\varepsilon$ follows the definition given in~\cref{appendix:DQN algo} and is strictly distinct from the exploration rate $\epsilon$ used in the behavior policy.

\begin{figure}[h]
    \centering
    \includegraphics[width=0.8\linewidth]{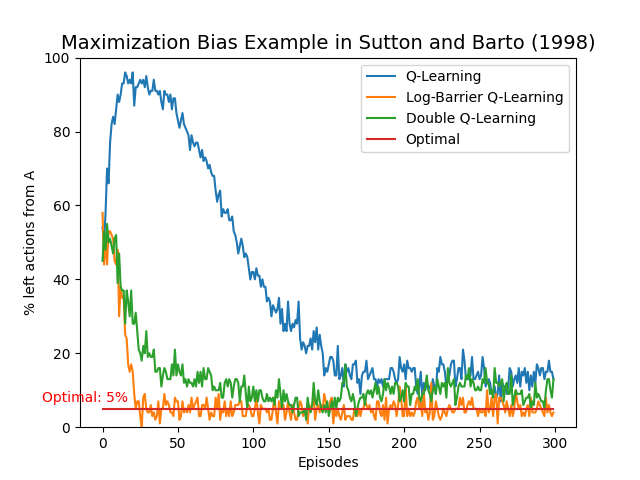}
    \caption{Comparison of maximization bias among Q-learning, log-barrier Q-learning, and double Q-learning. The implementations of Q-learning and double Q-learning strictly follow the descriptions in \citet{sutton1998reinforcement}.}

    \label{appfig:sutton}
\end{figure}

From~\cref{appfig:sutton}, we observe that our proposed log-barrier Q-learning effectively mitigates the overestimation issue. In contrast, original Q-learning exhibits substantial overestimation, as is well known, and double Q-learning, which was specifically developed to address this problem, nevertheless retains a small but non-negligible positive bias above the 5\% line even after convergence. By comparison, log-barrier Q-learning converges to the optimal line more rapidly and maintains unbiased action selection throughout training. Moreover, whereas double Q-learning requires two separate Q-tables, our approach is able to accurately estimate the optimal action using only a single Q-table.

\end{document}